\pdfoutput=1
\documentclass[11pt]{article} \usepackage{times}
\usepackage[letterpaper, left=1in, right=1in, top=1in, bottom=1in]{geometry}

\usepackage{xspace}
\usepackage{dsfont}
\usepackage[utf8]{inputenc} \usepackage[T1]{fontenc}    \usepackage{url}            \usepackage{booktabs}       \usepackage{amsfonts}       \usepackage{nicefrac}       \usepackage{microtype}      \usepackage{mkolar_definitions}
\usepackage[normalem]{ulem}
\usepackage{color}
\usepackage{enumitem}
\usepackage{comment}
\usepackage{bm}
\usepackage[scaled=.9]{helvet}
\usepackage{url}
\usepackage{authblk}
\usepackage{caption}

\usepackage{graphicx} \usepackage{subfigure}

\usepackage{algorithm,algorithmicx,algpseudocode}

\usepackage{natbib}

\usepackage[dvipsnames]{xcolor}
\usepackage[colorlinks=true, linkcolor=blue, citecolor=blue]{hyperref}

\usepackage{natbib}
\bibpunct{(}{)}{;}{a}{,}{,}

\newcommand{\OMIT}[1]{}
\newcommand{\refeq}[1]{Equation~(\ref{#1})}
\newcommand{\LDOTS}{\, ,\ \ldots\ ,}

\newcommand{\KL}{\mathrm{KL}}
\DeclareMathOperator*{\supp}{{\rm supp}}

\usepackage{xcolor}
\newcount\Comments  \Comments=0 \definecolor{darkgreen}{rgb}{0,0.5,0}
\definecolor{darkred}{rgb}{0.7,0,0}
\definecolor{teal}{rgb}{0.3,0.8,0.8}
\newcommand{\kibitz}[2]{\ifnum\Comments=1\textcolor{#1}{\textsf{\footnotesize #2}}\fi}

\usepackage{ifthen}

\newcommand{\version}{arxiv}
\ifthenelse{\equal{\version}{diff}}{
\newcommand{\df}[1]{\textcolor{blue}{#1}}
\newcommand{\dfc}[1]{\textcolor{red}{\protect{\sout{#1}}}}
}{
\newcommand{\df}[1]{#1}
\newcommand{\dfc}[1]{}
}

\usepackage[suppress]{color-edits}
\addauthor{as}{red}      

\newcommand{\term}[1]{\ensuremath{\mathtt{#1}}\xspace}
\newcommand{\indicator}[1]{\one_{\left\{#1\right\}}}
\newcommand{\bench}{\term{Bench}} \newcommand{\ball}{\term{B}} \newcommand{\ExpL}{\lambda} \newcommand{\base}{\nu}
\newcommand{\eps}{\epsilon}

\newcommand{\smooth}{\term{Smooth}}

\newcommand{\expf}{\term{EXP4}}
\newcommand{\expfour}{\expf}

\newcommand{\corral}{\term{Corral}}
\newcommand{\pe}{\term{PolicyElimination}}

\newcommand{\smoothexp}{\term{ContinuousEXP4}}
\newcommand{\corralexpf}{\term{Corral\!+\!EXP4}}

\newcommand{\smoothpe}{\term{SmoothPolicyElimination}}

\newcommand{\zpelip}{\term{SmoothPolicyElimination.L}}

\DeclareMathOperator\median{median}

\newcommand{\Reg}{\term{Regret}}

\newcommand{\DTV}{\mathrm{d}_{\text{TV}}}
\newcommand{\alg}{\textsc{Alg}\xspace}

\newcommand{\Rs}{R_{\textrm{S}}}
\newcommand{\Rl}{R_{\textrm{Lip}}}

\newcommand{\setdiff}{\Delta}

\newcounter{protocol}

 
\title{Contextual Bandits with Continuous Actions:\\ Smoothing, Zooming, and Adapting}
\date{}
\author[1]{Akshay Krishnamurthy}
\author[1]{John Langford}
\author[1]{Aleksandrs Slivkins}
\author[2]{Chicheng Zhang\thanks{\{akshaykr,jcl,slivkins\}@microsoft.com, chichengz@cs.arizona.edu}}
\affil[1]{Microsoft Research, New York, NY}
\affil[2]{University of Arizona, Tucson, AZ}

\begin{document}

\maketitle

\begin{abstract}
We study contextual bandit learning with an abstract policy class and
continuous action space.  We obtain two qualitatively different regret
bounds: one competes with a smoothed version of the policy class under
no continuity assumptions, while the other requires standard Lipschitz
assumptions. Both bounds exhibit data-dependent ``zooming'' behavior
and, with no tuning, yield improved guarantees for benign problems.
We also study adapting to unknown smoothness parameters, establishing
a price-of-adaptivity and deriving optimal adaptive algorithms that
require no additional information.
\end{abstract}
 
\newpage
\setcounter{tocdepth}{2}
\tableofcontents
\newpage

\section{Introduction}
\label{sec:intro}

We consider contextual bandits, a setting in which a learner
repeatedly makes an action on the basis of contextual information and
observes a loss for the action, with the goal of minimizing cumulative
loss over a series of rounds. Contextual bandit learning has received
much attention, and has seen substantial success in
practice~\citep[e.g.,][]{auer2002nonstochastic,Langford-nips07,
  agarwal2014taming,DS-arxiv}.  This line of work mostly considers small,
finite action spaces, yet in many real-world problems actions are chosen from an interval, so the action space is continuous and infinite. \df{Therefore, we ask:}

\vspace{-0.2cm}
\begin{quote}
\emph{How can we learn to make decisions from continuous action spaces,\\ {using (only) bandit feedback?}}
\end{quote}
\vspace{-0.2cm}

We could assume that nearby actions have similar losses, for example
that the losses are Lipschitz continuous as a function of the
action \citep[following][and a long line of subsequent
  work]{agrawal1995continuum}.  Then we could discretize the action space
and apply generic contextual bandit techniques \citep{Bobby-nips04} or
more refined ``zooming''
approaches~\citep{kleinberg2013bandits,bubeck2011x,slivkins2014contextual}
that are specialized to the Lipschitz structure.

However, this approach has several drawbacks.
A global Lipschitz assumption is crude and limiting; actual
problems exhibit more complex loss structures where smoothness varies
with location, often with discontinuities.
Second, prior works incorporating context --- including the zooming approaches --- employ a
nonparametric benchmark set of policies, which yields a poor dependence on the
context dimension and prevents application beyond low-dimensional
context spaces.  Finally, existing algorithms require knowledge of the
Lipschitz constant {or other pertinent parameters}, which are typically unknown.

Here we show that it is possible to avoid all of these drawbacks with
a conceptually new approach, resulting in a more robust solution for managing
continuous action spaces.
The key idea is to \emph{smooth} the actions: each action $a$ is mapped to a well-behaved distribution over actions.
{When the action space is the interval $[0,1]$, this distribution can be a uniform distribution over a narrow band around $a$: an interval $[a-h,a+h]$, where $h>0$ is a given \emph{bandwidth} parameter.
Rather than restrict the loss function, we posit a different, ``smoothed'' benchmark.}
This approach leads to provable guarantees with no assumptions on the loss function,
since the loss for smoothed actions is always well-behaved.
Essentially, we may focus on estimation considerations while ignoring approximation issues.
We recover prior results that assume a small Lipschitz constant, but the guarantees are meaningful in much broader scenarios.

Our algorithms work with any competitor policy set $\Pi$ of
mappings from context to actions, which we smooth as above. We measure
performance by comparing the learner's loss to the loss of the best
smoothed policy, and our guarantees scale with $\log |\Pi|$,
regardless of the dimensionality of the context space.
\df{Compared with prior work, this recovers some known worst-case results that can only accommodate nonparametric policy sets~\citep{slivkins2014contextual,cesa2017algorithmic}, but, more importantly, our results accommodate \emph{parametric} policy sets that scale to high-dimensional context spaces.}
\dfc{This recovers
results for nonparametric policy sets, but more importantly
accommodates \emph{parametric} policies that scale to high-dimensional
context spaces. }Further, we are able to exploit benign
structure in the policy set and the instance to obtain {better regret} rates.

We also design algorithms that require no knowledge of problem parameters. {Particularly, our algorithm works for all bandwidths $h$ at once, and} is  \emph{optimally adaptive}, matching lower bounds that we prove here. We {accomplish this}
with a unified algorithmic approach.

\begin{table}
\begin{center}
\begin{tabular}{|l|l|l|l|l|l|}
\hline
Type & Setting & Params & Regret Bound & Status & Sec.\\ \hline Smoothed & Worst-case & $h \in (0,1]$ & $\Theta\rbr{\sqrt{\nicefrac{T}{h}}}$ & New &~\sref{sec:smooth_instance}\\
Smoothed & \df{Instance}-dependent & $ h \in (0,1]$ & $O\rbr{\min_{\epsilon} T\epsilon + \theta_h(\epsilon)}$ & New &~\sref{sec:smooth_instance}\\
Smoothed & Adaptive: $h \in (0,1]$ & None & $\Theta\rbr{\sqrt{T}/h}$ & New &~\sref{sec:smooth_adapt}\\
Lipschitz & Worst-case & $L \geq 1$ & $\Theta\rbr{T^{2/3}L^{1/3}}$ & {``Old"} &~\sref{sec:lip_instance}\\
Lipschitz & \df{Instance}-dependent & $L \geq 1$ & $O\rbr{\min_{\epsilon} TL\epsilon + \nicefrac{\psi_{L}(\epsilon)}{L} }$ & {New} &~\sref{sec:lip_instance}\\
Lipschitz & Adaptive: $L \geq 1$ & None & $\Theta(T^{2/3}\sqrt{L})$ & New &~\sref{sec:lip_adapt}\\
\hline
\end{tabular}
\vspace{-0.2cm}
\caption{A summary of results for stochastic contextual bandits,
  specialized to action space $[0,1]$. For notation, $T$
  is the number of rounds, $h$ is the smoothing bandwidth, and
  $\theta_h(\epsilon) \leq 1/(h\epsilon)$ is the \emph{smoothing
    coefficient}. For the Lipschitz results, $L$ is the Lipschitz
  constant and $\psi_L(\epsilon) \leq \nicefrac{1}{\epsilon^2}$ is the
  \emph{policy zooming coefficient}. All algorithms take $T$ and $\Pi$
  as additional inputs. Logarithmic dependence on $|\Pi|$ and $T$ is
  suppressed in all upper bounds.  }
\label{tab:results}
\vspace{-0.5cm}
\end{center}
\end{table}

Our contributions, specialized to the interval $[0,1]$ action space for clarity, are:
\begin{enumerate}\item We define a new notion of \emph{smoothed regret} where policies
  map contexts to distributions over actions. These distributions are parametrized by
  a bandwidth $h$ governing the spread. We show that the optimal
  worst-case regret bound with bandwidth $h$ is
  $\Theta(\sqrt{\nicefrac{T}{h} \log |\Pi|})$, which requires no smoothness assumptions on the losses  (first row
  of~\pref{tab:results}).
\item We obtain \df{instance}-dependent guarantees in terms of a
  \emph{smoothing coefficient}, which can yield much faster rates in
  favorable instances (second row of~\pref{tab:results}).
\item We obtain an adaptive algorithm with $\sqrt{T}/h$
  regret bound for all bandwidths $h$ simultaneously. Further we
  show this to be optimal, demonstrating a price of adaptivity (third row of~\pref{tab:results}).
\end{enumerate}

We obtain analogous results when the losses are $L$-Lipschitz (see rows 3-6 of~\pref{tab:results}). First, we obtain an instance-dependent result {with improved regret rates when near-optimal arms are confined to a relatively small region of the action space. We capture the improvements via a new quantity called the \emph{policy zooming coefficient}, generalizing the \emph{zooming dimension} from prior work on the non-contextual case. Our regret bounds generalize and improve those from prior work on ``zooming'' in Lipschitz bandits, whereby the algorithm gradually ``zooms in'' on more promising regions of the action space. Second, we design an algorithm that adapts to an unknown $L$ and obtain matching lower bounds, thus} demonstrating the ``price of adaptivity'' in the Lipschitz case.

{Our results hold in much more general settings: for higher-dimensional and (almost) arbitrary action spaces and arbitrary smoothing distributions. Our results} also apply to the non-contextual case, where we obtain several new guarantees.

Our algorithms are not computationally efficient, with running times that scale polynomially in $|\Pi|$. The significance lies is in the new conceptual approach and the regret bounds. However, our algorithms \emph{are} computationally efficient in the non-contextual case.

\paragraph{\df{Our techniques.}}
\df{Our core conceptual contribution is the new definition of smoothed
  regret for continuous-action contextual bandits, which, as we have
  mentioned, offers many advantages over previous discretization based
  approaches. While many of our results are based on adapting
  techniques from prior work to the smoothing framework, there are
  many technical challenges that we pause now to highlight.}

\df{Our instance dependent guarantees are based on the \pe algorithm
  of~\citet{dudik2011efficient}, which was originally designed for
  discrete action stochastic contextual bandits. Here we provide a
  refined analysis of this algorithm, showing that it adapts to the
  effective size of the action space, which informally corresponds to
  the number of actions selected by the near-optimal policies. To
  obtain this adaptivity property, we crucially use the
  median-of-means technique to avoid an unfavorable range dependence
  in our estimates of the expected loss of each policy. We believe
  these robust estimation techniques will be broadly useful in other
  bandit settings. Indeed, since the preliminary version of this
  paper, robust estimators have been successfully used
  by~\citet{wei2020taking} to incorporate loss predictors into
  contextual bandit algorithms.}

\df{Our adaptive algorithms are based on aggregating instances of
  \expf~\citep{auer2002nonstochastic} using the \corral algorithm
  of~\citet{agarwal2016corralling}. The key challenge here is that
  \corral can only aggregate over a finite number of base algorithm,
  but we would like our final bound to hold for all bandwidths $h$
  taking continuous values. We address this with a discretization
  argument, using smoothing to show that a single instance of \expf
  obtains the desired guarantee for a small interval of $h$ values,
  which then allows us to use \corral with a finite number of base
  algorithms.}

\paragraph{Roadmap.}
For the majority of the paper, we focus on the setting where the
action space is the unit interval, which simplifies the discussion
while preserving all of the key ideas. The setup and key definitions
are described in~\pref{sec:defns-smooth}. Assumption-free results for
smoothed regret are developed in~\pref{sec:results-smooth} and results
for Lipschitz problems are developed
in~\pref{sec:results-lip}. General theorems extending beyond the unit
interval action space are presented in~\pref{sec:extensions}, where we
also introduce the necessary additional definitions. {The algorithms are analyzed in \pref{sec:proofs-instance-dependent} and \pref{sec:proofs-corral}.}
The lower bounds are presented in~\pref{sec:lower}. We close the paper with
some future directions.
 \section{Related work}
\label{sec:related-work}

With small, discrete action spaces, contextual bandit learning is
quite mature, with rich theoretical results and successful deployments
in practice. To handle large or infinite action spaces, two high-level
approaches exist
{\citep[see books][for surveys and  background]{bubeck2012regret,slivkins-MABbook,lattimore2018bandit}}.
The parametric approach, including work on linear or
combinatorial bandits, posits that the loss is a parametric function
of the action, e.g., a linear function. The
nonparametric approach, which is closer to our results, typically makes much weaker continuity assumptions.\footnote{However, we emphasize that for smoothed regret, we make no assumptions on the loss.}

Bandits with Lipschitz assumptions were introduced
in~\citet{agrawal1995continuum}, and optimally solved in the worst
case by \citet{Bobby-nips04}. \citet{LipschitzMAB-stoc08,kleinberg2013bandits,bubeck2011x}
achieve data-dependent regret bounds via {``zooming'' algorithms which gradually} ``zoom in''
on the more promising regions of the action space. {\citet{LipschitzMAB-stoc08,kleinberg2013bandits,DichotomyMAB-soda10} consider regret rates with instance-dependent constant,
analogous to the well-known $\log(t)$ instance-dependent rates for finitely many arms,
and use zooming algorithms to characterize the corresponding worst-case optimal regret rates for any given metric space.} \asedit{Further work focused on relaxing the smoothness assumptions and adapting to unknown smoothness parameters, as well as extensions to contextual bandits
\citep[see Ch. 4][for  a more comprehensive background]{slivkins-MABbook}.}

Several papers
relax global smoothness assumptions with various local
definitions~\citep{auer2007improved,LipschitzMAB-stoc08,kleinberg2013bandits,bubeck2011x,ImplicitMAB-nips11,valko2013stochastic,minsker2013estimation,grill2015black,shang2019general}. While
the assumptions \df{and results} vary, our smoothing-based approach can be used in many
of these settings. More importantly, in contrast with these
approaches, our guarantees remain meaningful even in pathological
instances, for example when the global optimum is a discontinuity as
in the top panel of~\pref{fig:examples} (See~\pref{ex:discontinuous}).

While most of this literature focuses on the non-contextual version,
three papers consider contextual settings, albeit only with fixed policy
sets $\Pi$.  \asedit{\citet{Pal-Bandits-aistats10} and} \cite{slivkins2014contextual} posit that the mean loss
function is Lipschitz in both context $x$ and action $a$ and the
learner must compete with the best mapping from $\Xcal$ to
$\Acal$. \asedit{While \citet{Pal-Bandits-aistats10} focus on worst-case regret bounds, the algorithm and guarantees in \cite{slivkins2014contextual}} exhibit ``zooming'' behavior in
the action space, which is qualitatively similar to ours. However, his
regret bound also has a zooming-dependence on the context dimension,
whereas our regret bound applies to arbitrary policy sets and defines
packing numbers via expectation over contexts rather than supremum.
\citet{cesa2017algorithmic} competes with policies that are themselves
Lipschitz (w.r.t. a given metric on contexts). We can recover their
result via~\pref{corr:stoch_ring_lip_instance} and a suitable
discretized policy set.

Turning to adaptivity,~\citet{bubeck2011lipschitz} develops an
algorithm that adapts to the Lipschitz constant in the non-contextual
setting given a bound on the second derivative.
\citet{locatelli2018adaptivity} obtain optimal adaptive
algorithms, but require knowledge of either the value of the minimum,
or a sharp bound on the achievable
regret.
\citet{ImplicitMAB-nips11,Bull-bandits14} achieve optimal regret bounds in terms of the zooming dimension, but their regret bounds depend on a certain ``quality parameter.''
\df{A line of work studying the non-contextual setting~\citep{valko2013stochastic,grill2015black,shang2019general}, establishes adaptive guarantees when performance is measured in terms of optimization error, which is the difference between the best action selected and the globally optimal action. However, these results do not translate to our performance measure, cumulative regret.}
Moreover, \df{all of the above} results concern the stochastic setting, while our optimally adaptive guarantees carry through to the adversarial setting.
\citet{locatelli2018adaptivity} also
obtain lower bounds against adapting to the smoothness exponent, and
we build on their construction for our lower bounds.

\df{A parallel line of work on Bayesian optimization, considers the
  related problem of maximizing either a sample from a Gaussian
  process, or a function with bounded norm in some Reproducing Kernel
  Hilbert Space (RKHS)~\citep{srinivas2012information}. The conceptual
  difference with our work is that these results impose regularity
  assumptions on the problem, in the same vein as prior work with
  Lipschitz assumptions, while we make no assumptions and instead
  provide guarantees in terms of smoothed regret. On the more
  technical side, ~\citet{krause2011contextual} consider a contextual
  Bayesian optimization setting where there is a kernel over the joint
  context-action space, which is analogous to the Lipschitz contextual
  bandits setting studied by~\citet{slivkins2014contextual}. As
  mentioned above, these results consider a specific ``nonparametric''
  policy set, while our results apply to arbitrary policy
  sets.~\citet{berkenkamp2019no} establish adaptive guarantees for
  Bayesian optimization, but they obtain incomparable results using
  very different techniques from ours.}

Finally, our smoothing-based importance weighted loss
estimator~\pref{eq:ips_estimate} was analyzed
by~\citet{kallus2018policy, chen2016personalized} in the related
off-policy evaluation problem, but they do not consider the smoothed
regret benchmark or the online setting, so the results are
considerably different.  We also use the median-of-means approach from
robust statistics --- specifically a result of~\citet{hsu2016loss} ---
to avoid an unfavorable range dependence in our loss estimator.
This
estimator has been used by~\citet{sen2018contextual} for contextual
bandits with discrete actions, but their results are incomparable to
ours.

 \section{Smoothed regret}
\label{sec:defns-smooth}

We work in a standard setup for stochastic contextual bandits. We have
a context space $\Xcal$, action space $\Acal$, a (possibly large but finite) policy set $\Pi: \Xcal \to
\Acal$, and a distribution $\Dcal$ over context/loss pairs
    $\Xcal \times \{\text{functions } \Acal \to [0,1]\}$.
The protocol proceeds
for $T$ rounds where in each round $t$: (1) nature samples
$(x_t,\ell_t) \sim \Dcal$; (2) the learner observes $x_t$ and chooses
an action $a_t\in \Acal$; (3) the learner suffers loss $\ell_t(a_t)$,
which is observed.  For simplicity, we focus on the case when \df{$\Dcal_X$,} the
marginal distribution \df{of $\Dcal$} over $\Xcal$ is known.\footnote{We mention how this can be relaxed in the next section.}
The learner's goal is to minimize
regret relative to the policy class.

\paragraph{Key new definitions.}

We depart from the standard setup by positing a \emph{smoothing operator}
    \[ \smooth_h : \Acal \to \Delta(\Acal),\]
where \df{$\Delta(\Acal)$ is the set of probability distributions over $\Acal$ and} $h \geq 0$ is the \emph{bandwidth}: a parameter that determines the spread of the distribution.\footnote{The term \emph{bandwidth} here is in line with the nonparametric statistics literature.}
Bandwidth $h=0$ corresponds to the Dirac distribution.
Each action $a$ then maps to the \emph{smoothed  action}
    $\smooth_h(a)$,
and each policy $\pi\in \Pi$ maps to a randomized \emph{smoothed policy}
    $\smooth_h (\pi): x \mapsto \smooth_h(\pi(x)) $.
We compete with the \emph{smoothed policy class}
    \[ \Pi_h \defeq \{\smooth_h(\pi):\; \pi\in \Pi\} .\]

\noindent We then define the \emph{smoothed loss} of a given policy $\pi \in \Pi$ and the \emph{benchmark} optimal loss as
\begin{align}\label{eq:defn-smooth-loss}
    \ExpL_h(\pi) \defeq \EE_{(x,\ell)\sim \Dcal}\; \EE_{a \sim \smooth_h(\pi(x))} \sbr{ \ell(a) }, \quad\textrm{and} \quad  \df{\bench(\Pi_h) \defeq \inf_{\pi \in \Pi} \lambda_h(\pi) = \inf_{\pi \in \Pi_h} \lambda_0(\pi)}.
\end{align}
\df{Note that there is a duality between smoothing the policy class and smoothing the loss function, as $\lambda_h(\pi) = \lambda_0(\pi_h)$. }
We are interested in \emph{smoothed regret}, which compares the learner's total loss against the benchmark:
\begin{align*}
  \Reg(T,\Pi_h) \defeq  \textstyle \EE\sbr{ \sum_{t=1}^T \ell_t(a_t) }
    - T\cdot\bench(\Pi_h).
\end{align*}
Our regret bounds work for an arbitrary policy set $\Pi$, leaving the choice of $\Pi$ to the practitioner. For comparison, a standard benchmark for contextual bandits is $\bench(\Pi)$, the best policy in the original policy class $\Pi$, and one is interested in $\Reg(T,\Pi)$.

For the first several sections of the paper,
we posit that the actions set is a unit interval:
    $\Acal \defeq [0,1]$,
endowed with a metric
    $\rho(a,a') \defeq |a-a'|$.
    $\smooth_h(a)$
is defined as a uniform distribution over the closed ball
    $\ball_h(a) \defeq \{a'\in \Acal:\; \rho(a,a')\leq h\} = [a-h,a+h]\cap [0,1]$.
Let $\base$ denote the Lebesgue measure, which corresponds to the uniform distribution over $[0,1]$.
\df{As notation,
$\smooth_{\pi,h}(a|x)$ is the probability density, w.r.t., $\base$, for
$\smooth_h(\pi(x))$ at action $a$.}
In~\pref{sec:extensions} we present results that apply to a more general setting where the action space $\Acal$ is embedded in some ambient space and the smoothing operator is given by a \df{probability} kernel\dfc{ function}.
However, all of the key ideas appear in the case of the unit interval.

For some intuition, the bandwidth $h$ governs a bias-variance tradeoff
inherent in the continuous-action setting:
for small $h$ the smoothed loss $\lambda_h(\pi)$
closely approximates the true loss $\lambda_0(\pi)$, but small $h$
also admits worse smoothed regret guarantees.

\begin{example}
The well-studied non-contextual version of the problem fits into our
framework as follows: there is only one context $\Xcal \defeq \{x_0\}$ and
policies are in one-to-one correspondence with actions: $\Pi \defeq
\cbr{x_0 \mapsto a: a \in \Acal}$. A problem instance is characterized
by the expected loss function $\lambda_0(a) \defeq \EE[\ell(a)]$ and the
smoothed benchmark is simply $\bench(\Pi_h) \defeq \inf_{a \in \Acal}
\lambda_h(a)$.
\label{ex:non_contextual}
\end{example}

\begin{figure}
\captionsetup{format=plain}
\begin{minipage}[c]{0.5\textwidth}
\includegraphics[width=\textwidth]{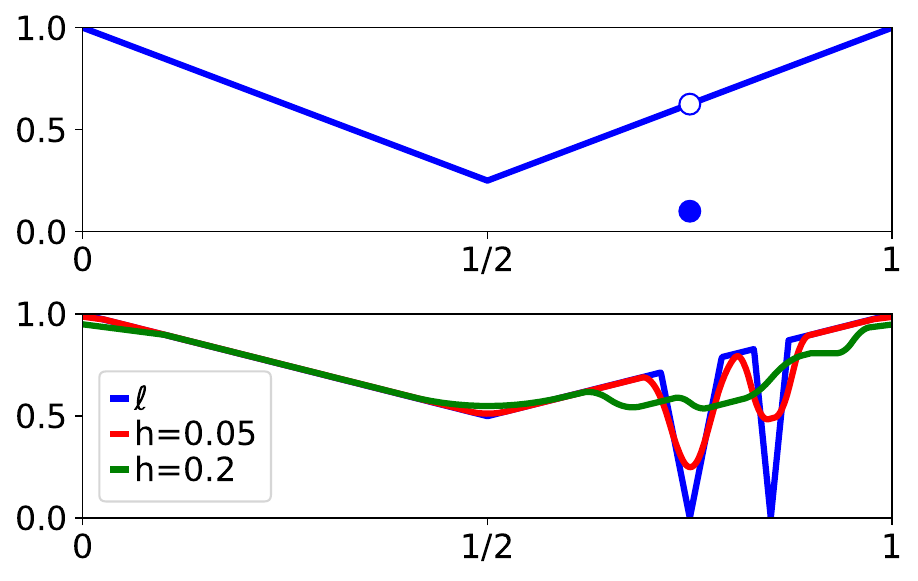}
\end{minipage}
\begin{minipage}{0.5\textwidth}
\vspace{0.5em}\caption{ The discontinuous function
  in~\pref{ex:discontinuous}. Smoothed regret provides a meaningful
  guarantee, competing with $a_h^\star = 1/2$. \vspace{2.5em}
  \\ The loss function (in blue) has large Lipschitz constant
  and ``needles" that are hard to find. Smoothing with small bandwidth
  does not change the optimum while a large bandwidth
  does.}\label{fig:examples}
\end{minipage}
\vspace{-0.75cm}
\end{figure}

Smoothing the policy class enables meaningful guarantees in much more
general settings than prior work assuming global continuity (e.g.,
Lipschitzness). \df{Our results require no smoothness assumptions on
  the loss function, in the spirit of the assumption-free analyses
  typical in the online learning literature. Our smoothed regret
  guarantees can be translated to standard regret bounds under
  significantly weaker assumptions than global smoothness; for example
  smoothness around the actions taken by the optimal policy suffices.} \df{Moreover, } the
guarantees remain meaningful even when the expected loss function has
discontinuities, as demonstrated by the following example.
\begin{example}
\label{ex:discontinuous}
Consider a family of non-contextual settings with expected loss function
\begin{align*}
 \ExpL_0(a) =
\left(\nicefrac{1}{4}+1.5\, \rho(a,\nicefrac{1}{2})\right) \cdot
\indicator{a \,\ne\, a'} +
\nicefrac{1}{10}\cdot\indicator{a \,=\, a'},
\quad \asedit{a' \in [0,1]}
\end{align*}
(see~\pref{fig:examples}). The optimal action $a^\star=
a'$ cannot be found in finitely many rounds due to
the discontinuity, so any algorithm is doomed to linear regret. However,
the smoothed loss function $\ExpL_h$ for any $h > 0$ essentially
ignores the discontinuity (and is minimized at $a^*_h =
\nicefrac{1}{2}$). Accordingly, as we shall prove, it admits
algorithms with sublinear smoothed regret.
\end{example}

\asedit{While the above example is pathological, discontinuous loss
  functions are common in applications. One generic example is, when the algorithm controls the system parameters in a computer or a data center, even a small change can make a large difference when resources are close to saturation. For a more mathematically concrete example, consider 
  the well-studied dynamic pricing problem~\citep{kleinberg2003value}, where the algorithm is a seller with an infinite inventory of identical goods. In each round the algorithm sets a price $p_t\in [0,1]$ for an item, a buyer arrives with value $v_t\in[0,1]$, and  purchases the item if only if $p_t \leq v_t$. The algorithm's goal is to maximize\footnote{To reformulate the problem in terms of losses, posit
    $\ell(p_t,v_t) = v_t - p_t\cdot \indicator{p_t\leq v_t}$.}
 the total revenue,
    $\sum_{t=1}^T\; p_t\cdot \indicator{p_t\leq v_t}$.
                  So, we have a discontinuity at $v_t=p_t$, even though the payoffs are $1$-Lipschitz everywhere else. More complex discontinuity structures can arise if the algorithm is selling multiple products at once, as the buyers can switch from one product to another.}

The bottom panel of~\pref{fig:examples} provides further intuition for the $\smooth_h$ operator.

\df{
\paragraph{Adversarial losses.}
Some of our results carry over as is to the adversarial setting in
which the context-loss pairs are chosen by an adaptive adversary. The benchmark is redefined as
\begin{align*}
\bench(\Pi_h) \defeq \tfrac{1}{T}\;\textstyle \inf_{\pi\in\Pi_h} \EE\sbr{\sum_{t\in [T]} \ell_t(\pi(x_t))}.
\end{align*}
where the expectation accounts for any randomness. We will always
explicitly specify which results apply to this setting.}

\df{
\paragraph{Additional notation.}
We use $\EE_{x \sim \Dcal_X}\sbr{\cdot}$ to denote expectation over the
marginal distribution over contexts. We use the standard big-Oh
notation and use the notation $g = \otil(f)$ to denote that $g =
\order(f \cdot \textrm{polylog}(f))$.}

 \section{Smoothed regret guarantees}
\label{sec:results-smooth}
In this section we obtain smoothed-regret guarantees without imposing any
continuity assumptions on the problem.

\subsection{\df{Instance}-dependent and worst-case guarantees}
\label{sec:smooth_instance}
Our first result is an \df{instance}-dependent smoothed regret bound for a given
bandwidth $h \geq 0$.

An important part of the contribution is setting up the definitions.
Recall the definition of the smoothed loss $\lambda_h(\cdot)$ \asedit{and optimal smooth loss $\bench(\Pi_h)$}
from~\pref{eq:defn-smooth-loss}.
The
\emph{version space} of $\epsilon$-optimal policies (according to the smoothed loss) is
\begin{align*}
\Pi_{h,\epsilon} \defeq \cbr{\pi \in \Pi:\; \lambda_h(\pi) \leq \bench(\Pi_h) + \epsilon}.
\end{align*}
For a given context $x \in \Xcal$, a policy subset $\Pi' \subset \Pi$
maps to an action set $\Pi'(x) \defeq \cbr{\pi(x):\; \pi \in \Pi'}$. We
are interested in $\Pi_{h,\epsilon}(x)$, the subset of actions chosen
by the $\epsilon$-optimal policies on context $x$, and specifically the
expected packing number of this set:
\begin{align}
M_h(\epsilon,\delta) \defeq \EE_{x \sim \Dcal} \sbr{\Ncal_\delta\rbr{\Pi_{h,\epsilon}(x)}},
\label{eq:packing_def}
\end{align}
where $\Ncal_{\delta}(A)$ is the $\delta$-packing number of subset $A\subset \Acal$ in the ambient metric space $(\Acal,\rho)$.\footnote{A subset $S$ of a set $A$ is a $\delta$-packing if any
  two points in $S$ are at a distance of at least $\delta$. The
  $\delta$-packing number of a set $A$ is the maximum
  cardinality of a $\delta$-packing of $A$.}  The \emph{smoothing coefficient} $\theta_h : \RR
\to \RR$ measures how the packing numbers
$M_h(12\epsilon,h)$ shrink with $\epsilon$:
\begin{align}\label{eq:smoothing-coeff}
\theta_h(\epsilon_0) \defeq \sup_{\epsilon \geq \epsilon_0} M_h(12\epsilon,h)/\epsilon.
\end{align}
For the unit interval, observe that $\theta_h(\epsilon_0) \leq
(h\epsilon_0)^{-1}$ always, but in favorable cases we might expect
$\theta_h(\epsilon_0) \leq \max\{\nicefrac{1}{h},
\nicefrac{1}{\epsilon_0}\}$, as demonstrated by the following
example. \df{Note that the constant $12$ is not fundamental, but is
  consistent with prior work on instance-dependent guarantees for continuous action spaces~\citep{slivkins2014contextual}.}

\begin{example}[Small smoothing coefficient]
\label{ex:small_smoothing}
Consider a non-contextual problem, where the expected loss function is
$\lambda(a) \defeq \EE[\ell(a)\mid x_0] = \abr{a - a^\star}$ for some
$a^\star \in [2h,1-2h]$.  Then $M_h(\epsilon, h) \leq O(\max\cbr{1,
  \nicefrac{\epsilon}{h}})$. Consequently, $\theta_h(\epsilon_0) \leq
O(\max\cbr{\nicefrac{1}{h}, \nicefrac{1}{\epsilon_0}})$. (See~\pref{sec:calculations} \df{for a derivation}.)
\end{example}

Our first result is in terms of this smoothing coefficient.
\begin{theorem}
\label{thm:stoch_ring_weak}
For any given bandwidth $h >0$, \df{in the stochastic setting,} \smoothpe (\pref{alg:ZPE-body}) with
parameter $h$ achieves
\begin{align*}
\Reg(T,\Pi_h)
\leq O\rbr{\inf_{\epsilon_0 > 0}
\cbr{ T\epsilon_0 + \theta_h(\epsilon_0)\;\log(|\Pi|\df{T})\;\log(1/\epsilon_0) } }.
\end{align*}
\end{theorem}

Since $\theta_h(\epsilon_0) \leq (h\epsilon_0)^{-1}$, we obtain a
worst case guarantee as a corollary.
\begin{corollary}
\label{corr:stoch_ring_instance}
Fix any bandwidth $h > 0$, \df{in the stochastic setting,} \smoothpe with parameter $h$ achieves
\begin{align*}
\Reg(T,\Pi_h) \leq \tilde{O} \rbr{\sqrt{\nicefrac{T}{h} \log|\Pi| }}.
\end{align*}
\end{corollary}

\asedit{Contrasting with} the standard $\Theta(\sqrt{T |\Acal| \log |\Pi|})$ regret bound for finite action spaces, \asedit{we see} that the $\nicefrac{1}{h}$ term can be viewed as the effective number of actions.

In fact, this worst case bound can also be achieved by a simple
variation of \expfour~\citep{auer2002nonstochastic}, which can operate
in the adversarial version of our problem and actually achieves
$O(\sqrt{\nicefrac{T}{h}\log |\Pi|})$ regret, eliminating the
logarithmic dependence on $T$. The pseudocode for this algorithm is
displayed in~\pref{alg:smooth_exp4}.
\begin{theorem}
\label{thm:smooth_exp4}
In the adversarial setting, \smoothexp with policy set $\Xi = \Pi_h$ \asedit{and learning rate
    $\eta = \sqrt{\frac{2\,h\;\ln|\Xi|}{T}}$}
achieves
$\Reg(T,\Pi_h) \leq
\order\rbr{\sqrt{\nicefrac{T}{h}\log|\Pi|}}$ .
\end{theorem}

\noindent \asedit{Both algorithms are not computationally efficient in general, as the per-round running time scales as $|\Pi|$. For the non-contextual case, one can take $|\Pi| = \nicefrac{T}{h}$, see \pref{sec:cases}(c).}

\paragraph{Remarks.}

It is not hard to show a $\Omega(\sqrt{\nicefrac{T}{h}\log |\Pi|})$
lower bound on smoothed regret. Specifically, every $K$ arm contextual
bandit instance can be reduced to a continuous action instance with
bandwidth $h=1/(2K)$ by using piecewise constant loss functions and by
mapping actions $a \in \{1,\ldots,K\}$ to $h\cdot(2a-1)$. Thus, we may
embed the lower bound construction for contextual bandits with finite
action space into our setup to verify
that~\pref{corr:stoch_ring_instance} is optimal up to logarithmic
factors (and~\pref{thm:smooth_exp4} is optimal up to constants).

While not technically very difficult, the worst-case bound showcases
the power and generality of the new definition.  In particular, we
obtain meaningful guarantees for discontinuous losses as
in~\pref{ex:discontinuous}. As we will see in the next section, under
global smoothness assumptions, we can also obtain a bound on the
more-standard quantity $\Reg(T,\Pi)$.

Turning to the instance-specific bound in~\pref{thm:stoch_ring_weak},
we obtain a more-refined dependence on the effective number of actions
$\nicefrac{1}{h}$, which can be thought of as a ``gap-dependent'' bound.
In the most favorable setting, we have
$\theta_h(\epsilon_0) =
\max\cbr{\nicefrac{1}{h},\nicefrac{1}{\epsilon_0}}$ which yields
$\Reg(T,\Pi_h) \leq \tilde{O}\rbr{\sqrt{T \log |\Pi|} +
  \frac{1}{h}\log |\Pi|}$, eliminating the dependence on $h$ in the
leading term (Recall that \pref{ex:small_smoothing} has this favorable
behavior). Further, via the correspondence with the finite action
setting, we also obtain a new \df{instance}-dependent bound for standard
stochastic contextual bandits, which improves on prior worst case results by
adapting to the effective size of the action
space~\citep{dudik2011efficient,agarwal2014taming}. This result for
the finite-action setting follows from our more general theorem
statement, given in~\pref{sec:extensions}.

\df{We also note that, while smoothing induces a Lipschitz loss
  function, a na\"{i}ve application of a Lipschitz bandits algorithm
  yields a suboptimal regret rate. For example, in the non-contextual
  version, the smoothed loss function is $\lambda_h: a \mapsto
  \EE_{\Dcal} \EE_{a' \sim \smooth_h}\sbr{\ell(a')}$, is
  $\nicefrac{1}{h}$-Lipschitz, so we may apply a Lipschitz bandits
  algorithm in a black box fashion.\footnote{\df{Formally, when the
      Lipschitz bandits algorithm recommends action $a_t'$, we sample
      $a_t \sim \smooth_h(a_t')$, observe $\ell_t(a_t)$ --- which has
      expectation $\lambda_h(a_t')$ --- and pass this value back to
      the algorithm.}} However, this reduction gives a smoothed regret
  bound of $O(T^{2/3}h^{-1/3})$, which is suboptimal when compared
  with our $\tilde{O}(\sqrt{\nicefrac{T}{h}})$ result. Our guarantees exploit
  additional information sharing between actions enabled by the
  smoothing operator, in particular the fact that when we choose a
  particular action, we learn about all smoothed actions in an
  interval of size $h$.}

Finally, we remark that~\pref{alg:ZPE-body} actually achieves a high
probability regret bound, which we have simplified to the stated expected
regret bound.

\begin{algorithm}[t]
\begin{algorithmic}
\State {\bf Parameters}: Bandwidth $h>0$, policy set $\Pi$, number of rounds $T$.
\State {\bf Initialize}: $\Pi^{(1)} = \Pi$, Batches $\delta_T = 5\lceil \log(T|\Pi|\log_2(T)) \rceil$, Radii $r_m = 2^{-m}, m=1,2,\ldots$.
    \For{each epoch $m = 1,2,\ldots$}
\State // Before the epoch: compute distribution $Q_m$ over policy set $\Pi^{(m)}$.
\State Set $V_m \gets \EE_{x \sim \Dcal}\base\rbr{\bigcup_{\pi \in \Pi^{(m)}}\ball_h(\pi(x))}$
    \hspace{1cm} // \emph{characteristic volume} of $\Pi^{(m)}$
\State Set batch size $\tilde{n}_m = \frac{320 V_m}{r_m^2 h}$, epoch length $n_m = \tilde{n}_m \delta_T$.
\State Find distribution $Q_m$ over policy set $\Pi^{(m)}$ which minimizes
\begin{align}
    & \max_{\text{policies }\pi \in \Pi^{(m)}} \quad
    \EE_{\text{context } x \sim \df{\Dcal_X}} \quad
    \EE_{\text{action }a\sim \smooth_{h}(\pi(x))}\;
        \sbr{\frac{1}{q_m(a \mid x) }},
        \label{eq:pe_optimization-body}\\     & \textrm{\df{where density} } q_m(a \mid x) \defeq \EE_{\pi \sim Q_m} \smooth_{\pi,h}(a|x) .\notag
\end{align}
\vspace{-0.25cm}
\For{each round $t$ in epoch $m$}
\State Observe context $x_t$,
    sample action $a_t$ from \df{density} $q_m(\cdot \mid x_t)$,
    observe loss $\ell_t(a_t)$.
\EndFor
\State // After the epoch: update the policy set.

\For{each batch $i =1,2,\ldots,\delta_T$}
\State Define $S_{i,m}$ as the indices of the
$(i-1)\tilde{n}_m+1, \ldots, i \tilde{n}_m^{\textrm{th}}$ examples in epoch $m$.
\State Estimate $\lambda_h(\pi)$ with
    $ \hat{L}_m^i(\pi) =
        \frac{1}{\tilde{n}_m}\sum_{t\in S_{i,m}} \hat{\ell}_{t,h}(\pi)$
for each policy $\pi\in\Pi^{(m)}$ where
\begin{align}
 \hat{\ell}_{t,h}(\pi) \defeq
    \tfrac{\smooth_{\pi,h}(a_t|x_t) \; \ell_t(a_t)}{q_m(a_t \mid x_t)}. \label{eq:ips_estimate}
 \end{align}
\EndFor
\State Estimate the loss $\hat{L}_m(\pi) = \median \rbr{\hat{L}_m^1(\pi), \hat{L}_m^2(\pi), \ldots,
\hat{L}_m^{\delta_T}(\pi)}$.

\State $\Pi^{(m+1)} = \cbr{\pi \in \Pi^{(m)} : \hat{L}_m(\pi) \leq \min_{\pi' \in \Pi^{(m)}} \hat{L}_m(\pi') + 3\,r_m}$.
\EndFor
\end{algorithmic}
\caption{\smoothpe}
\label{alg:ZPE-body}
\end{algorithm}

\paragraph{The algorithm.}
The algorithm is an adaptation of \pe from~\citet{dudik2011efficient},
with pseudocode displayed in~\pref{alg:ZPE-body}. It is epoch based,
maintaining a version space of good policies, denoted $\Pi^{(m)}$ in
the $m^{\textrm{th}}$ epoch, and pruning it over time by eliminating
the provably suboptimal policies. In the $m^{\textrm{th}}$ epoch,  the algorithm computes a distribution $Q_m$ over
$\Pi^{(m)}$ by solving a convex
program~\pref{eq:pe_optimization-body}.  The objective function is
related to the variance of the loss estimator we use, and so $Q_m$
ensures high-quality loss estimates for all policies in $\Pi^{(m)}$.
We use $Q_m$ to select actions at each round in the epoch by sampling
$\pi \sim Q_m$ and playing $\smooth_h(\pi(x))$ on context $x$.
 To compute $\Pi^{(m+1)}$ for the next epoch, we use importance
 weighting to form single-sample unbiased estimates for
 $\lambda_h(\pi)$ in~\pref{eq:ips_estimate}, and we aggregate these
 via a median-of-means approach (see
 e.g.,~\citet{hsu2016loss}). $\Pi^{(m+1)}$ is then defined as the set
 of policies with low empirical regret measured via the
 median-of-means estimator. Na\"{i}vely, the running time is
 $\textrm{poly}(T,|\Pi|)$.\footnote{For the non-contextual case, the algorithm simplifies and the running time becomes $\textrm{poly}(T)$, see also \pref{sec:cases}(c).}

The key changes over \pe are as follows.  First, we
write~\pref{eq:pe_optimization-body} as an optimization problem rather
than a feasibility problem, which allows for \df{instance}-dependent
improvements in our loss estimates. Second, our importance weighting
crucially exploits smoothing for low variance. Finally, we employ the
median-of-means estimator to eliminate an unfavorable range dependence
with importance weighting. The immediate consequence of this estimator
is that we can eliminate the need for uniform exploration, which appears
in prior literature on contextual bandits with finite action
spaces~\cite[e.g.][]{dudik2011efficient,agarwal2014taming}.
Perhaps more
interestingly, the median-of-means estimator is unnecessary
for~\pref{corr:stoch_ring_instance} and for prior results with finite
action spaces, but it is crucial for obtaining our \df{instance}-dependent bound,
since we need the error of our loss estimator to scale with the
characteristic volume \df{$V_m \defeq \EE_{x \sim \Dcal}\base\rbr{\bigcup_{\pi \in \Pi^{(m)}}\ball_h(\pi(x))}$}.

As we have described the algorithm, it requires knowledge of the
marginal distribution over $\Xcal$, which appears in the computation
of $V_m$ and in the optimization problem. Both of these can be
replaced with empirical counterparts, and since the random variables
are non-negative, via Bernstein's inequality, the approximation only
affects the regret bound in the constant factors. This argument has
been used in several prior contextual bandit
results~\citep{dudik2011efficient,agarwal2014taming,krishnamurthy2016contextual},
and so we omit the details here.

For the proof, we first use convex duality to upper bound the value
of~\pref{eq:pe_optimization-body} in terms of the characteristic volume $V_m$,
refining~\citet{dudik2011efficient}. As the objective divided by $h$
bounds the variance of the importance weighted estimate
in~\pref{eq:ips_estimate}, we may use Chebyshev and Chernoff bounds to
control the error of the median-of-means estimator in terms of $V_m,h$,
and $n_m$. Our setting of $n_m$ then implies that $\Pi^{(m+1)} \subset
\Pi_{h,12r_{m+1}}$. Two crucial facts follow: (1) the instantaneous
regret in epoch $m+1$ is related to $r_{m+1}$ and (2) $V_{m+1}$, which determines the length of the
epoch, is related to the packing number $M_h(12r_{m+1},h)$. Roughly
speaking, this shows that the regret in epoch $m$ is $n_m r_m \lesssim
M_h(12r_m,h)/r_m$, which we can easily relate to the smoothing
coefficient. 
\begin{algorithm}[t]
\begin{algorithmic}
\State {\bf Parameters:}
    Collection of randomized policies $\Xi$,
    learning rate $\eta >0$.
    \State // $\xi(\cdot \mid x_t)$ is the probability density for policy $\xi$ given context $x_t$.
\State {\bf Initialization:}
    weights $W_1(\xi) \gets 1$ for all policies $\xi \in \Xi$.
\For{$t=1,\ldots,T$}
\State Sample policy $\xi_t  \propto W_t$,
    sample action $a_t$ from $\xi_t(\cdot \mid x_t)$.
\State // $p_t(\cdot \mid x_t)$ is
    the probability density for action $a_t$ given context $x_t$.
\State Observe loss $\ell_t(a_t)$ and define
\begin{align*}
\hat{\ell}_t(\xi) \defeq \frac{\xi(a_t\mid x_t)}{p_t(a_t \mid x_t)}\cdot\ell_t(a_t).
\end{align*}
\State Update weights:
    $W_{t+1}(\xi) \gets W_t(\xi)\cdot\exp(-\eta\hat{\ell}_t(\xi))$.
\EndFor
\end{algorithmic}
\caption{\smoothexp: \expfour with continuous sampling}
\label{alg:smooth_exp4}
\end{algorithm}

\subsection{One algorithm for all $h$}
\label{sec:smooth_adapt}
\smoothpe guarantees a refined regret bound against $\bench(\Pi_h)$
for a given $h > 0$. Yet choosing the bandwidth in practice seems
challenging: since $\bench(\Pi_h)$ is unknown and not monotone in
general, there is no a priori way to choose $h$ to minimize the
benchmark plus the regret.  As such, we seek algorithms that can
achieve a smoothed regret bound simultaneously for all bandwidths $h$,
a guarantee we call \emph{uniformly-smoothed}.  This is achieved by our next result.

\begin{theorem}
\label{thm:smooth_adaptive}
Consider the adversarial setting.
For each parameter $\beta \in [0,1]$, \df{there exists an algorithm that guarantees}
\begin{align*}
\forall h \in (0,1]: \Reg(T,\Pi_h) \leq \tilde{O}\rbr{T^{\frac{1}{1+\beta}} h^{-\beta}} \cdot (\log |\Pi|)^{\frac{\beta}{1+\beta}}.
\end{align*}
For the non-contextual setting, it achieves a uniformly-smoothed regret of
$\tilde{O}\rbr{T^{\frac{1}{1+\beta}} h^{-\beta}}$. Moreover, for the non-contextual \df{stochastic} setting,
\df{there exist positive constants $c$ and $T_0$, such that for any algorithm and any $T \geq T_0$, there exists $h \in (0,1]$ and a problem instance, such that on this instance,
\begin{align*}
\Reg(T,\Pi_h) \geq c \cdot T^{\frac{1}{1+\beta}} h^{-\beta}.
\end{align*}
}
\end{theorem}

\paragraph{Remarks.}
The theorem provides a family of upper and lower bounds, one for each
$\beta \in [0,1]$. As two examples, taking $\beta=1$ we obtain
regret rate $\tilde{O}(\sqrt{T}/h)$ as listed in the third row
of~\pref{tab:results}, while $\beta=\nicefrac{1}{2}$ yields
$\tilde{O}(T^{2/3}/\sqrt{h})$. These bounds are incomparable in
general and so the result establishes a Pareto frontier of exponent
pairs. In the non-contextual setting, all pairs are optimal, and, in
particular, the $\sqrt{\nicefrac{T}{h}}$ rate
from~\pref{corr:stoch_ring_instance} is not achievable uniformly. More
generally, the optimal uniformly-smoothed regret bounds
are very different from those for a fixed bandwidth.

Note that while $\beta$ is a parameter to the algorithm, it simply
governs where on the Pareto frontier the algorithm lies, and is not
based on any property of the problem.

\paragraph{The algorithm.}
The algorithm\df{, \corralexpf, } is an instantiation
of~\corral~\citep{agarwal2016corralling}, which can be used to run
many sub-algorithms in parallel. \corral maintains a master
distribution over sub-algorithms, and in each round it samples a
sub-algorithm and chooses the action the sub-algorithm
recommends. \corral sends an importance weighted loss (weighted by the
master distribution) to all the sub-algorithms and it updates the
master distribution using online mirror descent with the log-barrier
mirror map.

For the sub-algorithms we use our variant of \expfour. Each
sub-algorithm instance operates with a different bandwidth scale, and
if run in isolation achieves the optimal non-adaptive smoothed regret
for that bandwidth. Aggregating these sub-algorithms with \corral
yields the uniformly-smoothed guarantee. Note that here and elsewhere,
\corral results in a worse overall regret than the best individual
sub-algorithm, but in our setting it nevertheless achieves all
Pareto-optimal uniformly-smoothed guarantees. We describe \corralexpf
formally in Section~\ref{sec:corralexpf-analysis}.

The proof for the upper bound involves a more refined analysis for
\expfour than required for~\pref{thm:smooth_exp4}. First, we
discretize bandwidths to multiples of $1/T^2$ and show that, for any
$i \in \NN$, a single instance of \expfour using discretized
bandwidths can compete with all $h \in [2^{-i},2^{-i+1}]$
simultaneously, without using \corral. Second, we show that \expfour
is stable in the sense that, in randomized environments, the regret
scales linearly with the standard deviation of the losses and that
this standard deviation need not be known a priori.\footnote{This
  property was shown by~\citet{agarwal2016corralling}, but our variant
  of \expfour is necessarily slightly different. Nevertheless, the
  proof is quite similar.}  Stability is crucial for aggregating with
\corral as the master's importance weighting induces high-variance
randomized losses for each sub-algorithm. We finish the proof by
applying the guarantee for \corral~\citep{agarwal2016corralling} with
$\log(T)$ instances of \expfour as sub-algorithms, one for each
bandwidth scale $[2^{-i},2^{-i+1}]$. For each $\beta \in [0,1]$, we
use a weakening of the \expfour regret guarantee, essentially that
$\min\cbr{\sqrt{\nicefrac{T}{h}},T} \leq
T^{\frac{1}{1+\beta}}h^{-\frac{\beta}{1+\beta}}$ for all $\beta \in
[0,1]$.

The lower bound is inspired by a construction
of~\citet{locatelli2018adaptivity}. We show that if an algorithm,
\alg, has small regret against $\bench(\Pi_{\nicefrac{1}{4}})$, then
it must suffer large regret against $\bench(\Pi_{h})$ for
$h\ll\nicefrac{1}{4}$. The intuition is that the
$\nicefrac{1}{4}$-smoothed regret bound prevents \alg from
sufficiently exploring.  Specifically, we construct one instance where
small losses occur in a subinterval $I_0\subset [0,1]$ of length
$\nicefrac{1}{4}$ and another that is identical on $I_0$ but where
even smaller losses occur in a subinterval $I_1$ of width $h \ll
\nicefrac{1}{4}$. Since \alg has low $\nicefrac{1}{4}$-smoothed regret
it cannot afford to explore to find $I_1$.
In comparison
with~\citet{locatelli2018adaptivity}, the details of the construction
are somewhat different, since they focus on adaptivity to unknown
smoothness exponent, while we are adapting to bandwidth $h$ (and later
to unknown Lipschitz constant).
 \section{Lipschitz regret guarantees}
\label{sec:results-lip}

Our results and techniques for smoothed regret project onto the
well-studied Lipschitz contextual bandits problem: each of the three results in~\pref{sec:results-smooth} has a ``twin'' for the Lipschitz version. We posit a Lipschitz condition on the
expected loss $\lambda(\cdot \mid x) \defeq \EE[\ell(\cdot) \mid x]$:
\begin{align}\label{eq:Lip-condition}
\forall x \in \Xcal,\, a, a' \in \Acal: \quad
\abr{\lambda(a \mid x) - \lambda(a' \mid x)} \leq L\cdot \rho(a,a'), \quad L\geq 1.
\end{align}
We assume that $L \geq 1$ to avoid the pathological situation where
Lipschitzness restricts the effective loss range. If the
Lipschitz constant is less than $1$, we set $L=1$ in our
results.

\df{A version of the standard \emph{uniform discretization} approach
  applies, even for the adversarial setting. Here, we uniformly
  discretize the action space and the policies (if needed), and we run
  $\expfour$. Standard arguments yield the following regret bound:
\begin{align}
\label{eq:Lip-worst-case-1d}
\Reg(T,\Pi) \leq \tilde{O}\rbr{ T^{\nicefrac{2}{3}}\; \rbr{L \log |\Pi|}^{\nicefrac{1}{3}}}.
\end{align}
This result appears in prior work \asedit{on the non-contextual case} and is known to be
optimal~\citep{Bobby-nips04,kleinberg2013bandits,bubeck2011x},
although the generalization to an arbitrary policy set $\Pi$ is
new. Interestingly, the worst-case regret bounds in
\citet{slivkins2014contextual,cesa2017algorithmic} --- on Lipschitz
contextual bandits with a metric on contexts or context-arm pairs and
with specific policy sets $\Pi$, respectively --- can be obtained from
this uniform discretization approach. \refeq{eq:Lip-worst-case-1d}
is the point of departure for several results presented below.  }

The key observation enabling results for the Lipschitz version is as follows: 
\begin{lemma}
\label{lem:smooth_to_lip}
If $f:\Acal \to [0,1]$ is $L$-Lipschitz, then $\abr{\EE_{a' \sim
    \smooth_h(a)} f(a') - f(a)} \leq Lh$.
\end{lemma}
In particular if $\lambda(\cdot \mid x)$ is $L$-Lipschitz, we have
$\bench(\Pi_h) \leq \bench(\Pi) + Lh$, which allows us to easily
obtain results for the Lipschitz version by way of smoothed regret.

\subsection{\df{Instance}-dependent and worst-case guarantees}
\label{sec:lip_instance}
In correspondence with~\pref{thm:stoch_ring_weak}, our first
result here is an \df{instance}-dependent regret bound.  We recover
the optimal worst-case regret bound for the Lipschitz setting, but we
obtain an improvement when actions taken by near-optimal policies tend
to lie in a relative small region of the action space. Specializing,
we recover several state-of-the-art \df{instance}-dependent regret bounds from
prior work. \asedit{Our algorithm is a minor modification of $\smoothpe$ (\pref{alg:ZPE-body}), which we denote $\zpelip$ and spell out later in this section.}

We reuse the packing numbers $M_h(\epsilon,\delta)$ defined
in~\pref{eq:packing_def}, but the instance-dependent complexity is
slightly different. Instead of the smoothing coefficient
$\theta_h(\epsilon_0)$, we use the \emph{policy zooming coefficient}:
\begin{align}\label{eq:policy-zooming-coeff}
\psi_L(\epsilon_0) \defeq \sup_{\epsilon \geq \epsilon_0} M_0(12L\epsilon,\epsilon)/\epsilon.
\end{align}
The main differences over the smoothing coefficient are that version
space of good policies is based on the unsmoothed loss
$\lambda_0(\pi)$, and we are using the $\epsilon$- rather than
$h$-packing number for a fixed bandwidth $h$. For intuition, we always
have $\psi_L(\epsilon_0) \leq O(\epsilon_0^{-2})$ but a favorable
instance might have $\psi_L(\epsilon_0) \leq O(\epsilon_0^{-1})$ which
yields improved rates.

\begin{theorem}
\label{thm:stoch_ring_zooming}
\df{In the stochastic setting,} Algorithm \zpelip with parameter $L$ achieves regret bound
\begin{align}\label{eq:thm:stoch_ring_zooming}
\Reg(T,\Pi)
\leq
O\rbr{\inf_{\epsilon_0 > 0}
\cbr{TL\epsilon_0 +
\frac{\psi_L(\epsilon_0)}{L} \cdot \log(|\Pi|\df{T})\log(1/\epsilon_0)}
}
.
\end{align}
\end{theorem}

Since $\psi_L(\epsilon_0) \leq O(\epsilon_0^{-2})$,
we obtain\dfc{ a} the following worst-case bound,
which is known to be optimal up to \asedit{$\log(T)$} factors.
\begin{corollary}
\label{corr:stoch_ring_lip_instance}
\df{In the stochastic setting,} algorithm \zpelip with parameter $L$ achieves the regret bound in~\pref{eq:Lip-worst-case-1d}.
\end{corollary}

The worst-case result is in correspondence
with~\pref{corr:stoch_ring_instance}. It recovers the worst-case
regret bound from prior work focusing on the non-contextual
version~\citep{Bobby-nips04,bubeck2011lipschitz}.  This regret bound
can also be achieved by \smoothexp \asedit{as a simple corollary of \pref{thm:smooth_exp4}}
(see~\pref{cor:smooth_exp4_lip} in~\pref{sec:extensions}).

The result can also be applied to a nonparametric policy set in the
setting of~\citet{cesa2017algorithmic}. Here we assume $\Xcal$ is a
$p$-dimensional metric space and the policy set is all $1$-Lipschitz
mappings from $\Xcal \to \Acal$. By a suitable
discretization,~\pref{corr:stoch_ring_lip_instance} yields
$\otil\rbr{T^{\frac{p+2}{p+3}}}$ regret, which matches their result
(since the interval is a $1$-dimensional action space).

The advantage of~\pref{thm:stoch_ring_zooming} is its \df{instance}-dependence.
Since the packing number $M_0(\cdot,\cdot)$ is always at least $1$,
the most favorable instances have $\psi_L(\epsilon_0) =
O(\epsilon_0^{-1})$.  In this case,~\pref{thm:stoch_ring_zooming}
gives the much faster $\tilde{O}(\sqrt{T\log|\Pi|})$ regret rate. The
next example demonstrates such favorable behavior.
\begin{example}
\label{ex:small_zooming}
Let $\Sbb^{d-1}$ denote the unit sphere in $\RR^d$. Consider an
instance where $\Xcal \defeq \Sbb^{d-1}$, $\Acal \defeq [-1,1]$ and
where the policy class $\Pi$ is a finite subset of \emph{linear}
policies $\cbr{\pi_w: w \in \Sbb^{d-1}}$ where $\pi_w : x \mapsto
\inner{w}{x}$. The marginal distribution over contexts is uniform over
$\Sbb^{d-1}$ and the expected losses satisfy
\begin{align}
\forall x \in \Xcal: \EE\sbr{\ell(a) \mid x} = f(a - \pi_{w^\star}(x)), \label{eq:local_abs}
\end{align}
where $\pi_{w^\star} \in \Pi$ is some fixed policy, $f$ is
$L$-Lipschitz and satisfies $f(z) - f(0) \geq L_0 \abr{z}$ for all $z$
in $\RR$. By construction, $\EE[\ell(a) \mid x]$ is $L$-Lipschitz in
$a$, for all $x$. This instance has $M_0(L \epsilon, \epsilon) =
O(\nicefrac{L}{L_0}\cdot\sqrt{d})$, and $\psi_L(\epsilon) =
O(\frac{L}{L_0 \epsilon}\cdot\sqrt{d})$. (See~\pref{sec:calculations} for a derivation.)
\end{example}

\df{Instance}-dependent bounds from prior work are often stated in terms of
a packing number growth rate, called the \emph{zooming
  dimension}. Our bound can
also be stated in this way, so as to facilitate comparisons. With
\emph{zooming constant} $\gamma > 0$ the zooming dimension is defined as
\begin{align}\label{eq:zooming-dim-defn}
z \defeq \inf\cbr{d > 0: M_0(12L\epsilon,\epsilon) \leq \gamma \cdot \epsilon^{-d},\ \forall \epsilon \in (0,1)}.
\end{align}

\noindent
It is easy to see that
    $\psi_L(\epsilon_0) \leq \gamma\cdot \epsilon_0^{-z-1}$,
and so~\pref{thm:stoch_ring_zooming} may be
further simplified to
\begin{align}\label{eq:zooming-dim-result}
\Reg(T,\Pi) \leq
    O\rbr{ L^\frac{z}{2+z}T^{\frac{1+z}{2+z}}} \cdot
    \rbr{ \gamma \log(T\,|\Pi|) }^{\frac{1}{2+z}
    }.
\end{align}

This result agrees with prior zooming results in the non-contextual setting~\citep{kleinberg2013bandits,bubeck2011x}.
In the contextual setting, our result is in general incomparable with
the ``contextual zooming algorithm''
of~\citet{slivkins2014contextual}, which scales with a different
quantity called the \emph{contextual zooming dimension}. Formally the
contextual zooming dimension measures the growth of the
$\epsilon$-packing numbers of the set $\{(x,a): \EE\sbr{\ell(a)\mid
  x} - \min_{a' \in \Acal}\EE\sbr{\ell(a')\mid x} \leq \epsilon\}$ as
a function of $\epsilon$. This definition, and our zooming dimension
are conceptually similar, as both measure the size of certain
near-optimal sets, but they are generally incomparable. Informally,
the definition in~\citet{slivkins2014contextual} is adapted to a
Lipschitz structure on the context space, which does not naturally
accommodate arbitrary policy sets $\Pi$ as we do, and our zooming
dimension involves the ``expected context'' rather than the ``worst
context.'' In more detail:
\begin{enumerate}
\item \citet{slivkins2014contextual} needs to assume a metric structure
on $\Xcal \times \Acal$, whereas we only assume a metric structure on $\Acal$.
In addition, ~\citet{slivkins2014contextual}'s
contextual zooming dimension is at worst the covering dimension of $\Xcal \times \Acal$,
whereas our notion of zooming dimension is at worst the covering dimension of
$\Acal$. On the other hand, our bound scales with $\log |\Pi|$ while his does not.
\item Aside from the metric structure, \citet{slivkins2014contextual}'s
contextual zooming dimension is only dependent on the conditional distribution of loss given
context $\Dcal(\ell|x)$.
In contrast, our notion is dependent on the policy class $\Pi$, along with $\Dcal$,
the joint distribution of $(x,\ell)$, which
admits policy class and distribution specific upper bounds.
\item Finally, \citet{slivkins2014contextual} considers a setting
  where contexts are adversarially chosen, and so his contextual zooming
  dimension considers pessimistic context arrivals. Our definition
  involves an expectation over contexts, which may be more favorable.
\end{enumerate}

\paragraph{The algorithm.}
The algorithm is almost identical to \smoothpe. The main difference is
that instead of a fixed bandwidth $h$ across all epochs, we use $h_m =
2^{-m}$ in the $m^{\textrm{th}}$ epoch. We also set the radius
parameter $r_m = L2^{-m}$ which is slightly different from before.
We call this algorithm \zpelip, to highlight the differences.

At a technical level, the main difference with the Lipschitz setting
is that we must carefully balance bias and variance in loss
estimates. This is not an issue for smoothed regret since we have
unbiased estimators for $\lambda_h(\pi)$, but not for
$\lambda_0(\pi)$. We do this by decreasing the bandwidth geometrically
over epochs, but the rest of the algorithm, and much of the analysis
are unchanged.

\subsection{Optimal Lipschitz-Adaptivity}
\label{sec:lip_adapt}

We now present the corresponding result to~\pref{thm:smooth_adaptive}. We consider \emph{Lipschitz-adaptive} algorithms: those that \emph{do not know} any information about the problem, apart from $T$ and $\Pi$, and yet achieve regret bounds in terms of $T,L$, and $|\Pi|$ \emph{only}. In particular, the algorithm does not know $L$.

\begin{theorem}
\label{thm:lipschitz_adaptive}
Consider the adversarial setting.
For each $\beta \in [0,1]$,
\corralexpf (with parameter $\beta$) is Lipschitz-adaptive with
\begin{align*}
\Reg(T,\Pi) \leq \otil
    \rbr{ T^{1-a}\;L^b\; \rbr{\log|\Pi|}^a},\;
\text{where}\quad
a = \frac{\beta}{1+2\beta}
\;\text{and}\;
b = \frac{\beta}{1+\beta}.
\end{align*}
For the non-contextual version it achieves a regret
$\otil\rbr{T^{1-a}\; L^b }$
without
knowing the Lipschitz constant $L$. Moreover, for the non-contextual
\df{stochastic} version, \df{there exist positive constants $c$ and
  $T_0$, such that for any algorithm and any $T \geq T_0$, there
  exists $L \geq 1$ and a problem instance with $L$-Lipschitz losses,
  such that on that instance
\begin{align*}
\Reg(T,\Pi) \geq c \cdot T^{1-a} L^b.
\end{align*}}
\end{theorem}

\paragraph{Remarks.}
As in~\pref{thm:smooth_adaptive}, we obtain a family of upper and
lower bounds, one for each $\beta \in [0,1]$, which make up a Pareto
frontier. With $\beta=1$ an optimal Lipschitz-adaptive rate is
$T^{2/3}\sqrt{L}$ which is worse than the $T^{2/3}L^{1/3}$
non-adaptive rate from~\pref{corr:stoch_ring_lip_instance}. Note that
it is easy to obtain the worse adaptive rate of $\otil\rbr{LT^{2/3}}$
simply by guessing that the Lipschitz constant is $1$ in our variant
of \expfour.

Several prior works develop adaptive algorithms that either require
knowledge of unknown problem parameters, or yield regret bounds that, in addition to $T$ and $L$, scale with such parameters~\citep{ImplicitMAB-nips11,bubeck2011lipschitz,Bull-bandits14,locatelli2018adaptivity}. These algorithms are not Lipschitz adaptive, contrasting with our algorithm that requires no additional knowledge or
assumptions.  However, this dependence on other parameters allows these
prior results to side-step our lower bound and achieve faster rates.

Note that Lipschitz-adaptivity is qualitatively quite different from
the uniformly-smoothed adaptivity studied
in~\pref{thm:smooth_adaptive}. With Lipschitz-adaptivity there is a
single fixed benchmark policy class and we simply seek a guarantee
against that class, albeit in an environment with unknown smoothness
parameter. However, for~\pref{thm:smooth_adaptive} we are effectively
competing with infinitely many policy sets simultaneously ($\Pi_h$ for
each $h \in (0,1]$) and we seek a regret bound against all of them.
  Somewhat surprisingly, both settings demonstrate a similar
  price-of-adaptivity and the optimally adaptive algorithms are nearly
  identical.

\paragraph{The algorithm.}
The algorithm, \df{\corralexpf}, is again \corral with our variant of \expfour as the
sub-algorithms. The only difference is in how we set the learning rate
for the master algorithm.

 \section{Our results in a general setup}
\label{sec:extensions}

All results discussed so far are special cases of a more general set of results that we now present. While all of the key ideas appear in the special case of the unit interval, the following results demonstrate the generality of our approach. As we have already made many of the essential remarks, the discussion here is somewhat terse.

We generalize in two directions. First, all results extend to higher-dimensional action spaces. Formally,
$\Acal$ can be an arbitrary convex subset of the $d$-dimensional unit
cube $[0,1]^d$, equipped with $p$-norm $\rho(a,a') \defeq \textstyle
\| a-a'\|_p$, for any $p\geq 1$. As before, $\smooth_h(a)$ is a
uniform distribution over the closed ball of radius $h$. The \df{instance}-dependent regret bounds carry
over as is, and zooming dimension now takes values in $[0,d]$
depending on the problem instance. Regret bounds in the worst-case
corollaries are modified so as to accommodate the dependence on
$d$. In \pref{corr:stoch_ring_instance}, the dependence on $h$ is replaced with $h^d$, and there is a matching lower bound. In
\pref{corr:stoch_ring_lip_instance}, the dependence on $T$ becomes
$\tilde{O}(T^{(d+1)/(d+2)})$, which is known to be optimal. The
smoothness-adaptive regret bounds are modified
similarly.

Second, we essentially allow
\emph{arbitrary} action spaces and smoothing operators.  Formally, the action space $\Acal$ is endowed with a base metric $\rho$ and a base measure $\base$. The smoothing distribution $\smooth(a)$ can be any distribution with a well-defined density with respect to $\base$, and the effective number of actions is the largest possible density value.  \asedit{We define bandwidth relative to the base metric. In particular, we can handle the unit cube  $[0,1]^d$, endowed with a uniform measure and the $p$-norm, $p\geq 1$ as a base metric.}

The proofs for all instance-dependent regret bounds are deferred to
\pref{sec:proofs-instance-dependent}. The proofs for all ``smoothness-adaptive'' regret bounds can be found in \pref{sec:proofs-corral}.

\subsection{General setup}
\label{sec:general-setup}

For completeness, let us recap the basic setup of contextual bandits. There are two sets $\Xcal,\Acal$, where $\Xcal$ is an abstract \emph{context space}, and $\Acal$ is an abstract \emph{action space}. The following protocol continues over $T$ rounds: at each round $t$, (i) nature chooses context $x_t \in
\Xcal$ and loss function $\ell_t \in (\Acal \to [0,1])$ and presents $x_t$ to
the learner, (ii) learner chooses action $a_t \in \Acal$, (iii) learner
suffers loss $\ell_t(a_t)$, which is also observed. Performance of the
learner is measured relative to a class of policies $\Pi: \Xcal \to
\Acal$ via the notion of regret
\begin{align*}
\Reg(T,\Pi) \defeq \EE\sbr{\sum_{t=1}^T \ell_t(a_t)} -\min_{\pi \in \Pi} \EE\sbr{\sum_{t=1}^T\ell_t(\pi(x_t))}.
\end{align*}

We consider both adversarial and stochastic settings.
In the \emph{adversarial setting} the contexts and losses are chosen
by an adaptive adversary, meaning that $(x_t,\ell_t)$ may be a
randomized function of the entire history of interaction. In the
\emph{stochastic setting}, we assume $(x_t,\ell_t) \sim \Dcal$ iid at
each round $t$, for some unknown distribution $\Dcal$, \df{although we assume $\Dcal_X$, the marginal distribution over $\Xcal$, is known}.

\paragraph{Base structure.}
Action space $\Acal$ is endowed with \asedit{\emph{base structure} $(\Acal,\rho,\base)$,} where $\rho$ is a metric called the \emph{base metric}, and $\base$ is a probability measure called the \emph{base measure}. The two are consistent, in the sense that $\base$ is well-defined and strictly positive on the closed balls in $(\Acal,\rho)$ of strictly positive radius. This structure may have no bearing on the loss functions; it serves only to define and/or instantiate smoothed regret. \asedit{Essentially, we smoothen relative to the base measure, and define bandwidth relative to the base metric.}

The closed balls are denoted
    $\ball(a,r) \defeq \{b \in \Acal: \rho(a,b) \leq r\}$,
where $a\in\Acal$ is the center and $r\geq 0$ is the radius. For normalization, we assume that the metric space has diameter $1$.

\paragraph{Smoothing kernel.}
We generalize the $\smooth$ operator to a \emph{smoothing kernel}: a mapping
    $K: \Acal \to \Delta(\Acal)$,
the set of distributions over actions. For policy
$\pi$, we use $K\pi :x \mapsto K(\pi(x))$ to denote the usual function composition.  With $\Pi_K \defeq \cbr{K\pi : \pi \in \Pi}$ as the smoothed policy class, smoothed regret is simply given by $\Reg(T,\Pi_K)$.

We posit that distributions $K(a)$, $a\in\Acal$ are absolutely continuous with respect to the base measure $\base$, and represent them via their \asedit{density functions $f_{K(a)}$. Formally, $f_{K(a)}$ is the Radon-Nikodym
derivative of $K(a)$ relative to $\base$.} As a convention, denote
    $(Ka)(a') := f_{K(a)}(a') $, $a'\in\Acal$.
In words, it is the density of distribution $K(a)$ with respect to $\base$, evaluated at $a'$.

We derive (worst-case) bounds on $\Reg(T,\Pi_K)$ for an arbitrary smoothing kernel $K$, without any assumptions on the loss functions. The regret bounds are in terms of the largest possible density assigned by $K$,
\begin{align}\label{eq:def:kappa}
\kappa \defeq \sup_{a,a'\in A} (Ka)(a').
\end{align}
This quantity, called \emph{kernel complexity}, serves as the effective number of actions.

We trade off $\kappa$ against a suitably generalized notion of \emph{bandwidth}: \ascomment{added this defn}
\begin{align}\label{eq:def:bandwidth}
\sup_{a,a'\in A:\;\; (Ka)(a')>0}\quad \rho(a,a').
\end{align}
In words, it is the largest distance that any action can be perturbed by.

The canonical example is the \emph{rectangular kernel} $K_h$, where $h\in (0,1]$ is the \emph{bandwidth}:
\begin{align}\label{eq:rect-kernel}
(K_h\, a)(a') =  \frac{\one\cbr{\rho(a,a') \leq h}}{\base(\ball(a,h))}, \quad \forall a,a'\in\Acal.
\end{align}
In words, $K_h(a)$ puts uniform density on $\ball(a,h)$, and zero density elsewhere.\footnote{\asedit{If the action space $\Acal$ is a unit interval, the plot of the density function for $K_h(a)$ is a rectangle, hence the name \emph{rectangular kernel}.}}

\ascomment{BEGIN: new text}

\paragraph{Discussion.}
In practice, we may have some freedom in choosing the base structure. The action space $\Acal$ may naturally admit a set system: e.g., the open/closed intervals when $\Acal$ is a unit cube, or the subtrees when $\Acal$ is the leaf set of a tree. Then, we may have some leeway in defining the base metric, e.g., it could be any $p$-norm, $p\geq 1$ when $\Acal$ is the unit cube, or any ``exponential tree metric"
    $\rho(x,y) = \alpha^{\mathtt{depth}(\mathtt{LCA}(x,y))}$,
    $\alpha\in(0,1)$
when $\Acal$ is a leaf set.\footnote{$\mathtt{LCA}(x,y)$ is the least common ancestor of leaves $x$ and $y$.} Then, we may be able to tailor the base measure to the chosen base metric, so as to improve the kernel complexity (more on this below).

One can choose smoothing kernels other than the rectangular kernel. One fairly general formulation is the \emph{$f$-symmetric kernel} $K_f$ defined by
\begin{align}\label{eq:sym-kernel}
(K_f\, a)(a') \sim f(\rho(a,a'))
    \quad \forall a,a'\in\Acal,
\end{align}
for some function $f: [0,1]\to [0,\infty) $. In particular, the \emph{triangular kernel} is the special case when
    $f(x) = \max(0,1-x/h)$, where $h>0$ is the bandwidth.
For more refined kernel complexity vs. bandwidth tradeoff, one could consider an averaged version of bandwidth:
\begin{align}\label{eq:def:bandwidth-ave}
\sup_{a\in A} \;\int \rho(a,\cdot) \; \mathtt{d} K(a).
\end{align}
That said, in our analysis the smoothing kernel is either arbitrary or rectangular.

\paragraph{Example: covering dimension.}
To instantiate kernel complexity, consider the notion of \emph{covering dimension}. The formal definition is as follows:

\begin{definition}\label{def:CovDim}
For a metric space $(\Acal,\rho)$, the covering dimension with multiplier $\gamma$ is the smallest number $d \geq 0$ such that for each $r \in (0,1]$, the metric space can be covered with $\gamma\cdot r^{-d}$ balls of radius $r$.
\end{definition}

\noindent This notion has been used to summarize the complexity of a metric space for Lipschitz bandits \citep{Bobby-nips04}. We can also use it to bound kernel complexity.

\begin{claim}\label{cl:CovDim}
Fix the base metric space $(\Acal,\rho)$ of covering dimension $d$ with multiplier $\gamma$. Fix bandwidth $h>0$. Then there exists a probability measure $\base$ such that
\begin{align}\label{eq:CovMeasure}
\base(\ball(a,h))\geq (\nicefrac{h}{2})^d/\gamma \quad
\text{for each center $a\in\Acal$}.
\end{align}
With $\base$ as the base measure, the rectangular kernel $K_h$ has complexity $\kappa \leq \gamma\cdot (\nicefrac{h}{2})^{-d}$.
\end{claim}

\begin{proof}
By definition of the covering dimension, there is a collection $\Ccal$ of at most $N = \gamma\cdot (\nicefrac{h}{2})^{-d}$ balls of radius $\nicefrac{h}{2}$ whose union covers $\Acal$. Define probability measure $\base$ as follows: pick a ball $B\in \Ccal$ uniformly at random, then pick a point inside $B$ according to an arbitrary fixed probability measure $\base_B$. Any ball $\ball(a,h)$, $a\in \Acal$ contains some ball $B\in\Ccal$, namely a ball in $\Ccal$ that covers $a$. Hence,
    $\ball(a,h) \geq \nu(B) \geq 1/N$.
\end{proof}

\ascomment{END: new text}

\paragraph{Example: local uniformity.}
\asedit{It may be desirable to ensure that the base structure is uniform, in the sense that balls of similar radius have a similar measure. A ``local'' version of this property can be stated as follows:} for some number $d\geq 0$ called the \emph{doubling dimension},
\begin{align}\label{eq:DblMeasure}
 \base(\ball(a,2r)) \leq
2^d\cdot \base(\ball(a,r))
\qquad \forall a\in\Acal,\,r>0.
\end{align}
Then the rectangular kernel $K_h$, $h>0$ has complexity
    $\kappa \leq (\nicefrac{h}{2})^{-d}$,
like in Claim~\ref{cl:CovDim}.

By way of background, doubling dimension is a combinatorial notion of low-dimensionality, widely used in theoretical computer science.\footnote{Doubling dimension have been studied in many different contexts such as metric embeddings, traveling salesman and compact data structures, e.g.,
\citet{Gup03,KL-soda04,KLMN04,Tal04,Slivkins-focs04,Slivkins-focs05-full,Slivkins-podc05,Men05,Meridian-sigcomm05}.}
It is a stronger notion than the covering dimension: it upper-bounds the covering dimension (with multiplier $\gamma = 2^d$). A canonical example is that any subset of $([0,1]^d,\ell_p)$, $d\in \mathbb{N}$, $p\geq 1$ has doubling dimension $d+O(1)$. However, there are examples that are provably very different \citep{Gup03}. When the metric space $(\Acal,\rho)$ is complete,  a probability measure $\base$ satisfying \eqref{eq:DblMeasure} exists if and only if the metric space satisfies a more basic property: any ball of radius $r$ can be covered by a collection of $2^d$ balls of radius $r/2$ \citep{Volberg87,Wu-ams98,Luuk-ams98}. The latter property is typically used to define the doubling dimension. More background on doubling dimension can be found in
\citep[][Chapter 2]{Slivkins-thesis}.

\paragraph{Global uniformity of the base structure.}
\asedit{For our instance-dependent results in the stochastic setting, we require a ``global'' generalization of \eqref{eq:DblMeasure} which states that} any two balls of a similar radius have a similar size. Formally:
\begin{assum}
[\asedit{instance-dependent results only}]
\label{assum:uniform}
\[ \sup_{\asedit{a,a'\in \Acal,\; h\in(0,\nicefrac12]}}
\frac{\base(\ball(a,2h))}{\base(\ball(a',h))} \leq \alpha < \infty.\]
\end{assum}
The effect of this assumption is that the $\alpha$ is a multiplier in the regret bounds. \asedit{While $\alpha$ gives a direct bound on kernel complexity of the rectangular kernel $K_h$ as
    $\kappa \leq (\nicefrac{h}{2})^{-\log \alpha}$,
it can be productive to bound $\kappa$ using the covering dimension. Indeed, the latter is a much weaker property, in the sense that it is smaller than $\log \alpha$, and can be \emph{much} smaller.}

The canonical example is a finite subset of $[0,1]^d$ of near-uniform density, defined as follows. For a fixed scale $\eps>0$, partition $[0,1]^d$ into axis-parallel hypercubes with side $\eps$, called \emph{$\eps$-cells}. A subset $\Acal \subset [0,1]^d$ is \emph{uniform-density} at scale $\eps$ if each $\eps$-cell contain exactly one point in $\Acal$.
Then \pref{assum:uniform} holds with $\alpha=O(1)^d$, with \asedit{$\ell_\infty$} as the base metric and the uniform measure over $\Acal$ as the base measure.\footnote{To see this, observe that for every $a$ in $\Acal$, $(\lfloor \frac{h}{2\eps} \rfloor + 1 )^d \leq \abr{\ball(a,h) \cap \Acal} \leq (\frac{2h}{\eps}+2)^d$.}
Similar assumptions have been used in theoretical computer science literature on networks
\citep{Kle00-STOC,Kempe-stoc01,Kempe-focs02,Sarkar-colt10,SocialDistance-soda13-sicomp}.

\OMIT{ \begin{claim}
Let $\Acal \subset \mathbb{R}^d$ be a uniform-density subset at a given scale $\eps>0$. Then \pref{assum:uniform} holds with $\alpha=O(4^d)$, with $\ell_p$, $p\in [1,\infty]$ as the base metric and the counting measure as the base measure.
\end{claim}
} 
\subsection{Special cases}
\label{sec:cases}

\ascomment{NB: I need the numbered list to refer to the non-contextual example!}

\begin{itemize}

\item[(a)] \emph{Unit interval.}
Suppose action space $\Acal = [0,1]$ is endowed with base metric $\rho(a,a') \defeq \abr{a-a'}$, and the base measure is uniform over $\Acal$. Then the rectangular kernel \eqref{eq:rect-kernel} is precisely the $\smooth_h$ operator from~\pref{sec:defns-smooth}. This example satisfies properties \eqref{eq:CovMeasure} and \eqref{eq:DblMeasure}
(with $d=1$) and \pref{assum:uniform} (with $\alpha = 4$, because of the edge effects). So, all results we presented in \pref{sec:results-smooth} and \pref{sec:results-lip} follow from the general development.

\item[(b)] \emph{Discretized unit interval.}
Discretize the $[0,1]$ interval into $M$ actions: $\Acal \defeq \{i/M : i \in [M]\}$, with base metric/measure defined as above. The rectangular kernel $K_h$ takes local averages across actions. Kernel complexity is $\kappa = \nicefrac{1}{h}$, and we obtain bounds on smooth regret that are independent of the number of actions $M$.

\item[(c)] \emph{Non-contextual setting.}
The non-contextual setting can be embedded in ours by positing a single context $\Xcal \defeq \{x_0\}$ and policy set $\Pi: \{x_0 \mapsto a : a \in \Acal\}$. (When we state results for the non-contextual version, $\Pi$ is \emph{always} assumed to be this class.) Since our upper bounds typically scale with $\log |\Pi|$, they do not immediately yield meaningful guarantees when $|\Acal| = \infty$, but we can obtain meaningful results here via discretization.

\asedit{For example, consider the basic setup in \pref{sec:defns-smooth}. Since the smoothed loss function $\ExpL_h(\cdot)$ is $(\nicefrac{1}{h})$-Lipschitz, we can discretize the action space uniformly with step $\sqrt{\nicefrac{h}{T}}$, to ensure that $|\Pi| = \sqrt{\nicefrac{T}{h}} $ and the discretization error --- increase in regret due to the discretization --- is at most $\sqrt{\nicefrac{T}{h}}$.}

\item[(d)] \emph{Standard (non-smoothed) regret.}
As a sanity check, let us recover a standard (non-smoothed) contextual bandit problem as a special case. Let $\Acal$ be a finite set of $M$ actions, equipped with an identity metric
    $\rho(a,a') \defeq \one\{a \ne a'\}$
and uniform base measure $\base$.  Then with \emph{identity kernel} $K: a\mapsto \delta_a$ we have $\Pi_K = \Pi$.

\item[(e)] \asedit{\emph{Fixed discretization.} Interestingly, we also recover regret bounds relative to a fixed discretization of the action space $\Acal$. Formally, let $\Acal_0$ be a finite subset of $\Acal$, let the base measure be the uniform distribution over $\Acal_0$, and define the smoothing kernel $K$ to deterministically map each action $a$ to the closest point in $\Acal_0$. Then the smoothed policy set $\Pi_K$ is precisely the set of policies whose actions are discretized to $\Acal_0$. It is easy to see that the kernel complexity (``effective number of arms") is $|\Acal_0|$.}

\end{itemize}

\subsection{Results for smoothed regret}
\label{sec:results-smoothed}

\subsubsection*{Stochastic Setting: instance-dependent results}

We focus on the rectangular kernel $K_h$. Our results are stated in terms of the \emph{smoothing coefficient} $\theta_h(\eps_0)$, as defined in \eqref{eq:smoothing-coeff}, with \emph{smoothed loss} suitably redefined as
\begin{align}\label{eq:defn-smooth-loss-generalized}
\ExpL_h(\pi) \defeq
    \EE_{(x,\ell)\sim \Dcal}\;\;
    \EE_{a \sim K_h(\pi(x))} \sbr{ \ell(a) },\quad \pi\in\Pi.
\end{align}

\OMIT{  Let us introduce the notation for the general case, which is not substantially different from before. Recall that the \emph{smoothed loss} for policy $\pi \in \Pi$ is
\begin{align*}
\lambda_h(\pi) \defeq \EE_{(x,\ell)\sim\Dcal}[ \inner{K_h\pi(x)}{\ell} ],
\end{align*}
The optimal loss is $\bench(\Pi_h) \defeq \inf_{\pi \in
  \Pi}\lambda_h(\pi)$ and the $\eps$-optimal policies are
\begin{align*}
\Pi_{h,\eps} \defeq \{ \pi \in \Pi: \lambda_h(\pi) \leq \bench(\Pi_h) + \eps\}.
\end{align*}
The projection of a policy set $\Pi'$ onto context $x$ is $\Pi'(x) =
\{\pi(x):\pi \in \Pi'\}$, \df{which is a subset of $\Acal$}. Define
\begin{align*}
M_h(\eps,\delta) \defeq \EE_{x \sim \Dcal}\sbr{ \Ncal_\delta(\Pi_{h,\eps}(x)) },
\end{align*}
where recall that $\Ncal_\delta(A)$ is the $\delta$-packing number of
$A \subset \Acal$. The \emph{smoothing coefficient} is
\begin{align*}
\theta_h(\eps_0) \defeq \sup_{\eps \geq \eps_0} M_h(12\eps,h)/\eps.
\end{align*}
} 

\noindent Generalizing \smoothpe requires the following changes. Rather than use
the $\smooth$ operator, we use the kernel $K_{\df{h}}$ in the variance
constraint, action selection scheme, and importance weighted loss. We
also update the batch size parameter $\tilde{n}_m \defeq \frac{320
  \kappa_{h} V_m}{r_m^2}$.

\begin{theorem}[generalizes~\pref{thm:stoch_ring_weak}]
\label{thm:smooth_pe_weak}
Consider the stochastic setting with rectangular kernel $K_h$, under ~\pref{assum:uniform}.
Then~\smoothpe has
\begin{align*}
\Reg(T,\Pi_{K_h})
\leq O\rbr{\inf_{\eps_0 > 0 }
\cbr{T\eps_0 +
    \alpha\; \theta_h(\eps_0)\; \log(|\Pi|\df{T})\;\log(1/\eps_0)}}.
\end{align*}
\end{theorem}

While \pref{assum:uniform} enables better regret rates for benign instances, the algorithm can be analyzed without this assumption, and for arbitrary kernels. The following regret bound can be extracted from the proof of \pref{thm:smooth_pe_weak} without much difficulty:

\begin{theorem}[generalizes~\pref{corr:stoch_ring_instance}]
\label{thm:smooth_pe_general}
Consider the stochastic setting with an arbitrary smoothing kernel of kernel complexity $\kappa$.
Then~\smoothpe has
\begin{align*}
\Reg(T,\Pi_K)
\leq \otil\rbr{\sqrt{T\kappa \log(|\Pi|\dfc{/\delta})}}.
\end{align*}
\end{theorem}

\subsubsection*{Adversarial setting}

We use a version of \expfour (\pref{alg:smooth_exp4}), as before. We handle an arbitrary smoothing kernel $K$, obtaining a regret bound in terms of its kernel complexity $\kappa$. We obtain \pref{thm:smooth_exp4} by specializing to the unit interval, as explained in \pref{sec:general-setup}.

\begin{theorem}[generalizes \pref{thm:smooth_exp4}]
\label{thm:smooth_exp4_general}
Consider the adversarial setting  with an arbitrary smoothing kernel of kernel complexity $\kappa$.
\smoothexp (\pref{alg:smooth_exp4}) with policy set $\Xi = \Pi_K$
and learning rate
    $\eta = \sqrt{\frac{2\ln|\Xi|}{T \kappa}}$
admits smoothed regret
\[ \Reg(T,\Pi_K) \leq
\order\rbr{\sqrt{T\kappa \log |\Pi|}}.\]
\end{theorem}

\begin{proof}
\asedit{One subtle point in the proof is that we separate the base measure, the smoothing kernel, and the action sampling distribution. The details follow standard techniques.}

From the analysis of algorithm \textsc{Hedge}~\citep{FS97}, we obtain
\begin{align}
\sum_{t=1}^T \EE_{\xi \sim P_t} \hat{\ell}_t(\xi) - \min_{\xi \in \Xi} \sum_{t=1}^T \hat{\ell}_t(\xi)
\leq \frac{\eta}{2}\sum_{t=1}^T \EE_{\xi \sim P_t} \hat{\ell}_t(\xi)^2 + \frac{\log |\Xi|}{\eta},
\label{eq:localnorm_expect}
\end{align}
where $P_t$ is a distribution over policies proportional to the weights $W_t$ in the algorithm.

Now, by standard importance weighting arguments we have (i) $\EE_{\xi
  \sim P_t}\hat{\ell}_t(\xi) = \ell_t(a_t)$ and (ii) $\EE_{a_t \sim
  p_t}\hat{\ell}_t(\xi) = \EE_{a \sim \xi(\cdot \mid
  x_t)}\ell_t(a)$. For the variance term, we have
\begin{align*}
\EE_{a_t,\xi}\hat{\ell}_t(\xi)^2 \leq \kappa \EE_{a_t,\xi}\ell_t(a_t)^2\frac{\xi(a_t\mid x_t)}{p_t(a_t \mid x_t)^2} = \kappa \int \ell^2_t(a) \frac{p_t(a \mid x_t)}{p_t(a \mid x_t)} d\lambda(a) = \kappa \nbr{\ell_t}_2^2 \leq \kappa \nbr{\ell_t}_\infty^2.
\end{align*}

Therefore, taking expectation over both sides
of~\pref{eq:localnorm_expect}, we have
\begin{align}
\EE \sum_{t=1}^T \EE_{\xi \sim P_t} \hat{\ell}_t(\xi) - \EE \min_{\xi \in \Xi} \sum_{t=1}^T \hat{\ell}_t(\xi)
\leq \EE \sum_{t=1}^T \frac{\eta\kappa}{2} \|\ell_t\|_{\infty}^2 + \frac{\log |\Xi|}{\eta}.
\label{eqn:exp4_generic}
\end{align}
Applying Jensen's inequality on the left hand side, using the fact
that $\|\ell_t\|_{\infty} \leq 1$, and optimizing for $\eta$, we obtain
the claimed regret bound.
\end{proof}

\OMIT{ Instantiating~\pref{thm:smooth_exp4_general} in special cases, we obtain
\begin{itemize}
\item $\sqrt{\nicefrac{T}{h}\cdot\log |\Pi|}$ regret against $\Pi_h$
  in the unit interval. Note that this matches the
  $\Omega\rbr{\sqrt{\nicefrac{T}{h}\cdot\log |\Pi|}}$ regret lower
  bound, as discussed in~\pref{sec:results-smooth}.
  \item $\sqrt{KT\log|\Pi|}$ regret for standard (non-smoothed) contextual bandits.
\end{itemize}
} 

\subsubsection*{Uniformly-smoothed regret}

\asedit{ 
We consider an arbitrary finite family of smoothing kernels $K_1 \LDOTS K_M$. The goal is to obtain small smoothed regret with respect to each of these kernels.

We start with a simple result in terms of the maximal kernel complexity. We use \smoothexp (\pref{alg:smooth_exp4}) with policy set $\Xi = \cup_i; \Pi_{K_i}$, the union of the smoothed policy classes. The analysis of \pref{thm:smooth_exp4_general} carries over verbatim.

\begin{theorem}
\label{thm:general_adaptivity_simple}
Consider the adversarial setting with smoothing kernels
    $K_1 \LDOTS K_M$
defined on the same base structure. Suppose each kernel has complexity at most $\kappa$. Then \smoothexp (\pref{alg:smooth_exp4}) with policy set
    $\Xi = \cup_{i=1}^M\; \Pi_{K_i}$
and learning rate
    $\eta = \sqrt{\frac{2\ln|\Xi|}{T \kappa}}$
admits smoothed regret
\[ \Reg(T,\Xi) \leq
\order\rbr{\sqrt{T\kappa \log|\Xi|}},\;\;
\text{where $|\Xi| = M \cdot |\Pi|$}.\]
\end{theorem}

Our main result is more nuanced, obtaining improved smoothed regret relative to kernels of small kernel complexity.

} 
\OMIT{ family of smoothing kernels $\Kcal \defeq \{K_1,\ldots,K_M\}$ and we define
$\kappa_i \defeq \sup_{a,a'}\abr{(K_ia)(a')}$ and
$\kappa_\star \defeq
\min_{i \in [M]} \kappa_i$.  Further define $r \defeq \frac{\max_{i
    \in [M]} \kappa_i}{\min_{i \in [M]} \kappa_i}$.  The following is
the general adaptive regret \df{upper bound.}
}

\begin{theorem}
\label{thm:general_adaptivity}
Consider the adversarial setting with smoothing kernels
    $K_1 \LDOTS K_M$,
\asedit{whose their respective kernel complexities are
    $\kappa_1 \LDOTS \kappa_M$.
Let $\kappa_\star$ and $\kappa_{\max}$ be, resp., the smallest and the largest kernel complexity.}
Algorithm \corralexpf with parameter $\beta\in [0,1]$ guarantees the following for each $i\in[M]$:
\begin{align*}
\Reg(T,\Pi_{K_i}) \leq \order
\rbr{ T^{\frac{1}{1+\beta}}\;
    \rbr{\kappa_i  \log(|\Pi|M)}^{\frac{\beta}{1+\beta}}}
\cdot
\rbr{\min\cbr{M, \log \frac{\kappa_{\max}}{\kappa_\star}}}^{\frac{1}{1+\beta}}\; \rbr{\frac{\kappa_i}{\kappa_\star}}^{\frac{\beta^2}{1+\beta}} .
\end{align*}
\end{theorem}

\begin{remark}
\asedit{Each kernel $K_i$, $i\in[M]$ can have its own base structure
    $(\Acal,\rho_i,\base_i)$.}
\end{remark}

The setup above (specialized to action space $\Acal = [0,1]$)
almost yields the upper bound in~\pref{thm:smooth_adaptive}, except that we can only compete with a finite set of kernels. For example, choosing $K_i$ as the rectangular kernel with bandwidth $h=2^{-i}$ recovers a weaker version of the theorem. For the stronger version that competes with all bandwidths $h \in [0,1]$, we must exploit further structure. The next result achieves this, generalizing \pref{thm:smooth_adaptive} to action space $\Acal = [0,1]^d$ with an arbitrary dimension $d$.

\begin{theorem}\label{thm:refined_adaptivity}
Consider the adversarial setting, with action space $\Acal = [0,1]^d$, $d\in\NN$, uniform base measure $\base$, and base metric $\rho = \ell_{\infty}$. \pref{thm:smooth_adaptive} extends, with $h$ replaced by $h^d$.
\end{theorem}

\OMIT{ \begin{theorem}[generalizes \pref{thm:smooth_adaptive}]
\label{thm:refined_adaptivity}
Consider the adversarial setting with action space $\Acal = [0,1]^d$, $d\in\NN$, endowed with uniform base measure $\base$ and base metric $\rho = \ell_{\infty}$.
Fix $\beta \in [0,1]$. A parametrization of
\corralexpf guarantees
\begin{align*}
\forall h \in [0,1] : \Reg(T,\Pi_h) \leq \otil\rbr{T^{\frac{1}{1+\beta}}(\log |\Pi|)^{\frac{\beta}{1+\beta}} h^{-d\beta} }.
\end{align*}
In the non-contextual setting, \corralexpf achieves
$\otil\rbr{T^{\frac{1}{1+\beta}} h^{-d\beta} }$ uniformly smoothed
regret.  \df{Moreover, in the non-contextual \df{stochastic} setting,
  there exist positive constants $c$ and $T_0$, such that for any
  algorithm and any $T \geq T_0$, there exists some $h \in (0,1]$ and
a problem instance for which
\begin{align*}
\Reg(T,\Pi_h) \geq c \cdot T^{\frac{1}{1+\beta}} h^{-d \beta}.
\end{align*}}
\end{theorem}
}

Compared to \pref{thm:refined_adaptivity}, the dependence on $T,|\Pi|$, and $\kappa$ in the regret bound is unchanged. Indeed, for the rectangular kernel $K_h$ in $d$ dimensions, we have $\kappa = O(h^{-d})$ and of course $\kappa_\star = O(1)$ here. Therefore, the main improvement is that we have eliminated the dependence on the number of kernels, $M$.  We also provide a refinement for the non-contextual version, eliminating the dependence on $\log |\Pi|$, which is infinite in this case.

\df{Turning to the lower bound, observe that the lower bound
  in~\pref{thm:smooth_adaptive} is precisely the second claim here
  with $d=1$. This result, coupled with the upper bound for the
  non-contextual version establishes the optimal uniformly-smoothed
  regret rate. It further implies a lower bound for competing with
  multiple arbitrary kernels. Specfically, there exist an action
  space, two kernels $K_1$ and $K_2$, and positive constants $c$,
  $T_0$, such that for any algorithm and any $T \geq T_0$, there
  exists an instance for which either
\begin{align*}
\Reg(T,\Pi_{K_1}) \geq c \cdot T^{\frac{1}{1+\beta}}\kappa_1^{\frac{\beta}{1+\beta}} \quad \textrm{or} \quad \Reg(T,\Pi_{K_2}) \geq c \cdot T^{\frac{1}{1+\beta}}\kappa_2^{\frac{\beta}{1+\beta}} (\nicefrac{\kappa_2}{\kappa_1})^{\frac{\beta^2}{1+\beta}}.
\end{align*}
This confirms the near-optimality of~\pref{thm:general_adaptivity}.
}

\subsection{Results for Lipschitz losses}

Let us turn to Lipschitz contextual bandits, where we posit the Lipschitz condition \eqref{eq:Lip-condition} with Lipschitz  constant $L\geq 1$. The uniform discretization approach applies to general action spaces, \asedit{and yields a suitable generalization of regret bound \eqref{eq:Lip-worst-case-1d}. The latter is stated in terms of the covering dimension $d$ (recall \pref{def:CovDim}):}
\begin{align}\label{eq:Lip-worst-case-regret}
\Reg(T,\Pi) = \tilde{O}
    \rbr{ T^{1-a}\; L^{1-2a}\;
        \rbr{\gamma \log |\Pi|}^a },
    \text{where}\quad a = \tfrac{1}{2+d}.
\end{align}
This regret bound is a departure point for several results presented below.

\OMIT{ This regret bound is the best possible for the non-contextual case \citep{Bobby-nips04}. \refeq{eq:Lip-worst-case-regret} and its special cases are the only previously known results for Lipschitz contextual bandits with policy sets,\footnote{$T^{\frac{d+p+1}{d+p+2}}L^{\frac{p+1}{d+p+2}}$ regret for the
Lipschitz contextual bandit setting of~\citet{cesa2017algorithmic} (with
 $p$-dimensional context space and $d$-dimensional action space) can be obtained as a special case of \refeq{eq:Lip-worst-case-regret} via suitable discretization.}
and it is the point of departure for several results presented below.

Doubling dimension is lower-bounded by the covering dimension.\footnote{For a metric space of diamter $1$, the \emph{covering dimension} (with multiplier $\gamma$) is the smallest $d\geq 0$ such that the metric space can be covered with $\gamma\cdot r^{-d}$ balls of radius $r$, for each $r\in (0,1]$. This notion is used to express worst-case regret bounds for Lipschitz bandits
\citep{Bobby-nips04,kleinberg2013bandits,bubeck2011x,slivkins2014contextual}.
\label{fn:covDim}}
} 
\OMIT{ We use $\inner{\cdot}{\cdot}$ to denote the standard L_2(\base)$ inner product,
$\inner{f}{g} \defeq\int f(a)g(a)d\base(a)$
and at times we write $\inner{a}{f} \defeq f(a)$ for $a \in \Acal$.

Translating from $\Pi_K$ to $\Pi$ with Lipschitz losses is facilitated
by the following lemma, which generalizes~\pref{lem:smooth_to_lip}.
\begin{lemma}[Smooth to Lipschitz]
\label{lem:smooth_to_lipschitz}
If $\forall a \in \Acal$, $\supp(Ka) \subseteq \ball(a,h)$ and $\ell$ is
$L$-Lipschitz, then
\begin{align*}
\abr{\inner{Ka}{\ell} - \ell(a)} \leq Lh.
\end{align*}
\end{lemma}
\begin{proof}
We have,
\[
\abr{\inner{Ka}{\ell} - \ell(a)} = \abr{\EE_{b \sim Ka}\ell(b) - \ell(a)} \leq \abr{L \EE_{b \sim Ka}\rho(a,b)} \leq Lh,
\]
which yields the result.
\end{proof}
} 

\subsubsection*{Stochastic setting: instance-dependent results}

\OMIT{ In the stochastic setting, recall that $\lambda_0(\pi)$ is unsmoothed
expected loss for policy $\pi$, so that $\Pi_{0,\eps}$ are the
$\eps$-optimal policies on the unsmoothed loss. As in the prequel,
the \emph{policy zooming coefficient} is
\begin{align*}
\psi_L(\eps_0) \defeq \sup_{\eps \geq \eps_0}M_0(12L\eps,\eps)/\eps.
\end{align*}
Recall also the \emph{zooming dimension}, which for zooming constant
$\gamma>0$, is
\begin{align*}
z \defeq \inf\cbr{ d' > 0: M_0(12L\eps,\eps) \leq \gamma \eps^{-d'}, \forall \eps \in (0,1)}.
\end{align*}
} 

We make several minor modifications
to \zpelip, as in~\pref{sec:results-smoothed}. We use rectangular kernel $K_{\df{h}}$ instead of the $\smooth$
operator.
We also set
the parameters as follows: \df{recall from~\pref{sec:results-lip} that instead of using a single smoothing parameter throughout, \zpelip uses bandwidth $h_m=2^{-m}$ at epoch $m$. In addition,} we set\dfc{$h_m=2^{-m}$}, $r_m=L2^{-m}$, $V_m :=
\EE_{x \sim \Dcal}\base\rbr{\bigcup_{\pi \in
    \Pi^{(m)}}\ball_{h_m}(\pi(x))}$ and $\tilde{n}_m := \frac{320
  \kappa_{h_m} V_m}{r_m^2}$.

\begin{theorem}[generalizes~\pref{thm:stoch_ring_zooming}]
\label{thm:zooming}
Consider the stochastic setting under \pref{assum:uniform}. Recall policy-zooming coefficient $\psi_L(\eps_0)$ and zooming dimension $z$ (with constant $\gamma)$, as defined in
\eqref{eq:policy-zooming-coeff} and \eqref{eq:zooming-dim-defn}. In this setting,
\zpelip with parameter $L$ achieves
\begin{align*}
\Reg(T,\Pi)
&\leq O\rbr{\inf_{\eps_0 > 0} TL\eps_0 +
\frac{\alpha\cdot \psi_L(\eps_0)}{L} \cdot \log(|\Pi|\df{T})\log(1/\eps_0)} \\
&\leq \tilde{O}
    \rbr{ T^{1-a}\; L^{1-2a}\;
        \rbr{\gamma \log |\Pi|}^a },\;
    \text{where}\quad a = \tfrac{1}{2+z}.
\end{align*}
\end{theorem}

\noindent It is easy to see that the zooming dimension is upper-bounded by the covering dimension.

\subsubsection*{Adversarial setting}

Our algorithm for smoothed regret in the adversarial setting  --- a suitably parameterized version of \expfour in \pref{thm:smooth_exp4_general} --- yields meaningful guarantees for the Lipschitz setting, and in fact \asedit{essentially recovers the optimal} regret rate in \eqref{eq:Lip-worst-case-regret}.

\begin{corollary}
\label{cor:smooth_exp4_lip}
Consider the adversarial setting. Let $K_h$, $h>0$ be a rectangular kernel, and let $\kappa_h$ be its kernel complexity. Consider \smoothexp (\pref{alg:smooth_exp4}) parametrized as in \pref{thm:smooth_exp4_general}: policy set $\Xi = \Pi_{K_h}$ and and learning rate
    $\eta = \sqrt{\frac{2\,\ln|\Xi|}{T\kappa}}$.
Then
\begin{align}
\Reg(T,\Pi) \leq TLh + O\rbr{T\kappa_h \log |\Pi|}. \label{eq:exp4_lip-kappa}
\end{align}
\asedit{We recover the regret bound \refeq{eq:Lip-worst-case-regret}  in terms of the covering dimension $d$, up to the multiplicative factor of $2^d$,} for suitable choice of bandwidth
    $h = \Theta( (\log |\Pi|/T)^{\frac{1}{d+2}}L^{\frac{-2}{d+2}})$.
\end{corollary}

\noindent This is an immediate consequence of~\pref{thm:smooth_exp4_general} and the following simple fact:

\begin{lemma}[generalizes \pref{lem:smooth_to_lip}]
\label{lem:smooth_to_lip_general}
Let $K_h$ be a rectangular kernel, and $f:\Acal \to [0,1]$ be an $L$-Lipschitz function. Then
$\abr{\EE_{a' \sim
    K_h(a)} f(a') - f(a)} \leq Lh$.
\end{lemma}

\OMIT{ Instantiating~\pref{cor:smooth_exp4_lip} in special cases, we get:
\begin{itemize}
\item $T^{2/3}(L\log|\Pi|)^{1/3}$-Lipschitz regret for the unit interval metric.
\item The optimal $T^{\frac{d+1}{d+2}}L^{\frac{d}{d+2}}$-Lipschitz
  regret for non-contextual problems with $d$-dimensional metric,
  matching prior
  results~\citep{kleinberg2013bandits,bubeck2011lipschitz}.
\item $T^{\frac{d+p+1}{d+p+2}}L^{\frac{p+1}{d+p+2}}$ for the
  Lipschitz contextual bandit setting of~\citet{cesa2017algorithmic} with
  $p$-dimensional context space and $d$-dimensional action space,
  which matches their result.
\end{itemize}
} 
\subsubsection*{Lipschitz-adaptivity}

We extend~\pref{thm:lipschitz_adaptive} to higher dimension, specifically to metric space $([0,1]^d,\; \ell_\infty)$.

\begin{theorem}
\label{thm:general_lipschitz_adaptivity}
Consider the adversarial setting, with action space $\Acal = [0,1]^d$, $d\in\NN$ and base metric $\rho = \ell_\infty$. \pref{thm:lipschitz_adaptive} extends, with exponents
    $a = \frac{\beta}{1+d\beta+\beta}$
and
    $b = \frac{d\beta}{1+d\beta}$.
\end{theorem}

\OMIT{ \begin{theorem}
\label{thm:general_lipschitz_adaptivity}
Fix $\beta \in [0,1]$ and $d \in \NN$ \df{and consider the adversarial setting}. For $\Acal =
[0,1]^d$, a parametrization of \corralexpf
guarantees
\begin{align*}
\Reg(T,\Pi) \leq \otil\rbr{T^{\frac{1+d\beta}{1+(d+1)\beta}}(\log |\Pi|)^{\frac{\beta}{1+(d+1)\beta}} L^{\frac{d\beta}{1+d\beta}} },
\end{align*}
when losses are $L$-Lipschitz. Crucially the algorithm does not need
to know $L$. In the non-contextual setting, a parametrization achieves
a regret of $\otil\rbr{ T^{\frac{1+d\beta}{1+(d+1)\beta}}
  L^{\frac{d\beta}{1+d\beta}} }$.  Moreover, \df{in the non-contextual
  stochastic setting, there exist positive constants $c$ and $T_0$,
  such that for any algorithm and any $T \geq T_0$, there exists a
  value $L > 0$ and an instance with $L$-Lipschitz losses, such that
\begin{align*}
\Reg(T,\Pi) \geq c \cdot T^{\frac{1+d\beta}{1+(d+1)\beta}} L^{\frac{d\beta}{1+d\beta}}.
\end{align*}}
\end{theorem}
}

\section{Analysis: instance-dependent regret bounds}
\label{sec:proofs-instance-dependent}

We prove both instance-dependent regret bounds: \pref{thm:smooth_pe_weak} for smoothed regret and~\pref{thm:zooming} for Lipschitz losses. In fact, we present a joint proof for both results.

\subsection{Auxiliary lemmas}

We start by stating two auxiliary lemmas whose proofs are deferred to the end of this section. Recall that the marginal distribution over
$\Xcal$, \df{denoted $\Dcal_X$,} is assumed to be known.

The first lemma provides a guarantee on the optimization
problem~\pref{eq:pe_optimization-body}. For a policy set $\Pi' \subset
\Pi$, bandwidth $h$ and context $x$, define $A(x; \Pi',h) \defeq
\bigcup_{\pi \in \Pi'} \ball(\pi(x),h) = \bigcup_{a \in \Pi'(x)}
\ball(a,h)$ which is a subset of the action space. Similarly, let
$V(\Pi',h) = \EE_{\df{x\sim\Dcal_X}} \base (A(x;\Pi',h))$. Finally, for a distribution
$Q \in \Delta(\Pi')$, bandwidth $h$, we define
\dfc{the }\df{its induced} action-selection \df{density}\dfc{distribution} as
\begin{align*}
q(a \mid x) \defeq \sum_{\pi \df{\in \Pi'}}Q(\pi)(K_h\pi(x))(a).
\end{align*}
\df{Note that this is the density over the action space of the
  action-selection distribution induced by $Q$ on context $x$.}

\begin{lemma}
\label{lem:pe_duality_app}
For any subset $\Pi' \subset \Pi$ with $|\Pi'| < \infty$, \df{any bandwidth $h > 0$,} and any
\df{data} distribution $\Dcal_X \dfc{\in\Delta(\Xcal)}$, the
program~\pref{eq:pe_optimization-body} is convex and we have
\begin{align*}
\min_{Q \in \Delta(\Pi')} \max_{\pi \in \Pi'} \EE_{x \sim \Dcal_X}\EE_{a \sim K_{\df{h}}\pi(x)}\sbr{\frac{1}{q(a \mid x)}} \leq V(\Pi',h).
\end{align*}
\end{lemma}
Note that $V(\Pi',h) \leq 1$, which yields a weaker, but more
interpretable bound.

The following lemma gives a uniform deviation bound on
$\hat{L}_m(\pi)$ and $\EE_{(x,\ell) \sim \Dcal}
\inner{K_{h_m}\pi(x)}{\ell}$ in epoch $m$.  Recall that in epoch $m$,
the estimator $\hat{L}_m(\pi)$ is the median of several base
estimators $\cbr{\hat{L}_m^i(\pi)}_{i=1}^I$, where $I = \delta_T =
5\lceil \log(|\Pi|\log_2 (T)/\delta) \rceil$ is the number of
batches. In comparison to using the naive empirical mean estimator,
this median-of-means estimator has the advantage that it avoids a
dependency on the range of the individual losses, therefore admitting
sharper concentration.

 \begin{lemma}[Concentration of median-of-means loss estimator]
 \label{lem:weak_concentration}
 Fix $\Pi' \subset \Pi$, $h \in (0,1)$, $\delta \in (0,1)$ and let $Q \in
 \Delta(\Pi')$ be the solution to~\pref{eq:pe_optimization-body}. Let $I = 5\lceil \log(|\Pi|/\delta) \rceil$, $\tilde{n}$ be an integer, and
 $\{x_j,a_j,\ell_j(a_j)\}_{j=1}^n$ be a dataset of $n = I \tilde{n}$ samples, where
 $(x_j,\ell_j) \sim \Dcal$ and $a_j \sim q(\cdot \mid x_j)$.
 Define
 \[ \hat{L}(\pi) = \median(\hat{L}^1(\pi), \ldots, \hat{L}^I(\pi)), \]
 where $\hat{L}^i(\pi) = \frac{1}{\tilde{n}} \sum_{j=(i-1)\tilde{n}+1}^{i\tilde{n}} \frac{K_h(\pi(x_j))(a_j)}{q(a_j \mid x_j)} \ell_j(a_j)$.
 Then
 with probability at least $1-\delta$, for
 all $\pi \in \Pi'$, we have
 \begin{align*}
 \abr{\lambda_h(\pi) - \hat{L}(\pi)} \leq \sqrt{\frac{80 \kappa_h V(\Pi',h)}{n}\log(e |\Pi|/\delta)}.
 \end{align*}
 \end{lemma}

\subsection{Proof of \pref{thm:smooth_pe_weak} and~\pref{thm:zooming}}

The proof proceeds inductively over the epochs and we will do both
proofs simultaneously. In the proof of~\pref{thm:smooth_pe_weak} we
use $L(\pi)\defeq \lambda_h(\pi)$, while
for~\pref{thm:zooming} we use $L(\pi) \defeq \lambda_0(\pi) = \EE\ell(\pi(x))$.
In both
cases $\pi^\star \defeq \argmin_{\pi \in \Pi}L(\pi)$. For both proofs
we use $L_m(\pi) \defeq \lambda_{h_m}(\pi)$, noting that
for~\pref{thm:smooth_pe_weak}, $L_m(\pi) = L(\pi)$.  Recall the
definitions of the ``radii'' $r_m$ which are either $2^{-m}$ or
$L2^{-m}$ depending on the theorem statement.  In epoch $m$ we prove
two things, inductively:
\begin{enumerate}
\item $\pi^\star \in \Pi_{m+1}$ (assuming inductively that $\pi^\star \in \Pi_m$).
\item For all $\pi \in \Pi_{m+1}$ we have $L(\pi) \leq L(\pi^\star) +
  12r_{m+1}$.
\end{enumerate}
Before proving these two claims, we first lower bound $n_m$ which
provides a bound on the number of epochs. Assuming $\pi^\star \in
\Pi_m$, which we will soon prove, we have
\begin{align*}
n_m \geq \frac{\kappa_{h_m}V_m}{r_m^2} \geq \frac{\kappa_{h_m} \EE_{\df{x\sim\Dcal_X}} \base(\ball(\pi^\star(x),h_m))}{r_m^2} \geq \frac{1}{r_m^2} = 2^{2m}
\end{align*}
The first inequality requires $\delta_T \geq 1$ (which follows since
$\delta \leq 1/e$) while the third uses the fact that
$\textrm{supp}(K_{h_m}(a)) \subset \ball(a,h_m)$ so that $\kappa_{h_m} \geq
\sup_{a} \frac{1}{\base(\ball(a,h_m))}$. Hence we know that there are
at most $m_T \defeq \log_2(T)$ epochs.
Applying~\pref{lem:weak_concentration} to all $m_T$ epochs
and taking a union
bound, we have
\begin{align*}
\forall m \in [m_T], \forall \pi \in \Pi_m: \abr{ L_m(\pi) - \hat{L}_m(\pi)} \leq \sqrt{\frac{80\kappa_{h_m}V_m\delta_T}{n_m}}.
\end{align*}

Here we are using the fact that $V_m = V(\Pi_m,h_m)$ where $V_m$ is
defined in the algorithm.  Plugging in the choices for $n_m \defeq
\frac{320\kappa_{h_m}V_m\delta_T}{r_m^2}$ the above inequality simplifies
to
\begin{align}
\forall m \in [m_T], \forall \pi \in \Pi_m: \abr{ L_m(\pi) - \hat{L}_m(\pi)} \leq r_m/2.
\label{eq:smoothed-loss-conc}
\end{align}

Let us now prove the two inductive claims under the event that these inequalities hold, which occurs with probability \df{at least} $1-\delta$. For the base case, since
$\Pi_1 \gets \Pi$ we clearly have $\pi^\star \in \Pi$. We also always
have $L(\pi) \leq L(\pi^\star) + 2r_1$ since the losses are bounded in
$[0,1]$. For the inductive step, first we observe that for~\pref{thm:smooth_pe_weak},
$L(\pi) = L_m(\pi)$, and for~\pref{thm:zooming},
$|L(\pi) - L_m(\pi)| \leq L h_m = r_m$. In conjunction with~\pref{eq:smoothed-loss-conc}, in both cases, we have
\begin{align}
\forall m \in [m_T], \forall \pi \in \Pi_m: \abr{ L(\pi) - \hat{L}_m} \leq 3r_m/2.
\label{eqn:objective-loss-conc}
\end{align}
By the standard analysis of empirical
risk minimization, for the first claim,
\begin{align*}
\hat{L}_m(\pi^\star)
&\leq L(\pi^\star) + 3r_m/2
= \min_{\pi \in \Pi_m}L(\pi) + 3r_m/2
\leq \min_{\pi \in \Pi_m} \hat{L}_m(\pi) + 3r_m.
\end{align*}
which verifies that $\pi^\star \in \Pi_{m+1}$.
For the second claim, for both~\pref{thm:smooth_pe_weak} and~\pref{thm:zooming}, we have
for all $\pi$ in $\Pi_{m+1}$,
\begin{align*}
L(\pi) \leq \hat{L}_m(\pi) + 3r_m/2 \leq \min_{\pi' \in \Pi_m}\hat{L}_m(\pi') + 9r_m/2 \leq L(\pi^\star) + 6r_m.
\end{align*}
This proves the second claim since $r_m = 2r_{m+1}$.

For the final regret bound, define $\hat{m}_T$ to be the actual number
of epochs. For each $m \in \NN$, define $\hat{n}_m$ to be the actual
number of rounds in each epoch, formally defined as follows: (1) for
$m < \hat{m}_T$, $\hat{n}_m \defeq n_m$, (2) for $m > \hat{m}_T$,
$\hat{n}_m \defeq 0$, and (3) $\hat{n}_{\hat{m}_T} = T -\sum_{m <
  \hat{m}_T} \hat{n}_m$. We have that $\hat{n}_m \leq n_m$ for all $m$
and that $\sum_{m=1}^{\infty}\hat{n}_m= T$. Then, in the $1-\delta$
good event,
we can bound the regret of the algorithm as
\begin{align*}
\Reg \leq \sum_{m=1}^{\infty} \hat{n}_m \cdot 12r_m, \dfc{= \sum_{m=1}^{\infty} 12\hat{n}_mr_m,}
\end{align*}
where we have used the fact that $\sum_{m=1}^{\infty} \hat{n}_m = T$.

We optimize the bound as follows: For any $\epsilon_0 > 0$, we
first truncate the sum at epoch $m_{\epsilon_0} \defeq \lceil \log
\frac{1}{\epsilon_0} \rceil$.  Using the fact that $r_m \leq r_{m_{\epsilon_0}}$
for $m \geq m_{\epsilon_0}$, we can bound the regret in the later
epochs simply by $T \epsilon_0$. For the earlier epochs we substitute
the choice of $\hat{n}_m$. This gives
\begin{align*}
\sum_{m=1}^{\infty} 12\hat{n}_mr_m \leq 12 \min_{\epsilon_0 > 0} \rbr{T \epsilon_0 + 320 \sum_{m \leq m_{\epsilon_0} - 1} \frac{\kappa_{h_m}V_m\delta_T}{r_m}}.
\end{align*}
To simplify further, by our inductive hypothesis we know that
\begin{align*}
V_m \leq V(\Pi_m,h_m) = \EE_{\df{x\sim\Dcal_X}}\base(A(x;\Pi_m,h_m)) \leq \EE_{\df{x\sim\Dcal_X}}\Ncal_{h_m}(\Pi_m(x)) \cdot\sup_{a}\base(\ball(a,2h_m)).
\end{align*}
The final inequality is based on the fact that we can always cover
$A(x;\Pi_m,h_m)$ by a union of balls of radius $2h_m$ with
centers on a $h_m$-covering of $\Pi_m$, along with the fact that a maximum \df{(therefore, maximal)} $\delta$-packing is a $\delta$-covering. On the other hand we
have $\kappa_{h_m} \leq \sup_{a} \frac{1}{\base(\ball(a,h_m))}$, so that
under~\pref{assum:uniform} we have
\begin{align*}
\kappa_{h_m}V_m \leq \alpha \cdot \EE_{\df{x\sim\Dcal_X}}\Ncal_{h_m}(\Pi_m(x))
\end{align*}

Set $\Scal \defeq \{2^{-i}: i \in \NN\}$.
For~\pref{thm:smooth_pe_weak},
using the definition
of $M_h(\epsilon,\delta)$, and the fact that $\Pi_m \subseteq \Pi_{h,12r_m}$,
we have $\kappa_{h_m}V_m \leq \alpha \EE_{\df{x\sim\Dcal_X}}\Ncal_{h_m}(\Pi_m(x)) \leq \alpha M_h(12r_m, \df{h})$ .
Therefore, the bounds simplify to
\begin{align*}
\Reg(T,\Pi_h) &\leq 12 \min_{\epsilon_0 > 0} \rbr{  T\epsilon_0 + 320 \alpha \cdot \sum_{\epsilon \in \Scal, \epsilon \geq 2\epsilon_0} \frac{M_h(12\epsilon, \df{h}) \delta_T}{\epsilon} }\\
&\leq 12 \min_{\epsilon_0 > 0} \rbr{  T\epsilon_0 + 320 \alpha \cdot \theta_h(\epsilon_0) \cdot \log(|\Pi| \log_2(T)/\delta) \cdot \log_2(1/\epsilon_0) },
\end{align*}
where in the second inequality, we use the definition of $\theta_h(\epsilon)$, and the fact that there are $m_{\epsilon_0} - 1 \leq \log_2(1/\epsilon_0)$ summands in the second term.

Likewise, for~\pref{thm:zooming}, we have
\begin{align*}
\Reg(T,\Pi) &\leq  12 \min_{\epsilon_0 > 0} \rbr{  TL\epsilon_0 + 320 \alpha \sum_{\epsilon \in \Scal, \epsilon \geq 2\epsilon_0} \frac{M_0(12L\epsilon, \epsilon) \delta_T}{L\epsilon} }\\
&\leq 12 \min_{\epsilon_0 > 0} \rbr{ T \df{L} \epsilon_0 + 320 \alpha \cdot \psi_L(\epsilon_0)/L \cdot \log(|\Pi| \log_2(T)/\delta) \cdot \log_2(1/\epsilon_0)}
\end{align*}

Both bounds are conditional on the good event, which happens
with probability $1-\delta$. In the bad event, the expected regret is
at most $T$. Setting $\delta = 1/T$, the theorems follow.

\subsection{Proofs for the lemmata}

\ifthenelse{\equal{\version}{arxiv}}{\begin{proof}[Proof of~\pref{lem:pe_duality_app}]}{\begin{proofof}{\pref{lem:pe_duality_app}}}
The proof follows that of Lemma 1
of~\citet{dudik2011efficient}.
We introduce the following notation: for a distribution $P$ over a set of policies $\Pi'$, bandwidth $h$,
denote by its induced action-selection \df{density}\dfc{ distribution} as
\[ p(a \mid x) \defeq \sum_{\pi \in \Pi} P(\pi) (K_h\pi(x))(a). \]
Likewise, for a distribution $Q$ over a set of policies $\Pi'$, define
\[ q(a \mid x) \defeq \sum_{\pi \in \Pi} Q(\pi) (K_h\pi(x))(a). \]
Define $\one_{|\Pi'|}$ to be the $|\Pi'|$-dimensional vector that takes value $1$ on all its
entries; in addition, for policy $\pi$ in $\Pi'$, define $e_\pi$ as the $|\Pi'|$-dimensional vector that takes
value $1$ on the entry that corresponds to policy $\pi$ and takes value $0$ everywhere else.

In addition, for $Q$ and $P$ in $\Delta(\Pi')$, define
\[
f(Q,P) :=
\EE_{x \sim \Dcal_X} \int \frac{p(a \mid x)}{q(a \mid x)} \one(a \in A(x ; \Pi', h)) d\base(a).
\]
It suffices to show that
$\min_{Q \in \Delta(\Pi')} \max_{\pi \in \Pi'} f(Q,e_\pi) \leq V(\Pi', h)$,
as for any $\pi$ in $\Pi'$,
\begin{eqnarray*}
f(Q,e_\pi)
&=& \EE_{x \sim \Dcal_X} \int \frac{K_h\pi(x) (a)}{q(a \mid x)} \one(a \in A(x ; \Pi', h)) d\base(a) \\
&=& \EE_{x \sim \Dcal_X} \int \frac{K_h\pi(x) (a)}{q(a \mid x)} d\base(a) \\
&=& \EE_{x \sim \Dcal_X} \EE_{a \sim K_h \pi(x)} \frac{1}{q(a \mid x)},
\end{eqnarray*}
where the first equality is from the fact that for all $a$, if $K_h\pi(x) (a) \neq 0$ then
$a \notin A(x ; \Pi', h)$.

Define $\QQ \defeq \cbr{Q \in \Delta(\Pi'):  \max_{\pi \in \Pi'}
f(Q,e_\pi) < \infty}$.
Observe that $\QQ$ is a convex set. $\QQ$ is nonempty, as any vector $Q$
such that $Q_\pi > 0$ for all $\pi$ in $\Pi'$ (e.g. the uniform distribution over
$\Pi'$, $\frac{1}{|\Pi'|} \one_{|\Pi'|}$) is in $\QQ$. With this notation,
\[
\min_{Q \in \Delta(\Pi')} \max_{\pi \in \Pi'} f(Q,e_\pi)
=
\min_{Q \in \QQ} \max_{\pi \in \Pi'} f(Q,e_\pi).
\]

Now, note that
\[
\df{\min_{Q \in \QQ}} \max_{\pi \in \Pi'} f(Q,e_\pi)
=
\df{\min_{Q \in \QQ}} \max_{P \in \Delta(\Pi')} \EE_{\pi \sim P} f(Q,e_\pi)
=
\df{\min_{Q \in \QQ}} \max_{P \in \Delta(\Pi')} f(Q, P).
\]
where the first equality uses the fact that $f(Q, \cdot)$ is
linear.
Now, as $\Delta(\Pi')$ is compact and convex,
$\QQ$ is convex,
$f(\cdot, P)$ is convex and continuous and
$f(Q, \cdot)$ is concave and continuous, we may apply
Sion's minimax theorem~\cite[][Corollary 3.3]{sion1958general},
to obtain that the above is equal to
\[
\df{\max_{P \in \Delta(\Pi')}} \df{\min_{Q \in \QQ}} f(Q, P)
\]
Now, given any $P$ in $\Delta(\Pi')$, consider
$P_\epsilon = (1-\epsilon) P + \frac{\epsilon}{|\Pi'|} \one_{|\Pi'|}$.
We have that $P_\epsilon$ is in $\QQ$. Moreover,
\begin{eqnarray*}
f(P_\epsilon, P)
&\leq&
\EE_{\pi \sim P} \EE_{x \sim \Dcal_X} \int \frac{p(a \mid x)}{(1-\epsilon)p(a \mid x)}
\one(a \in A(x ; \Pi', h)) d\base(a) \\
&=& \frac{1}{1-\epsilon} \EE_{x \sim \Dcal_X} \base(A(x ; \Pi', h)) = \frac{1}{1-\epsilon} V(\Pi', h).
\end{eqnarray*}
Letting $\epsilon \to 0$, this implies that for any $P$ in $\Delta(\Pi')$,
$\inf_{Q \in \QQ} f(Q, P) \leq V(\Pi', h)$.
Therefore,
\[ \df{\max_{P \in \Delta(\Pi')}} \df{\min_{Q \in \QQ}} f(Q, P) \leq V(\Pi', h). \]
The lemma follows.
\ifthenelse{\equal{\version}{arxiv}}{\end{proof}}{\end{proofof}}

\ifthenelse{\equal{\version}{arxiv}}{\begin{proof}[Proof of~\pref{lem:weak_concentration}]}{\begin{proofof}{\pref{lem:weak_concentration}}}
First, as we have seen, $\EE
\hat{\ell}_i(\pi(x_i)) = \lambda_h(\pi)$. Moreover,
\begin{align*}
\Var\rbr{\hat{\ell}_i(\pi(x_i))} &\leq \EE\sbr{\hat{\ell}_i(\pi(x_i))^2} = \EE_{(x,\ell)\sim\Dcal}\sbr{\int \frac{(K_h\pi(x))^2(a)\ell(a)^2}{q(a|x)}d\base}\\
&\leq \kappa_h V(\Pi',h) \leq \kappa_h V(\Pi',h).
\end{align*}
where the penultimate inequality uses the fact that $Q$ is the
solution to~\pref{eq:pe_optimization-body}, so it satisfies the guarantee
in~\pref{lem:pe_duality_app}.
Therefore, using~\pref{lem:median-of-means} below, we have that for every $\pi \in \Pi'$, with
probability at least $1-\frac{\delta}{|\Pi|}$, the following holds:
\begin{align*}
\abr{ \bar{L}(\pi) - \hat{L}(\pi)} \leq \sqrt{\frac{80\kappa_hV(\Pi',h)}{n}\log(e|\Pi|/\delta)}.
\end{align*}
The lemma is concluded by taking a union bound over all $\pi$ in $\Pi'$.
\ifthenelse{\equal{\version}{arxiv}}{\end{proof}}{\end{proofof}}

\section{Analysis: smoothness-adaptive guarantees}
\label{sec:proofs-corral}

\asedit{We now turn to smoothness-adaptive guarantees:
\pref{thm:general_adaptivity} and \pref{thm:refined_adaptivity} for  uniformly-smoothed regret bounds, and \pref{thm:general_lipschitz_adaptivity} for Lipschitz-adaptive regret bounds. We prove these results via a joint exposition, building on \corral algorithm from \citet{agarwal2016corralling}.}

\subsection{Stability of \smoothexp}
\label{sec:EXP4-stability}

\asedit{We start with a result on the \emph{stability} of \smoothexp.} We consider a slightly modified protocol.  The learner is now presented with randomized loss
functions $\ell_t$ which are generated by importance weighting an
original loss function $\bar{\ell}_t$ with some probability $p_t$ set
by the adversary. Formally $\ell_t = Q_t\bar{\ell}_t/p_t$ where $Q_t
\sim \textrm{Ber}(p_t)$ at each round $t$. Note that here, the losses
presented to the learner are not guaranteed to be bounded, but we do
have variance information, via $p_t$. The original losses
$\bar{\ell}_t$ are bounded in $[0,1]$. Note further that $p_t$ is
revealed at the beginning of round $t$.

In this setup,~\cite{agarwal2016corralling} define the following
notion of stability.
\begin{definition}[See~\citet{agarwal2016corralling}, Definitions 3 and 14]
\label{def:stability}
An algorithm with policy class $\Xi$ is called \emph{$(\beta,
  R(T))$-stable}, if in the above protocol it achieves
\begin{align}
\EE \sum_{t=1}^T \bar{\ell}_t(a_t) - \min_{\xi \in \Xi} \EE \sum_{t=1}^T \inner{\xi(x_t)}{\bar{\ell}_t}
\leq
\EE[\rho]^\beta \cdot R(T),
\label{eq:stability}
\end{align}
where $\rho \defeq \max_{t \in [T]} \frac{1}{p_t}$.
\end{definition}

The definition here is slightly different than the one
in~\citet[][Definition 3]{agarwal2016corralling}, in that the right
hand side has the term $\EE[\rho]^\beta$ instead of
$\EE[\rho^\beta]$. This has no bearing on the analysis of \corral, but
is important for our application, as we will see.

\citet{agarwal2016corralling} shows that \expf is $(\nicefrac{1}{2},
\sqrt{K T \log|\Pi|})$-stable in the discrete action setting, where
$K$ is the number of actions.  We provide a similar result here,
replacing $K$ with $\kappa$ and establishing stability whenever the
first parameter is in $[0,\nicefrac{1}{2}]$.

\begin{theorem}
\label{thm:strong_stability}
\pref{alg:smooth_exp4} is $\rbr{\frac{\beta}{1+\beta},
  O\rbr{T^{\frac{1}{1+\beta}}(\kappa \log
    |\Xi|)^{\frac{\beta}{1+\beta}}}}$-stable, for each $\beta \in
     [0,1]$.
\end{theorem}
\begin{proof}For this proof only, we use $\xi(x_t)$ to denote the
\df{density for the action distribution of expert $\xi$ on context $x_t$, with respect to $\base$}. Thus the expected loss
for expert $\xi$ on round $t$ is $\inner{\xi(x_t)}{\ell_t}$.

We first show a weaker form of stability. Suppose that $\hat{\rho}
\geq \max_{t \in [T]} \nicefrac{1}{p_t}$ is provided to the algorithm
ahead of time. Then following the analysis for \expfour, we
have
\begin{align*}
\EE\sum_{t=1}^T\ell_t(a_t) - \min_{\xi \in \Xi} \EE \sum_{t=1}^T\inner{\xi(x_t)}{\ell_t} \leq \EE \frac{\eta \kappa}{2}\sum_{t=1}^T \|\ell_t\|_{\infty}^2 + \frac{\log |\Xi|}{\eta}.
\end{align*}
The key observation is that,
\begin{align*}
\EE\sum_{t=1}^T \|\ell_t\|_{\infty}^2 \leq \EE\sum_{t=1}^T \frac{Q_t}{p_t^2} = \sum_{t=1}^T \EE\nicefrac{1}{p_t} \leq T\hat{\rho}
\end{align*}
Therefore, with the choice of $\eta = \sqrt{\frac{2 \log
    |\Xi|}{T\kappa \hat{\rho}}}$, and using the fact that
the conditional expectation of $\ell_t$ is $\bar{\ell}_t$, we get
\begin{align}
\EE\sum_{t=1}^T \bar{\ell}_t(a_t) - \min_{\xi \in \Xi}\EE\sum_{t=1}^T \inner{\xi(x_t)}{\bar{\ell}_t} \leq \sqrt{2\kappa T \log |\Xi| \cdot \hat{\rho}}.
\label{eq:weak_stability}
\end{align}
This proves a weaker version of stability, where a bound on $\rho$ is
specified in advance. The stronger version is based on the ``doubling
trick'' argument in~\citet[Theorem 15]{agarwal2016corralling}.  We run
\expfour with a guess for $\hat{\rho}$ and if we experience a round
$t$ where $\nicefrac{1}{p_t} > \hat{\rho}$, we double our guess and
restart the algorithm, always with learning rate $\eta = \sqrt{\frac{2
    \log |\Xi|}{T\kappa \hat{\rho}}}$. In their Theorem 15, they prove
that if an algorithm is weakly stable in the sense
of~\pref{eq:weak_stability} then, with restarts, it is strongly stable
according to~\pref{def:stability}. In our setting, their result
reveals that the restarting variant of \expfour guarantees
\begin{align*}
\EE\sum_{t=1}^T \bar{\ell}_t(a_t) - \min_{\xi \in \Xi}\EE\sum_{t=1}^T \inner{\xi(x_t)}{\bar{\ell}_t} \leq \frac{\sqrt{2}}{\sqrt{2}-1}\cdot \EE[\rho]^{\frac{1}{2}} \cdot \sqrt{2\kappa T \log |\Xi|}
\end{align*}
To obtain a stability guarantee for every $\beta$, since the regret is trivially at most $T$, we obtain
\begin{align*}
\EE\sum_{t=1}^T \bar{\ell}_t(a_t) - \min_{\xi \in \Xi}\EE\sum_{t=1}^T \inner{\xi(x_t)}{\bar{\ell}_t} \leq \min\rbr{T, c \EE[\rho]^{\frac{1}{2}}\sqrt{\kappa T \log |\Xi|}}
\leq c T^{\frac{1}{1+\beta}} (\EE[\rho] \kappa \log |\Xi|)^{\frac{\beta}{1+\beta}}.
\end{align*}
where $c> 0$ is a universal constant. The second inequality is from
the simple fact that $\min\rbr{A,B} \leq A^{\gamma}B^{(1-\gamma)}$ for
$A,B > 0$, $\gamma \in [0,1]$.
\end{proof}

\subsection{\corralexpf and its analysis}
\label{sec:corralexpf-analysis}

\df{We first provide a formal description of \corralexpf in 
the notation of abstract smoothing kernels.} Given
a family of smoothing kernels $\Kcal= \cbr{K_1,\ldots,K_M}$, we bucket the kernels according
to their kernel complexity $\kappa$, $\Kcal_b = \cbr{i \in [M]: \lceil \log
  \kappa_{K_i} \rceil = b}$ for each $b \in \NN$, and we initialize
one instance of \expfour with restarting for each bucket. Then we run
\corral over these instances. We call $B = \cbr{\lceil \log
  \kappa_{K_i} \rceil: i \in \cbr{1,\ldots,M}}$
  the set of ``active'' indices.

Define $r \defeq \frac{\max_{K \in \Kcal} \kappa_K}{\min_{K \in \Kcal}
  \kappa_K}$ and $\kappa_\star \defeq \min_{K \in \Kcal} \kappa_K$.
Observe that $B \leq \min\cbr{M, \log r + 1}$. We have the following
guarantee for \corralexpf.

\begin{lemma}
\label{lem:corral-kernel}
Suppose \corralexpf is run with learning rate $\eta$
and horizon $T$. Then, for all $\beta \in [0, 1]$, it has the
following regret guarantee simultaneously for all kernels $K$ in $\Kcal$:
\[
\Reg(T,\Pi_{K})
\leq
\otil\rbr{
\frac{\min\cbr{M, \log r}}{\eta} + T \eta + T \rbr{\eta \ln (|\Pi|M) \kappa_K}^{\beta}
}.
\]
\end{lemma}
\begin{proof}
  This is almost a direct consequence of~\citet[Theorem
    4]{agarwal2016corralling}.  By the definition of $\Kcal_b$ and
  $\Xi_b$, $\kappa_b \defeq \max_{\xi \in \Xi_b} \max_{a,x} \xi(a|x)
  \leq 2^b$.  In addition, $|\Xi_b| \leq |\Pi| \cdot M$. Since for all
  $K \in \Kcal_b$ we have $\lceil \log\kappa_K \rceil = b$, therefore
  $\kappa_K \in (2^{b-1},2^b]$. By
    applying~\pref{thm:strong_stability} we see that \expfour with
    restarting has the stability guarantee when measuring regret
    against $\bench(\Pi_{K_i})$ for each $K_i \in \Kcal_b$.

    Now, by Theorem 4 of~\citep{agarwal2016corralling}, \corral ensures
\begin{align*}
\forall b \in [B], \forall K \in \Kcal_B: \Reg(T,\Pi_{K}) \leq \otil\rbr{ \frac{B}{\eta} + T\eta - \frac{\EE[\rho_b]}{\eta \log T} + T^{\frac{1}{1+\beta}}\rbr{\EE[\rho_b] \kappa_K \log (|\Pi| M)}^{\frac{\beta}{1+\beta}}}
\end{align*}
Optimizing over $\EE[\rho_b]$ gives
\begin{align*}
\forall K \in \Kcal: \Reg(T,\Pi_{K}) \leq \otil\rbr{ \frac{B}{\eta} + T\eta + T\rbr{\eta \kappa_K \log (|\Pi| M)}^{\beta}}.
\end{align*}
The result follows by observing that $B \leq \min\cbr{M,\log r}$.
\end{proof}

\ifthenelse{\equal{\version}{arxiv}}{\begin{proof}[Proof of upper bound in~\pref{thm:general_adaptivity}]}{\begin{proofof}{upper bound in~\pref{thm:general_adaptivity}}}
We simply run \corralexpf with
\begin{align*}
\eta = \frac{B^{\frac{1}{1+\beta}}}{{T^{\frac 1 {1+\beta}} (\ln(|\Pi| M) \kappa_\star)^{\frac \beta {1+\beta}}}},
\end{align*}
and apply~\pref{lem:corral-kernel}.
\ifthenelse{\equal{\version}{arxiv}}{\end{proof}}{\end{proofof}}

\ifthenelse{\equal{\version}{arxiv}}{\begin{proof}[Proof of upper bounds in~\pref{thm:refined_adaptivity}]}{\begin{proofof}{upper bounds in~\pref{thm:refined_adaptivity}}}
Recall that for~\pref{thm:refined_adaptivity} we are in the
$d$-dimensional cube with uniform base measure and with
$\ell_{\infty}$ metric. Our goal is to obtain a uniformly-smoothed
regret guarantee for all bandwidths $h \in [0,1]$, where we are using
the rectangular kernel. This requires a bit more work.

First, set $D \defeq d 2^{d+2} T^2$ and form the discretized set:
\begin{align*}
\Hcal = \cbr{h \in \{\tfrac{1}{D},\tfrac{2}{D},\ldots,1\}: 1 \leq \tfrac{1}{h^d} \leq 2^{\lceil \log_2 T\rceil+1}}.
\end{align*}
We run \corral with kernel class $\Kcal = \{K_h: h \in \Hcal\}$ and we use \expfour with restarts as the sub-algorithms.
As $|\Hcal| \leq d 2^{d+2} T^2$, applying~\pref{thm:general_adaptivity} gives
\begin{align}
\forall h \in \Hcal: \Reg(T,\Pi_h) \leq \otil\rbr{T^{\frac{1}{1+\beta}}h^{-d\beta} (\log |\Pi|)^{\frac{\beta}{1+\beta}}}.
\label{eq:smooth_corral_discretized}
\end{align}
We now must lift~\pref{eq:smooth_corral_discretized} to all $h \in
[0,1]$. We have the following lemma.

\begin{lemma}
\label{lem:discretization}
For any loss $\ell: \Acal \to [0,1]$ and bandwidth $h \geq
T^{-\nicefrac{1}{d}}$, there exists $\hat{h} \in \Hcal$ such that
$\frac{1}{\hat{h}^d} \leq \frac{2}{h^d}$ and $\sup_{a}
\inner{K_{\hat{h}}(a)-K_h(a)}{\ell_t} \leq \tfrac{1}{T}$.
\end{lemma}
Applying this lemma allows us to obtain a smoothed regret
bound for $h \notin \Hcal$ by translating to $\hat{h} \in \Hcal$,
since the former benchmark is smaller by at most $O(1)$ while the
latter has $\hat{h}^{-d} \leq 2(h)^{-d}$. This
yields~\pref{thm:refined_adaptivity}.

For the non-contextual bound, instantiate each sub-algorithm
with a policy set $\Pi': \{x_0 \mapsto a: a \in \Acal'\}$ where
$\Acal'$ is a $\varepsilon$-covering of $\Acal$, satisfying
$|\Acal'| \leq O(\epsilon^{-d})$. The above analysis carries through,
and to translate to $a \notin \Acal'$ we require a different
discretization lemma.
\begin{lemma}
\label{lem:discretization_shift}
For $\rho(a,a') \leq \varepsilon$ and $\ell: \Acal \to [0,1]$, we have $\abr{\inner{K_h(a) - K_h(a')}{\ell}} \leq
4d\varepsilon h^{-d}$.
\end{lemma}
The proofs of both lemmas are deferred to the end of this section.

To finish the proof set $\varepsilon = \frac{1}{4dT^2}$ and note that
for $h < T^{-1/d}$ the desired guarantee is trivial. Thus for all
$h\geq T^{-1/d}$ the cumulative approximation error introduced by
discretization is at most $1$ while the policy set $\Pi'$ has $\ln|\Pi'|
\leq O(d\log dT)$.
\ifthenelse{\equal{\version}{arxiv}}{\end{proof}}{\end{proofof}}

\ifthenelse{\equal{\version}{arxiv}}{\begin{proof}[Proof of upper bounds in~\pref{thm:general_lipschitz_adaptivity}]}{\begin{proofof}{upper bounds in~\pref{thm:general_lipschitz_adaptivity}}}
For a finite set of bandwidths $\Hcal$ let us
apply~\pref{lem:corral-kernel} with $\Kcal = \{K_h: h \in \Hcal\}$ to obtain
\begin{align*}
\forall h \in \Hcal: \Reg(T,\Pi_h) \leq \otil\rbr{ \frac{|\Hcal|}{\eta} + T\eta + T\rbr{\eta \log (|\Pi||\Hcal|) h^{-d}}^{\beta}}
\end{align*}
Applying~\pref{lem:smooth_to_lip_general}, we know that
\begin{align*}
\min_{\pi \in \Pi} \EE\sum_{t=1}^T \inner{K_h\pi(x_t)}{\ell_t} \leq \min_{\pi \in \Pi} \EE\sum_{t=1}^T \ell_t(\pi(x_t)) + TLh,
\end{align*}
and so we obtain
\begin{align*}
\Reg(T,\Pi) \leq \min_{h \in \Hcal} TLh + \otil\rbr{\frac{|\Hcal|}{\eta} + T\eta + T\rbr{\eta \log (|\Pi||\Hcal|) h^{-d}}^{\beta}}.
\end{align*}
Define $\Lcal = \{2^i: i \in \{1,2,\ldots,\lceil \log_2(T)\rceil\}\}$
to be an exponentially spaced grid. If the true parameter $L \geq T$
then the bound is trivial, and otherwise $L \leq \hat{L} \leq 2L$ from
some $\hat{L} \in \Lcal$. We choose $\Hcal$ of size $\lceil
\log_2(T)\rceil$ to optimize the above bound for each
$\hat{L} \in \Lcal$. Specifically, set
\begin{align*}
\Hcal = \cbr{ h_i = (\eta\log(|\Pi|\log_2(T)))^{\frac{\beta}{d\beta+1}}2^{\frac{-i}{d\beta+1}}: i \in [\lceil \log_2(T)\rceil] }.
\end{align*}
This yields
\begin{align*}
\Reg(T,\Pi) &\leq \min_{h \in \Hcal} TLh + \otil\rbr{\frac{|\Hcal|}{\eta} + T\eta + T\rbr{\eta \log (|\Pi||\Hcal|) h^{-d}}^{\beta}}.\\
& \leq \min_{h \in \Hcal} T\hat{L}h + \otil\rbr{\frac{|\Hcal|}{\eta} + T\eta + T\rbr{\eta \log (|\Pi||\Hcal|) h^{-d}}^{\beta}}\\
& \leq \otil\rbr{T\hat{L}^{\frac{d\beta}{d\beta+1}} (\eta \log |\Pi|)^{\frac{\beta}{d\beta+1}} + \frac{1}{\eta} + T\eta}\\
& \leq \otil\rbr{TL^{\frac{d\beta}{d\beta+1}} (\eta \log |\Pi|)^{\frac{\beta}{d\beta+1}} + \frac{1}{\eta} + T\eta}.
\end{align*}
We finish the proof by tuning the master learning rate $\eta$ while
ignoring $L$. This gives
\begin{align*}
\eta = T^{\frac{-(d\beta+1)}{1+(d+1)\beta}} (\log |\Pi|)^{\frac{-\beta}{1+(d+1)\beta}},
\end{align*}
and the overall regret bound is
\begin{align*}
\Reg(T,\Pi) \leq \otil\rbr{ L^{\frac{d\beta}{1+d\beta}} T^{\frac{1+d\beta}{1+(d+1)\beta}} (\log |\Pi|)^{\frac{\beta}{1+(d+1)\beta}}  }.
\end{align*}

As in the proof of~\pref{thm:refined_adaptivity}, for the
non-contextual case we discretize the action space to a minimal
$\varepsilon$ cover $\Acal'$ for $\Acal$. Choosing $\varepsilon =
(4dT^2)^{-1}$ as in that proof suffices here as well.
\ifthenelse{\equal{\version}{arxiv}}{\end{proof}}{\end{proofof}}

We remark that~\pref{thm:general_lipschitz_adaptivity} is not a direct
corollary of~\pref{thm:refined_adaptivity}. Rather we must start
with~\pref{lem:corral-kernel} and first tune $h$ to balance the
sub-algorithm's regret with the $TLh$ term. Then we tune the master's
learning rate. In particular for fixed exponent $\beta$ the master
learning rates for~\pref{thm:refined_adaptivity}
and~\pref{thm:general_lipschitz_adaptivity} are different.

\subsection{Proofs of the lemmata}

\asedit{We prove a few auxiliary lemmas used in the previous sections, namely \pref{lem:discretization}, \pref{lem:discretization_shift} and \pref{lem:median-of-means}.}

\begin{lemma}[follows from \citet{hsu2016loss}]
\label{lem:median-of-means}
Suppose $\delta \in (0, 1)$, $k = 5\lceil \ln \frac 1 \delta \rceil$, $\tilde{n}$ is an integer, and $n = k \tilde{n}$.
In addition, $X_1, \ldots, X_n$ are iid random variables with mean $\mu$ and variance $\sigma^2$.
Define \[ \hat{\mu} \defeq \median\cbr{ \frac{1}{\tilde{n}} \sum_{i=1}^{\tilde{n}} X_i, \frac{1}{\tilde{n}} \sum_{i=\tilde{n}+1}^{2\tilde{n}} X_i,
\ldots,
\frac{1}{\tilde{n}} \sum_{i=(k-1)\tilde{n}+1}^{k\tilde{n}} X_i }. \]
Then with probability $1-\delta$,
\[ |\hat{\mu} - \mu| \leq \sigma \sqrt{\frac{40 \ln\frac{e}{\delta}}{n}}. \]
\end{lemma}
\begin{proof}
From the first part of~\citet[][Proposition 5]{hsu2016loss}, taking
$k = 5\lceil \ln \frac 1 \delta \rceil$, we have that with probability $1-\delta$,
\[ |\hat{\mu} - \mu| \leq \sigma \sqrt{\frac{8 k}{n}}. \]
The proof is completed by noting that $k \leq 5 (1 + \ln \frac 1 \delta) = 5 \ln \frac e \delta$.
\end{proof}

\ifthenelse{\equal{\version}{arxiv}}{\begin{proof}[Proof of~\pref{lem:discretization}]}{\begin{proofofnobb}{\pref{lem:discretization}}}
Recall the definition of $\Hcal$:
\[ \Hcal \defeq \cbr{ h \in \cbr{\frac 1 D, \frac 2 D, \ldots, 1}: 1 \leq \frac 1 {h^d} \leq 2^{\lceil \log T \rceil + 1} }. \]

Set $h_D = \frac{\lfloor h D \rfloor}{D}$. Note that $h_D$ is a multiple
of $\frac 1 D$. In addition, we note that $h \geq T^{-1}$, and
$h_D \geq h - \frac 1 {d 2^{d+2} T^2} \geq h - \frac 1 {4 d T^2} \geq h(1 - \frac 1 {4dT})$.
Therefore, by~\pref{fact:power-d} below,
 $\frac{1}{h_D^d} \leq \frac{1}{h^d} (\frac 1 {1 - \frac 1 {4dT}})^d \leq \frac 2 {h^d} \leq 2T \leq 2^{\lceil \log T \rceil + 1}$.
Hence, $h_D$ is in $\Hcal$. Moreover, $\base(\ball(a,h)) \geq h^d$, and
\begin{align*}
  \base(\ball(a,h) \setdiff \ball(a,h_D))
  \leq (2h)^d - (2h_D)^d
  \leq (2h)^d (1 - (1 - \frac 1 {d 2^{d+2} T})^d)
  \leq \frac{(2h)^d}{2^d T} = \frac{h^d}{2T}.
\end{align*}
\df{Here $\setdiff$ denotes the symmetric set difference.}
Therefore, applying~\pref{fact:loss_diff_sets}, we obtain
\begin{align*}
\abr{\inner{K_h(a) - K_{h_D}(a)}{\ell}} &\leq \frac{2\base(\ball(a,h)\Delta\ball(a,h_D))}{\max\cbr{\base(\ball(a,h)),\base(\ball(a,h_D))}}
\leq \frac{1}{T}. \ifthenelse{\equal{\version}{arxiv}}{\tag*\qedhere}{\tag*\qedsymbol}
\end{align*}
\ifthenelse{\equal{\version}{arxiv}}{\end{proof}}{\end{proofofnobb}}

\ifthenelse{\equal{\version}{arxiv}}{\begin{proof}[Proof of~\pref{lem:discretization_shift}]}{\begin{proofofnobb}{\pref{lem:discretization_shift}}}
Since we are using the $\ell_{\infty}$ distance and $\rho(a,a') \leq
\varepsilon$, we have that $\base(\ball(a,h)\Delta\ball(a',h)) \leq 2
\nbr{a - a'}_1 \leq
2d\varepsilon$. Applying~\pref{fact:loss_diff_sets} we obtain
\begin{align*}
\abr{\inner{K_h(a) - K_h(a')}{\ell}} &\leq \frac{2\base(\ball(a,h)\Delta\ball(a',h))}{\max\cbr{\base(\ball(a,h)),\base(\ball(a',h))}} \leq 4d\varepsilon h^{-d}.\ifthenelse{\equal{\version}{arxiv}}{\tag*\qedhere}{\tag*\qedsymbol}
\end{align*}
\ifthenelse{\equal{\version}{arxiv}}{\end{proof}}{\end{proofofnobb}}

\begin{fact}
  \label{fact:power-d}
  For $T,d \geq 1$,
  $ \rbr{\frac 1 {1 - \frac 1 {4dT}}}^d \leq 1 + \frac 1 T $.
\end{fact}
\begin{proof}
We use the following simple facts: for all $x$ in $[0,1]$, $e^x \leq 1+2x$ and
$e^{-x} \leq 1 - \frac 1 2 x$.
The proof is completed by noting that
$\frac{1}{(1 - \frac 1 {4dT})^d} \leq e^{\frac 1 {2T}} \leq 1 + \frac 1 T$.
\end{proof}

\begin{fact}
\label{fact:loss_diff_sets}
For sets $S_1$ and $S_2$, and a loss function $\ell: \Acal \to [0,1]$
\[
\abr{ \frac{\int_{S_1} \ell(a) d\nu(a)}{\nu(S_1)} - \frac{\int_{S_2} \ell(a) d\nu(a)}{\nu(S_2)} }
\leq
\frac{2 \nu(S_1 \setdiff S_2)}{\max(\nu(S_1), \nu(S_2))}.
\]
\end{fact}

\begin{proof}
\begin{align*}
& \abr{ \frac{\int_{S_1} \ell(a) d\nu(a)}{\nu(S_1)} - \frac{\int_{S_2} \ell(a) d\nu(a)}{\nu(S_2)} } \\
& =  \abr{\frac{\int_{S_1} \ell(a) d\nu(a) \cdot (\nu(S_2) - \nu(S_1)) + \nu(S_1) \cdot (\int_{S_1} \ell(a) d\nu(a) - \int_{S_2} \ell(a) d\nu(a)) }{\nu(S_1) \nu(S_2)}} \\
& \leq \frac{\nu(S_1) \cdot \nu(S_1 \setdiff S_2) + \nu(S_1) \cdot \nu(S_1 \setdiff S_2)}{\nu(S_1) \nu(S_2)}
= \frac{2\nu(S_1 \setdiff S_2)}{\nu(S_2)}
\end{align*}
By symmetry, the above is also bounded by $\frac{2\nu(S_1 \setdiff S_2)}{\nu(S_1)}$. The proof is completed by taking the smaller of the two upper bounds.
\end{proof}
 \section{Lower bounds for smoothness-adaptive algorithms}
\label{sec:lower}

In this section, we prove the lower bounds
in~\pref{thm:general_adaptivity}
and~\pref{thm:general_lipschitz_adaptivity}, showing that the
exponent combinations we achieve with \corral are optimal.  We start
with two lemmas that describe the constructions and contain the main
technical argument. In the next subsection we prove the theorems.

\subsection{The constructions}
The following two lemmas are based on a construction due
to~\citet{locatelli2018adaptivity}.  Their work concerns adapting to
the smoothness exponent, while ours focuses on the smoothness
constant. We also use a similar construction to show lower bounds
against uniformly-smoothed algorithms.

We focus on the stochastic non-contextual setting, where
we consider policy class
$\Pi = \cbr{x_0 \mapsto a: a \in \Acal}$, and at each time,
a dummy context $x_0$ is shown.
We use the shorthand $\Reg(T, h)$ to denote $\Reg(T, \Pi_h)$.
We define
$\Lambda$ to be the set of all functions from $\Acal$ to $[0,1]$. A
function $\lambda \in \Lambda$ defines an instance where $\ell(a) \sim
\textrm{Ber}(\lambda(a))$ for all $a \in \Acal$.

\begin{lemma}
\label{lem:h_lb}
Fix \df{any $T \in \NN$} and $h \in (0,\nicefrac 1 8]$.  Suppose an algorithm
  $\alg$ guarantees that \df{for all instances $\lambda \in \Lambda$,} \dfc{$\sup_{\lambda \in \Lambda}$}$
  \Reg(T,\nicefrac{1}{4}) \leq \Rs(\nicefrac{1}{4},T)$ where
  $\Rs(\nicefrac{1}{4},T) \leq \frac{\sqrt{T}}{20 (8h)^d}$. Then there exists
  $\lambda \in \Lambda$ such that $\alg$ has
\begin{align*}
\Reg(T,h) \geq \min\cbr{ \frac{T}{40 \cdot 2^d}, \frac{T}{400 (8h)^d \Rs(\nicefrac{1}{4},T)}}.
\end{align*}
\end{lemma}
\ifthenelse{\equal{\version}{arxiv}}{\begin{proof}}{\begin{proofnobb}}
We let $N = \lfloor \nicefrac 1 {4h} \rfloor^d$. Note that as $h \leq \nicefrac 1 8$,
$(\nicefrac 1 {8h})^d \leq N \leq (\nicefrac 1 {4h})^d$.
We also define $\Delta = \min\cbr{\frac{N}{40\Rs(1/4,T)}, \nicefrac 1 4} \in (0, \nicefrac 1 4]$.
By our assumption that $\Rs(\nicefrac{1}{4},T) \leq \frac{\sqrt{T}}{20 (8h)^d}$,
we have
\begin{equation}
  \Rs(\nicefrac{1}{4}, T) \leq \min\cbr{ \frac{N^2 T}{200 \Rs(\nicefrac{1}{4}, T)}, \frac{NT}{20} } = 0.2 N T \Delta.
  \label{eq:small-r-h-const}
\end{equation}
For each tuple $(s_1, \ldots, s_d) \in [\lfloor \nicefrac 1 {4h} \rfloor]^d$, we define a
point
$c_{s_1,\ldots,s_d} = (h(2s_1 - 1), \ldots, h(2s_d - 1))$.
There are $N$ points in total, which we call $c_1, \ldots, c_N$.
Define regions
\begin{align*}
H_i = \ball(c_i, h), i = 1,\ldots, N,
\end{align*}
which are disjoint subsets in $[0, \nicefrac 1 2]^d$. Finally, define
$S = [\nicefrac 1 2, 1]^d = \ball(c_0, \nicefrac 1 4)$, where $c_0 =
(\nicefrac 3 4, \ldots, \nicefrac 3 4)$.  We define several plausible
loss functions $\phi_0,\ldots,\phi_N \in \Lambda$:
\begin{align*}
\phi_0(a) = \left\{\begin{aligned}
& \nicefrac{1}{2}, & a \notin S \\
& \nicefrac{1}{2} - \nicefrac{\Delta}{2} & a \in S
\end{aligned}
\right. \qquad \textrm{and} \qquad
& \phi_i(a) = \left\{\begin{aligned}
& \nicefrac{1}{2}, & a \notin (H_i \cup S) \\
& \nicefrac{1}{2} - \Delta & a \in H_i\\
& \nicefrac{1}{2} - \nicefrac{\Delta}{2} & a \in S
\end{aligned}
\right.
\end{align*}
Note that
$\EE_{a \sim \smooth_{\nicefrac{1}{4}}(c_0)} \phi_0(a) = \nicefrac{1}{2} - \nicefrac{\Delta}{2}$, and
$\EE_{a \sim \smooth_h(c_i)} \phi_i(a) = \nicefrac{1}{2} - \Delta$.

The environments are parameterized by $\phi_i$ where losses are always
Bernoulli with mean $\phi_i$.
Denote by $\EE_i$ (resp. $\PP_i$) the expectation (resp. probability) over the randomness of the algorithm, along
with the randomness in environment $\phi_i$.

Observe that under environment $\phi_0$, for $h = \nicefrac 1 4$, we have $T\cdot\df{\min_{a}\lambda_{\nicefrac{1}{4}}(a)} = T\cdot(\nicefrac{1}{2}-\nicefrac{\Delta}{2})$.
Since \alg guarantees that $\Reg(T, \nicefrac{1}{4}) \leq \Rs(\nicefrac{1}{4}, T)$,
we have
\begin{align*}
\mathbb{E}_{0} \sum_{t=1}^T \phi_0(a_t) - T \cdot (\nicefrac 1 2 - \nicefrac \Delta 2)  \leq \Rs(\nicefrac{1}{4}, T).
\end{align*}
As for all $a$, $\phi_0(a) - (\nicefrac 1 2 - \nicefrac \Delta 2) = \nicefrac \Delta 2 \one\cbr{a \notin S}$, we get that
\begin{align*}
\sum_{t=1}^T \mathbb{E}_0 \one\cbr{a_t \notin S} \leq \frac{2R(\nicefrac{1}{4}, T)}{\Delta}.
\end{align*}
Denote by $T_i = \sum_{t=1}^T \one\cbr{a_t \in H_i}$ and observe that
\begin{align*}
\sum_{j=1}^N \mathbb{E}_0[T_j] \leq \mathbb{E}_0\sbr{\one\cbr{a_t \in \cup_{j=1}^N H_j}} \leq \sum_{t=1}^T \mathbb{E}_0 \sbr{ \one\cbr{a_t \notin S} } \leq \frac{2R(\nicefrac{1}{4}, T)}{\Delta}.
\end{align*}
By the pigeonhole principle, there exists at least one $i$ such that
\begin{equation}
 \mathbb{E}_0[T_i] \leq \frac{1}{N}\sum_{j=1}^N \mathbb{E}_0[T_j] \leq \frac{2 R(\nicefrac{1}{4}, T)}{N\Delta}.  \label{eq:pigeonhole-h}
\end{equation}
Therefore, by~\pref{lem:kl-pulls} and the fact that $\Delta \leq \nicefrac{1}{4}$, we have
\begin{align*}
 \KL(\PP_0, \PP_i) \leq \mathbb{E}_0[T_i] \cdot (4\Delta^2) \leq \frac{8R(\nicefrac{1}{4}, T) \Delta}{N}.
\end{align*}
By the choice of $\Delta \leq \frac{N}{40R(\nicefrac{1}{4},T)}$, we
have $\KL(\PP_0, \PP_i) \leq 0.2$ and so Pinsker's inequality yields
$\DTV(\PP_0, \PP_i) \leq \sqrt{\nicefrac 1 2 \KL(\PP_0, \PP_i)} \leq
0.4$. Therefore,
\begin{align*}
\mathbb{E}_i [T_i] \leq \mathbb{E}_0[T_i] + T\cdot d_{\textrm{TV}}(\PP_0,\PP_i) \leq \frac{2 \Rs(\nicefrac{1}{4}, T)}{N\Delta} + 0.4 T \leq 0.8T.
\end{align*}
where the first inequality is from the definition of the total variation distance and that
$T_i \in [0, T]$ almost surely;
the second inequality is by~\pref{eq:pigeonhole-h}; the third inequality is
by~\pref{eq:small-r-h-const}.
Therefore, $\EE_i[T_i] \leq 0.8 T$, which implies that on $\phi_i$
\begin{align*}
\Reg(T,h)
&= \mathbb{E}_i \sum_{t=1}^T \phi_i(a_t) - (\nicefrac{1}{2} - \Delta)
\geq \frac \Delta 2 \cdot (T - \mathbb{E}_i[T_i])
\geq \frac{\Delta}{2}\cdot 0.2 T\\
&\geq \min\cbr{ \frac{T}{40 \cdot 2^d}, \frac{T}{400 (8h)^d \Rs(\nicefrac{1}{4},T)}}. \ifthenelse{\equal{\version}{arxiv}}{\tag*\qedhere}{\tag*\qedsymbol}
\end{align*}
\ifthenelse{\equal{\version}{arxiv}}{\end{proof}}{\end{proofnobb}}

\begin{lemma}
For $\Delta \in [0, \frac 1 4]$,  $\KL(\PP_0, \PP_i) \leq \EE_0[T_i] \cdot (4\Delta^2)$.
\label{lem:kl-pulls}
\end{lemma}
\begin{proof}
We abbreviate $l_t$ as the outcome of $\ell_t(a_t)$. We have the following:
\begin{align*}
\KL(\PP_0, \PP_i)
&= \sum_{a_1,l_1,\ldots,a_T,l_T} \PP_0(a_1,l_1,\ldots,a_T,l_T) \log\frac{\PP_0(a_1,l_1,\ldots,a_T,l_T)}{\PP_i(a_1,l_1,\ldots,a_T,l_T)} \\
&= \mathbb{E}_0 \sum_{t=1}^T \log \frac{\PP_0(l_t|a_t)}{\PP_i(l_t|a_t)}
= \mathbb{E}_0 \sum_{t=1}^T \one\cbr{a_t \in H_i} \cdot \KL(\textrm{Ber}(1/2), \textrm{Ber}(1/2 - \Delta)) \\
&= \mathbb{E}_0[T_i] \cdot (-\frac 1 2 \log(1-4\Delta^2)) \leq \mathbb{E}_0[T_i] \cdot (4 \Delta^2)
\end{align*}
where the last inequality uses the fact that $\log(1-\nicefrac{x}{2}) \geq -x$ for
$x \in [0,1]$.
\end{proof}

For the next lemma, let $\Lambda(L)$ be the set of all $L$-Lipschitz mean loss functions.
\begin{lemma}
\label{lem:L_lb}
Fix \df{any $T \in \NN$ and }$L \geq 1$. Suppose an algorithm \alg guarantees that \df{for all instances $\lambda$ in $\Lambda(1)$,} \dfc{$\sup_{\lambda
  \in \Lambda(1)}$}$\Reg(T,0) \leq \Rl(1,T)$ where $\Rl(1,T) \leq
\frac{T}{40} L^d$. Then there exists a loss function $\lambda \in \Lambda(L)$ such that
\begin{align*}
\Reg(T,0) \geq \min \cbr{ \frac{T}{80}, \frac{T L^{\frac d {d+1}} } { 3200 \Rl(1,T)^{\frac{1}{d+1}} }  }.
\end{align*}
\end{lemma}

\ifthenelse{\equal{\version}{arxiv}}{\begin{proof}}{\begin{proofnobb}}
We let $\Delta = \min\cbr{(\frac{L^d}{40 \cdot \Rl(1,T) \cdot 8^d})^{\frac 1 {d+1}}, \nicefrac 1 8} \in (0,\nicefrac 1 8]$, and $N = \lfloor \nicefrac {L} {4\Delta} \rfloor^d$.
As $L \geq 1$, $\nicefrac {L} {4\Delta} \geq 2$.
Therefore,
$(\nicefrac {L} {8\Delta})^d \leq N \leq (\nicefrac {L} {4\Delta})^d$.
Observe that by the choices of $\Delta$ and $N$:
\begin{align*}
\Delta \leq \frac{(\frac {L} {8\Delta})^d}{40 \Rl(1,T)} \leq \frac{N}{40 \Rl(1,T)}.
\end{align*}
By our assumption that $\Rl(1,T) \leq \frac{T}{40} L^d$, we have that
\[ \Rl(1,T) \leq 0.2 T \cdot \frac{L^d}{8^d (1/8)^{d-1}} \leq 0.2 T \cdot \frac{L^d}{8^d \Delta^{d-1}} \leq 0.2 N T \Delta, \]
where the first inequality is from that $\Rl(1,T) \leq \frac{T}{40} L^d$; the second inequality is from
the fact that $\Delta \leq \frac 1 8$; the third inequality is from the fact that $N \geq (\nicefrac {L} {8\Delta})^d$.

For each tuple $(s_1, \ldots, s_d) \in [\lfloor \nicefrac L {4 \Delta}
  \rfloor]^d$, we define a point $c_{s_1,\ldots,s_d} =
(\frac{\Delta}{L} (2s_1 - 1), \ldots, \frac{\Delta}{L}(2s_d - 1))$.
There are $N$ points in total which we call $c_1, \ldots, c_N$.
Define regions
\begin{align*}
H_i = \ball\rbr{c_i, \frac{\Delta}{L}}, i = 1,\ldots, N,
\end{align*}
which are disjoint subsets in $[0, \nicefrac 1 2]^d$. Finally, define
$S = [\nicefrac 1 2, 1]^d = \ball(c_0, \nicefrac 1 4)$, where
$c_0 = (\nicefrac 3 4, \ldots, \nicefrac 3 4)$.  We define several
plausible loss functions $\phi_0 \in \Lambda(1)$,
$\phi_1,\ldots,\phi_N \in \Lambda(L)$:
\[
\phi_0(a) =
\left\{\begin{aligned}
&\nicefrac12 - (\nicefrac\Delta2 - \norm{a-c_0}_{\infty})_+, & a \in S\\
&\nicefrac12, & \textrm{else} \\
\end{aligned}
\right.
\textrm{and} \quad \phi_i(a) =
\left\{\begin{aligned}
&\nicefrac12 - (\Delta - L\norm{a-c_i}_{\infty})_+, & a \in H_i\\
&\nicefrac12 - (\nicefrac\Delta2 - \norm{a-c_0}_{\infty})_+, & a \in S\\
&\nicefrac12, & \textrm{else} \\
\end{aligned}
\right.
\]
Observe that $\phi_0$ is $1$-Lipschitz, and each
$\phi_i$ is $L$-Lipschitz for $i \geq 1$.

Each mean loss function $\phi_i$ defines an environment where realized
losses are Bernoulli random variables.
Denote by $\EE_i$ (resp. $\PP_i$) the expectation (resp. probability) over the randomness of the algorithm, along
with the randomness in environment $\phi_i$.

As \alg guarantees $\Reg(T,0) \leq \Rl(1,T)$ against all loss
functions in $\Sigma(1)$, we have
\[
 \mathbb{E}_0 \sum_{t=1}^T \phi_0(a_t) - T \rbr{\nicefrac 1 2 - \nicefrac \Delta 2} \leq \Rl(1,T).
\]
Denote by $T_i = \sum_{t=1}^T \one\cbr{a_t \in H_i}$. Observe that the
instantaneous regret for playing in any $H_i$ is at least
$\nicefrac{\Delta}{2}$. Therefore, by pigeonhole principle, there
exists at least one $i$ such that
\begin{equation}
\mathbb{E}_0[T_i] \leq \frac{1}{N}\sum_{j=1}^N \mathbb{E}_0[T_j] = \frac{1}{N}\mathbb{E}_0 \sum_{j=1}^N T_j \leq \frac{2 \Rl(1,T)}{N\Delta}.\label{eqn:pigeonhole-L}
\end{equation}
Following the exact same calculation as in the proof
of~\pref{lem:h_lb} we get that $\EE_i[T_i] \leq 0.8T$, which implies
that on instance $\phi_i$
\begin{align*}
\Reg(T)
\geq \mathbb{E}_i \sum_{t=1}^T \phi_i(a_t) - \rbr{\nicefrac 1 2 - \Delta}
\geq 0.2 T \cdot \frac{\Delta}{2}
\geq \min\cbr{\frac{T}{80}, \frac{L^{\frac{d}{d+1}}}{3200 \Rl(1,T)^{\frac{1}{d+1}}}}.  \ifthenelse{\equal{\version}{arxiv}}{\tag*\qedhere}{\tag*\qedsymbol}
\end{align*}
\ifthenelse{\equal{\version}{arxiv}}{\end{proof}}{\end{proofnobb}}

\subsection{Proofs of the lower bounds}
\ifthenelse{\equal{\version}{diff}}{
\color{blue}
}{
}
\ifthenelse{\equal{\version}{arxiv}}{\begin{proof}[Proof of the lower bound in~\pref{thm:refined_adaptivity}]}{\begin{proofof}{the lower bound in~\pref{thm:refined_adaptivity}}}
We show that the lower bound statement holds for $T_0 = 2^{3d(1+\beta)}$ and $c = \frac{1}{80 \cdot 2^{d(\beta+3)}}$.

Fix $T \geq T_0$; let $h_1 = \frac14$ and $h_2 = T^{\frac{-1}{d(\beta+1)}} \in (0,\frac 1 8]$. In addition, let $f(T, h) = c \cdot T^{\frac{1}{1+\beta}} h^{-d \beta} = \frac{T^{\frac{1}{1+\beta}} h^{-d \beta}}{80 \cdot 2^{d(\beta+3)}}$.

To finish the proof, we claim that for any algorithm \alg, one of the following must hold:
\begin{enumerate}
\item There exists some instance $\lambda \in \Lambda$, under which
$\Reg(T,\Pi_{h_1}) \geq f(T, h_1) = \frac {4^{d\beta} T^{\frac{1}{1+\beta}}} {80 \cdot 2^{d(\beta+3)}}$;
\item There exists some instance $\lambda \in \Lambda$, under which
$\Reg(T,\Pi_{h_2}) \geq f(T, h_2) = \frac{T}{80 \cdot 2^{d(\beta+3)}}$.
\end{enumerate}

Indeed, suppose \alg is such that for all instances $\lambda$, $\Reg(T,\Pi_{h_1}) < f(T, h_1)$. By our choice of $h_2$ and $T \geq T_0$, $f(T, h_1) \leq \frac{\sqrt{T}}{20 \cdot (8 h_2)^d}$.
Provided that this is satisfied,~\pref{lem:h_lb} with $\Rs(\nicefrac{1}{4},T) = f(T, h_1)$ gives that there is an instance $\lambda'$, under which
\begin{align*}
 \Reg(T,\Pi_{K_2}) \geq&
\min\cbr{ \frac{T}{40 \cdot 2^d},  \frac{1}{400 \cdot 8^d} T^{\frac{\beta}{1+\beta}} h_2^{-d}}
= \min\cbr{\frac{1}{40 \cdot 2^d}, \frac{1}{5 \cdot 2^{d\beta}}} T
> f(T, h_2),
\end{align*}
proving the above claim.
\ifthenelse{\equal{\version}{arxiv}}{\end{proof}}{\end{proofof}}

\ifthenelse{\equal{\version}{arxiv}}{\begin{proof}[Proof of the lower bound in~\pref{thm:general_lipschitz_adaptivity}]}{\begin{proofof}{the lower bound in~\pref{thm:general_lipschitz_adaptivity}}}
We show that the lower bound statement holds for $T_0 = 1$ and $c = \frac{1}{3200}$.

Fix $T \geq T_0$; We take $L_1 = 1$ and $L_2 = T^{\frac{1+d\beta}{d(1+(d+1)\beta)}}$. In addition, we let $g(T, L) = c \cdot T^{\frac{1+d\beta}{1+(d+1)\beta}} L^{\frac{d\beta}{1+d\beta}} = \frac{1}{3200} \cdot T^{\frac{1+d\beta}{1+(d+1)\beta}} L^{\frac{d\beta}{1+d\beta}}$.

To finish the proof, we claim that for any algorithm \alg, one of the following must hold:
\begin{enumerate}
\item There exists some instance $\lambda \in \Lambda(L_1)$, under which $\Reg(T,\Pi) \geq g(T, L_1) = \frac{1}{3200} \cdot T^{\frac{1+d\beta}{1+(d+1)\beta}}$;
\item There exists some instance $\lambda \in \Lambda(L_2)$, under which
$\Reg(T,\Pi) \geq g(T, L_2) = \frac{T}{3200}$.
\end{enumerate}

Indeed, suppose \alg is such that for all instances $\lambda \in \Lambda(L_1)$, $\Reg(T,\Pi) < g(T, L_1)$. By our choice of $L_2$ and $T \geq T_0$, $g(T, L_1) \leq \frac{1}{40} L_2^d T$. Provided this is satisfied,~\pref{lem:L_lb} with $\Rl(1,T) = g(T, L_1)$ gives that there is an instance $\lambda' \in \Lambda(L_2)$, under which
\begin{align*}
\Reg(T,\Pi) \geq& \min\cbr{\frac{T}{80}, \frac{1}{3200^{\frac{d}{d+1}}} \cdot T^{1 - \frac{1+d\beta}{(d+1)(1+(d+1)\beta)}}  L_2^{\frac{d}{d+1}}}
> \frac{1}{3200} T = g(T, L_2),
\end{align*}
proving the above claim.
\ifthenelse{\equal{\version}{arxiv}}{\end{proof}}{\end{proofof}}

\ifthenelse{\equal{\version}{diff}}{
\color{black}
}{
}

\ifthenelse{\equal{\version}{diff}}{
\color{red}
\begin{proof}[Proof of lower bound in~\pref{thm:general_adaptivity}]
The lower bound is a consequence of~\pref{lem:h_lb}. Specifically,
let
$\Acal$ be $[0,1]^d$ equipped with $\ell_{\infty}$ metric and uniform
base measure.
Consider any $T \geq 2^{3d(1+\beta)}$, and let $K_1$ be the rectangular kernel with bandwidth
$\nicefrac{1}{4}$ while $K_2$ has bandwidth $h = T^{\frac{-1}{d(\beta+1)}} \in (0,\frac 1 8]$.

Define $f_1(T) \defeq \frac {4^{d\beta} T^{\frac{1}{1+\beta}}} {80 \cdot 2^{d(\beta+3)}}$
and $f_2(T) \defeq \frac{T}{80 \cdot 2^{d(\beta+3)}}$. It can be easily checked that
$f_1(T) =
\Omega\rbr{T^{\frac{1}{1+\beta}}\kappa_1^{\frac{\beta}{1+\beta}}}$,
$f_2(T) = \Omega\rbr{T^{\frac{1}{1+\beta}}\kappa_2^{\frac{\beta}{1+\beta}} (\nicefrac{\kappa_2}{\kappa_1})^{\frac{\beta^2}{1+\beta}}}$.

Suppose for algorithm \alg,
$\sup_{\lambda \in \Lambda} \Reg(T,\Pi_{K_1}) < f_1(T)$.
Now apply~\pref{lem:h_lb}. We may take $\Rs(\nicefrac{1}{4},T) = f_1(T)$
which satisfies the precondition that
$\Rs(\nicefrac{1}{4},T) \leq \frac{\sqrt{T}}{20 \cdot (8 h)^d}$,
by our choice of $h = T^{\frac{-1}{d(\beta+1)}}$ and $T \geq 2^{3d(1+\beta)}$.
Provided this is satisfied, we may conclude that
\begin{align*}
\sup_{\lambda \in \Lambda} \Reg(T,\Pi_{K_2}) \geq&
\min\cbr{ \frac{T}{40 \cdot 2^d},  \frac{1}{400 \cdot 8^d} T^{\frac{\beta}{1+\beta}} h^{-d}}
= \min\cbr{\frac{1}{40 \cdot 2^d}, \frac{1}{5 \cdot 2^{d\beta}}} T
> f_2(T).
\end{align*}

This shows that for \alg, $\sup_{\lambda \in \Lambda}
\Reg(T,\Pi_{K_1}) < f_1(T)$ and $\sup_{\lambda \in \Lambda}
\Reg(T,\Pi_{K_2}) < f_2(T)$ cannot hold simultaneously.
\end{proof}

\begin{proof}[Proof of lower bound in~\pref{thm:general_lipschitz_adaptivity}]
Let $f_L(T) \defeq \frac{1}{6400} T^{\frac{1+d\beta}{1+(d+1)\beta}}
L^{\frac{d\beta}{1+d\beta}}$.

Assume \alg guarantees $\sup_{\lambda \in \Lambda(1)} \Reg(T,\Pi) \leq f_1(T)$,
otherwise we have already
proved what is required for \alg. In applying~\pref{lem:L_lb} we may
take
$L = T^{\frac{1+d\beta}{d(1+(d+1)\beta)}}$, and
$\Rl(1,T) = 2f_1(T) = \frac{1}{3200} T^{\frac{1+d\beta}{1+(d+1)\beta}}$ which
satisfies the precondition that $\Rl(1,T) \leq \frac{1}{40} L^d T$.
Provided this is satisfied, we may conclude that
\begin{align*}
\sup_{\lambda \in \Lambda(L)} \Reg(T,\Pi) \geq& \min\cbr{\frac{T}{80}, \frac{1}{3200^{\frac{d}{d+1}}} \cdot T^{1 - \frac{1+d\beta}{(d+1)(1+(d+1)\beta)}}  L^{\frac{d}{d+1}}}
> \frac{1}{3200} T = 2 f_L(T).
\end{align*}

This shows that for \alg,
$\sup_{\lambda \in \Lambda(1)} \Reg(T, \Pi) < 2f_1(T)$ and
$\sup_{\lambda \in \Lambda(L)} \Reg(T, \Pi) < 2f_L(T)$
cannot hold simultaneously. Therefore, for every $T$,
\[ \sup_L \sup_{\lambda \in \Lambda(L)} (f_L(T))^{-1} \Reg(T, \Pi) \geq 2, \]
which proves the lower bound.
\end{proof}
\color{black}
}{
}
 \section{Calculations for the examples}
\label{sec:calculations}

\paragraph{Calculation for~\pref{ex:small_smoothing}.}
Straightforward computations reveal that (1) $\lambda_h(a^\star) =
\nicefrac{h}{2}$, (2) $\forall a \in [a^\star - h, a^\star+h]$
$\lambda_h(a) \leq \lambda_h(a^\star) + \nicefrac{h}{2}$, and (3)
$\forall a \notin [a^\star - h, a^\star+h]$, $\lambda_h(a) \geq
\lambda_h(a^\star) + \nicefrac{\abr{a-a^\star}}{2}$. In particular,
the third item follows from a Taylor expansion, since
$\nicefrac{\partial\lambda_h(a)}{\partial a} \geq \nicefrac{1}{2}$ for
$a \geq a^\star+h$ (with a similar property for $a \leq a^\star-h$).

Therefore, if $\epsilon \leq \nicefrac{h}{2}$, we have
$\Pi_h(\epsilon) \subset \Pi_h(\nicefrac{h}{2}) \subset [a^\star - h,
  a^\star+h]$, which implies that $M_h(\epsilon,h) \leq 1$. On the
other hand, if $\epsilon > \nicefrac{h}{2}$, we have $\Pi_h(\epsilon)
\subset [a^\star - 2\epsilon,a^\star+2\epsilon]$, implying that
$M_h(\epsilon,h) \leq \nicefrac{4\epsilon}{h}$. Together we have that
$M_h(\epsilon,h) \leq O(\max\{1, \nicefrac{\epsilon}{h}\})$, and
plugging into the definition of $\theta_h(\cdot)$ concludes the proof.

\paragraph{Calculation for~\pref{ex:small_zooming}.}
First, for all $x$ and all $w$ in $\Sbb^{d-1}$,
$\EE[\ell(\pi_{w^\star}(x))|x] = f(0) \leq f(\inner{w}{x} - \inner{w^\star}{x}) = \EE[\ell(\pi_{w}(x))|x]$, which implies that
$\pi_{w^\star}$ is the optimal policy.  Next, consider the expected
regret of any policy $\pi_w$ in $\Pi$. Using the properties of $f$, we have
\begin{align*}
\EE[\ell(\pi_w(x))] - \EE[\ell(\pi_{w^\star}(x))]
  &\geq L_0 \EE[|\inner{w^\star}{x} - \inner{w}{x}|] \\
  &= L_0 \|w^\star - w\|_2 \EE[|x_1|]
  \geq\Omega(\nicefrac{L_0}{ \sqrt{d}}) \cdot \|w^\star - w\|_2.
\end{align*}
The equality follows from spherical symmetry, while the last
inequality follows since the probability density function of $x_1$ is
$p(x_1) = \frac{(1-x_1^2)^{\frac{d-3}{2}}}{\mathrm{B}(\frac{d-1}{2},
  \frac{1}{2})}$ and thus $\PP(|x_1| \geq \frac{1}{\sqrt{d}}) =
\Omega(1)$.

This latter inequality implies that, for any $\pi_w \in
\Pi_{0,L\epsilon}$, we have $\nbr{w - w^\star}_2 \leq
O(\nicefrac{L}{L_0} \cdot \sqrt{d} \epsilon)$. Therefore, for any $x$ we have
\begin{align*}
\Pi_{0,L\epsilon}(x) \subset \sbr{\inner{w^\star}{x} - O(\nicefrac{L}{L_0} \cdot \sqrt{d} \epsilon), \inner{w^\star}{x} + O(\nicefrac{L}{L_0} \cdot \sqrt{d} \epsilon) }.
\end{align*}
This implies that $M_0(L\epsilon,\epsilon) =
\EE_{\df{x\sim\Dcal_X}}\sbr{\Ncal_\epsilon(\Pi_{0,L\epsilon}(x))} \leq
O(\nicefrac{L}{L_0} \cdot
\sqrt{d})$. This immediately implies that $\psi_L(\epsilon) = O(\frac{L}{L_0 \epsilon}\cdot\sqrt{d})$.
Instantiating~\pref{thm:zooming} yields the regret bound.

 \section{Conclusions}
\label{sec:conclusions}

The main conceptual contribution of our paper is a new smoothing-based notion of
regret that admits guarantees with no assumptions on the loss.  Using
this, we design new algorithms providing \df{instance}-dependent guarantees
with optimal worst-case performance and Pareto-optimal adaptivity.
This also yields new guarantees for non-contextual and Lipschitz
bandits.

While our algorithms are computationally efficient in the
low-dimensional non-contextual setting, they are not tractable in
general since they require enumerating the policy set. Hence, the key
open question is: Are there algorithms with similar statistical
performance \emph{and} fast running time?

\subsection*{Acknowledgements}
We thank Wen Sun for insightful conversations during the initial
stages of this work. Part of this work was done while CZ was a postdoc at Microsoft
Research New York City.

\newpage

\renewcommand{\theHsection}{A\arabic{section}}

\vskip 0.2in
\bibliography{refs,bib-abbrv,bib-embedding,bib-slivkins-subset.bib,bib-nodeLabeling.bib}

\begin{thebibliography}{57}
\providecommand{\natexlab}[1]{#1}
\providecommand{\url}[1]{\texttt{#1}}
\expandafter\ifx\csname urlstyle\endcsname\relax
  \providecommand{\doi}[1]{doi: #1}\else
  \providecommand{\doi}{doi: \begingroup \urlstyle{rm}\Url}\fi

\bibitem[Abraham et~al.(2015)Abraham, Chechik, Kempe, and
  Slivkins]{SocialDistance-soda13-sicomp}
Ittai Abraham, Shiri Chechik, David Kempe, and Aleksandrs Slivkins.
\newblock Low-distortion inference of latent similarities from a multiplex
  social network.
\newblock \emph{SIAM Journal on Computing}, 2015.
\newblock Expanded and revised version of the conference paper in {\em SODA
  2013}.

\bibitem[Agarwal et~al.(2014)Agarwal, Hsu, Kale, Langford, Li, and
  Schapire]{agarwal2014taming}
Alekh Agarwal, Daniel Hsu, Satyen Kale, John Langford, Lihong Li, and Robert
  Schapire.
\newblock Taming the monster: A fast and simple algorithm for contextual
  bandits.
\newblock In \emph{International Conference on Machine Learning}, 2014.

\bibitem[Agarwal et~al.(2017{\natexlab{a}})Agarwal, Bird, Cozowicz, Hoang,
  Langford, Lee, Li, Melamed, Oshri, Ribas, Sen, and Slivkins]{DS-arxiv}
Alekh Agarwal, Sarah Bird, Markus Cozowicz, Luong Hoang, John Langford, Stephen
  Lee, Jiaji Li, Dan Melamed, Gal Oshri, Oswaldo Ribas, Siddhartha Sen, and
  Alex Slivkins.
\newblock Making contextual decisions with low technical debt.
\newblock \emph{arXiv:1606.03966}, 2017{\natexlab{a}}.

\bibitem[Agarwal et~al.(2017{\natexlab{b}})Agarwal, Luo, Neyshabur, and
  Schapire]{agarwal2016corralling}
Alekh Agarwal, Haipeng Luo, Behnam Neyshabur, and Robert~E Schapire.
\newblock Corralling a band of bandit algorithms.
\newblock In \emph{Conference on Learning Theory}, 2017{\natexlab{b}}.

\bibitem[Agrawal(1995)]{agrawal1995continuum}
Rajeev Agrawal.
\newblock The continuum-armed bandit problem.
\newblock \emph{SIAM Journal on Control and Optimization}, 1995.

\bibitem[Auer et~al.(2002)Auer, Cesa-Bianchi, Freund, and
  Schapire]{auer2002nonstochastic}
Peter Auer, Nicolo Cesa-Bianchi, Yoav Freund, and Robert~E Schapire.
\newblock The nonstochastic multiarmed bandit problem.
\newblock \emph{SIAM Journal on Computing}, 2002.

\bibitem[Auer et~al.(2007)Auer, Ortner, and Szepesv{\'a}ri]{auer2007improved}
Peter Auer, Ronald Ortner, and Csaba Szepesv{\'a}ri.
\newblock Improved rates for the stochastic continuum-armed bandit problem.
\newblock In \emph{Conference on Learning Theory}, 2007.

\bibitem[Berkenkamp et~al.(2019)Berkenkamp, Schoellig, and
  Krause]{berkenkamp2019no}
Felix Berkenkamp, Angela~P Schoellig, and Andreas Krause.
\newblock No-regret bayesian optimization with unknown hyperparameters.
\newblock \emph{Journal of Machine Learning Research}, 2019.

\bibitem[Bubeck and Cesa-Bianchi(2012)]{bubeck2012regret}
S{\'e}bastien Bubeck and Nicolo Cesa-Bianchi.
\newblock Regret analysis of stochastic and nonstochastic multi-armed bandit
  problems.
\newblock \emph{Foundations and Trends in Machine Learning}, 2012.

\bibitem[Bubeck et~al.(2011{\natexlab{a}})Bubeck, Munos, Stoltz, and
  Szepesv{\'a}ri]{bubeck2011x}
S{\'e}bastien Bubeck, R{\'e}mi Munos, Gilles Stoltz, and Csaba Szepesv{\'a}ri.
\newblock X-armed bandits.
\newblock \emph{Journal of Machine Learning Research}, 2011{\natexlab{a}}.

\bibitem[Bubeck et~al.(2011{\natexlab{b}})Bubeck, Stoltz, and
  Yu]{bubeck2011lipschitz}
S{\'e}bastien Bubeck, Gilles Stoltz, and Jia~Yuan Yu.
\newblock Lipschitz bandits without the lipschitz constant.
\newblock In \emph{Algorithmic Learning Theory}, 2011{\natexlab{b}}.

\bibitem[Bull(2015)]{Bull-bandits14}
Adam Bull.
\newblock Adaptive-treed bandits.
\newblock \emph{Bernoulli}, 2015.

\bibitem[Cesa-Bianchi et~al.(2017)Cesa-Bianchi, Gaillard, Gentile, and
  Gerchinovitz]{cesa2017algorithmic}
Nicol{\`o} Cesa-Bianchi, Pierre Gaillard, Claudio Gentile, and S{\'e}bastien
  Gerchinovitz.
\newblock Algorithmic chaining and the role of partial feedback in online
  nonparametric learning.
\newblock In \emph{Conference on Learning Theory}, 2017.

\bibitem[Chan et~al.(2009)Chan, Dhamdhere, Gupta, Kleinberg, and
  Slivkins]{Slivkins-focs05-full}
Hubert T-H. Chan, Kedar Dhamdhere, Anupam Gupta, Jon Kleinberg, and Aleksandrs
  Slivkins.
\newblock Metric embeddings with relaxed guarantees.
\newblock \emph{SIAM Journal on Computing}, 2009.
\newblock Preliminary version in \emph{FOCS 2005}, merged with an independent
  effort by another research group.

\bibitem[Chen et~al.(2016)Chen, Zeng, and Kosorok]{chen2016personalized}
Guanhua Chen, Donglin Zeng, and Michael~R Kosorok.
\newblock Personalized dose finding using outcome weighted learning.
\newblock \emph{Journal of the American Statistical Association}, 2016.

\bibitem[Dudik et~al.(2011)Dudik, Hsu, Kale, Karampatziakis, Langford, Reyzin,
  and Zhang]{dudik2011efficient}
Miroslav Dudik, Daniel Hsu, Satyen Kale, Nikos Karampatziakis, John Langford,
  Lev Reyzin, and Tong Zhang.
\newblock Efficient optimal learning for contextual bandits.
\newblock In \emph{Uncertainty in Artificial Intelligence}, 2011.

\bibitem[Freund and Schapire(1997)]{FS97}
Yoav Freund and Robert~E. Schapire.
\newblock A decision-theoretic generalization of on-line learning and an
  application to boosting.
\newblock \emph{Journal of Computer and System Sciences}, 1997.

\bibitem[Grill et~al.(2015)Grill, Valko, and Munos]{grill2015black}
Jean-Bastien Grill, Michal Valko, and R{\'e}mi Munos.
\newblock Black-box optimization of noisy functions with unknown smoothness.
\newblock In \emph{Advances in Neural Information Processing Systems}, 2015.

\bibitem[Gupta et~al.(2003)Gupta, Krauthgamer, and Lee]{Gup03}
Anupam Gupta, Robert Krauthgamer, and James~R. Lee.
\newblock Bounded geometries, fractals, and low--distortion embeddings.
\newblock In \emph{Symposium on Foundations of Computer Science}, 2003.

\bibitem[Hsu and Sabato(2016)]{hsu2016loss}
Daniel Hsu and Sivan Sabato.
\newblock Loss minimization and parameter estimation with heavy tails.
\newblock \emph{Journal of Machine Learning Research}, 2016.

\bibitem[Kallus and Zhou(2018)]{kallus2018policy}
Nathan Kallus and Angela Zhou.
\newblock Policy evaluation and optimization with continuous treatments.
\newblock In \emph{International Conference on Artificial Intelligence and
  Statistics}, 2018.

\bibitem[Kempe and Kleinberg(2002)]{Kempe-focs02}
David Kempe and Jon Kleinberg.
\newblock Protocols and impossibility results for gossip-based communication
  mechanisms.
\newblock In \emph{Symposium on Foundations of Computer Science}, 2002.

\bibitem[Kempe et~al.(2005)Kempe, Kleinberg, and Demers]{Kempe-stoc01}
David Kempe, Jon Kleinberg, and Alan Demers.
\newblock Spatial gossip and resource location protocols.
\newblock \emph{Journal of the ACM}, 2005.
\newblock Preliminary version in \emph{STOC 2001}.

\bibitem[Kleinberg(2000)]{Kle00-STOC}
Jon Kleinberg.
\newblock {The small-world phenomenon: an algorithmic perspective}.
\newblock In \emph{Symposium on Theory of Computing}, 2000.

\bibitem[Kleinberg et~al.(2009)Kleinberg, Slivkins, and
  Wexler]{Slivkins-focs04}
Jon Kleinberg, Aleksandrs Slivkins, and Tom Wexler.
\newblock Triangulation and embedding using small sets of beacons.
\newblock \emph{Journal of the ACM}, 2009.
\newblock Subsumes conference papers in \emph{FOCS 2004} and \emph{SODA 2005}.

\bibitem[Kleinberg(2004)]{Bobby-nips04}
Robert Kleinberg.
\newblock Nearly tight bounds for the continuum-armed bandit problem.
\newblock In \emph{Advances in Neural Information Processing Systems}, 2004.

\bibitem[Kleinberg and Leighton(2003)]{kleinberg2003value}
Robert Kleinberg and Tom Leighton.
\newblock The value of knowing a demand curve: Bounds on regret for online
  posted-price auctions.
\newblock In \emph{Symposium on Foundations of Computer Science}, 2003.

\bibitem[Kleinberg and Slivkins(2010)]{DichotomyMAB-soda10}
Robert Kleinberg and Aleksandrs Slivkins.
\newblock Sharp dichotomies for regret minimization in metric spaces.
\newblock In \emph{Symposium on Discrete Algorithms}, 2010.

\bibitem[Kleinberg et~al.(2008)Kleinberg, Slivkins, and
  Upfal]{LipschitzMAB-stoc08}
Robert Kleinberg, Aleksandrs Slivkins, and Eli Upfal.
\newblock Multi-armed bandits in metric spaces.
\newblock In \emph{Symposium on Theory of Computing}, 2008.

\bibitem[Kleinberg et~al.(2019)Kleinberg, Slivkins, and
  Upfal]{kleinberg2013bandits}
Robert Kleinberg, Aleksandrs Slivkins, and Eli Upfal.
\newblock Bandits and experts in metric spaces.
\newblock \emph{Journal of the ACM}, 2019.
\newblock Merged and revised version of conference papers in {\em STOC 2008}
  and {\em SODA 2010}. Also available at {\tt http://arxiv.org/abs/1312.1277}.

\bibitem[Krause and Ong(2011)]{krause2011contextual}
Andreas Krause and Cheng~S Ong.
\newblock Contextual gaussian process bandit optimization.
\newblock In \emph{Advances in neural information processing systems}, 2011.

\bibitem[Krauthgamer and Lee(2004)]{KL-soda04}
Robert Krauthgamer and James~R. Lee.
\newblock Navigating nets: simple algorithms for proximity search.
\newblock In \emph{Symposium on Discrete Algorithms}, 2004.

\bibitem[Krauthgamer et~al.(2005)Krauthgamer, Lee, Mendel, and Naor]{KLMN04}
Robert Krauthgamer, James Lee, Manor Mendel, and Assaf Naor.
\newblock Measured descent: {A} new embedding method for finite metrics.
\newblock \emph{Geometric and Functional Analysis}, 2005.
\newblock Preliminary version in {\em FOCS}, 2004.

\bibitem[Krishnamurthy et~al.(2016)Krishnamurthy, Agarwal, and
  Dudik]{krishnamurthy2016contextual}
Akshay Krishnamurthy, Alekh Agarwal, and Miro Dudik.
\newblock Contextual semibandits via supervised learning oracles.
\newblock In \emph{Advances In Neural Information Processing Systems}, 2016.

\bibitem[Langford and Zhang(2007)]{Langford-nips07}
John Langford and Tong Zhang.
\newblock The epoch-greedy algorithm for contextual multi-armed bandits.
\newblock In \emph{Advances in Neural Information Processing Systems}, 2007.

\bibitem[Lattimore and Szepesv\'{a}ri(2020)]{lattimore2018bandit}
Tor Lattimore and Csaba Szepesv\'{a}ri.
\newblock \emph{Bandit Algorithms}.
\newblock Cambridge University Press, 2020.
\newblock Versions available at {\tt https://banditalgs.com/} since 2018.

\bibitem[Locatelli and Carpentier(2018)]{locatelli2018adaptivity}
Andrea Locatelli and Alexandra Carpentier.
\newblock Adaptivity to smoothness in x-armed bandits.
\newblock In \emph{Conference on Learning Theory}, 2018.

\bibitem[Lu et~al.(2010)Lu, P\'{a}l, and P\'{a}l]{Pal-Bandits-aistats10}
Tyler Lu, D\'{a}vid P\'{a}l, and Martin P\'{a}l.
\newblock {Showing Relevant Ads via Lipschitz Context Multi-Armed Bandits}.
\newblock In \emph{14th Intl. Conf. on Artificial Intelligence and Statistics
  (AISTATS)}, 2010.

\bibitem[Luukkainen and Saksman(1998)]{Luuk-ams98}
Jouni Luukkainen and Eero Saksman.
\newblock {Every complete doubling metric space carries a doubling measure}.
\newblock \emph{Proceedings of the American Mathematical Society}, 1998.

\bibitem[Mendel and Har-Peled(2005)]{Men05}
Manor Mendel and Sariel Har-Peled.
\newblock {Fast construction of nets in low dimensional metrics, and their
  applications}.
\newblock In \emph{Symposium on Computational Geometry}, 2005.

\bibitem[Minsker(2013)]{minsker2013estimation}
Stanislav Minsker.
\newblock Estimation of extreme values and associated level sets of a
  regression function via selective sampling.
\newblock In \emph{Conference on Learning Theory}, 2013.

\bibitem[Sarkar et~al.(2010)Sarkar, Chakrabarti, and Moore]{Sarkar-colt10}
Purnamrita Sarkar, Deepayan Chakrabarti, and Andrew~W. Moore.
\newblock Theoretical justification of popular link prediction heuristics.
\newblock In \emph{Conference on Learning Theory}, 2010.

\bibitem[Sen et~al.(2018)Sen, Shanmugam, and Shakkottai]{sen2018contextual}
Rajat Sen, Karthikeyan Shanmugam, and Sanjay Shakkottai.
\newblock Contextual bandits with stochastic experts.
\newblock In \emph{International Conference on Artificial Intelligence and
  Statistics}, 2018.

\bibitem[Shang et~al.(2019)Shang, Kaufmann, and Valko]{shang2019general}
Xuedong Shang, Emilie Kaufmann, and Michal Valko.
\newblock General parallel optimization a without metric.
\newblock In \emph{Algorithmic Learning Theory}, 2019.

\bibitem[Sion(1958)]{sion1958general}
Maurice Sion.
\newblock On general minimax theorems.
\newblock \emph{Pacific Journal of mathematics}, 1958.

\bibitem[Slivkins(2006)]{Slivkins-thesis}
Aleksandrs Slivkins.
\newblock \emph{Embedding, Distance Estimation and Object Location in
  Networks}.
\newblock PhD thesis, Cornell University, 2006.
\newblock Available online at
  \texttt{http://research.microsoft.com/en-us/people/slivkins}.

\bibitem[Slivkins(2007)]{Slivkins-podc05}
Aleksandrs Slivkins.
\newblock Distance estimation and object location via rings of neighbors.
\newblock \emph{Distributed Computing}, 2007.
\newblock Special issue for {\em PODC 2005}. Preliminary version in {\em PODC
  2005}.

\bibitem[Slivkins(2011)]{ImplicitMAB-nips11}
Aleksandrs Slivkins.
\newblock Multi-armed bandits on implicit metric spaces.
\newblock In \emph{Advances in Neural Information Processing Systems}, 2011.

\bibitem[Slivkins(2014)]{slivkins2014contextual}
Aleksandrs Slivkins.
\newblock Contextual bandits with similarity information.
\newblock \emph{Journal of Machine Learning Research}, 2014.
\newblock Preliminary version in \emph{COLT 2011}.

\bibitem[Slivkins(2019)]{slivkins-MABbook}
Aleksandrs Slivkins.
\newblock Introduction to multi-armed bandits.
\newblock \emph{Foundations and Trends in Machine Learning}, 2019.
\newblock Also available at {\tt https://arxiv.org/abs/1904.07272}.

\bibitem[Srinivas et~al.(2012)Srinivas, Krause, Kakade, and
  Seeger]{srinivas2012information}
Niranjan Srinivas, Andreas Krause, Sham~M Kakade, and Matthias~W Seeger.
\newblock Information-theoretic regret bounds for gaussian process optimization
  in the bandit setting.
\newblock \emph{IEEE Transactions on Information Theory}, 2012.

\bibitem[Talwar(2004)]{Tal04}
Kunal Talwar.
\newblock Bypassing the embedding: Algorithms for low-dimensional metrics.
\newblock In \emph{Symposium on Theory of Computing}, 2004.

\bibitem[Valko et~al.(2013)Valko, Carpentier, and Munos]{valko2013stochastic}
Michal Valko, Alexandra Carpentier, and R{\'e}mi Munos.
\newblock Stochastic simultaneous optimistic optimization.
\newblock In \emph{International Conference on Machine Learning}, 2013.

\bibitem[Volberg and Konyagin(1987)]{Volberg87}
A.~L. Volberg and S.~V. Konyagin.
\newblock {On measures with the doubling condition}.
\newblock \emph{Izvestiya Akademii Nauk SSSR}, 1987.
\newblock In Russian; English translation in Mathematics of the USSR-Izvestiya,
  1988.

\bibitem[Wei et~al.(2020)Wei, Luo, and Agarwal]{wei2020taking}
Chen-Yu Wei, Haipeng Luo, and Alekh Agarwal.
\newblock Taking a hint: How to leverage loss predictors in contextual bandits?
\newblock \emph{arXiv:2003.01922}, 2020.

\bibitem[Wong et~al.(2005)Wong, Slivkins, and Sirer]{Meridian-sigcomm05}
Bernard Wong, Aleksandrs Slivkins, and Emin~G\"{u}n Sirer.
\newblock Meridian: A lightweight network location service without virtual
  coordinates.
\newblock In \emph{SIGCOMM Conference on Applications, Technologies,
  Architectures, and Protocols for Computer Communications}, 2005.
\newblock Full version is available at
  \texttt{http://research.microsoft.com/en-us/people/slivkins}.

\bibitem[Wu(1998)]{Wu-ams98}
Jang-Mei Wu.
\newblock {Hausdorff dimension and doubling measures on metric spaces}.
\newblock \emph{Proceedings of the American Mathematical Society}, 1998.

\end{thebibliography}

\end{document}